\documentclass{article}

\PassOptionsToPackage{numbers, compress}{natbib}
\usepackage[preprint]{template}
\usepackage[title]{appendix}
\usepackage{graphicx} 
\usepackage{float} 
\usepackage{subfigure} 
\usepackage[subfigure]{tocloft}
\newcommand{\listappendixname}{Appendix Contents}
\newlistof{appendix}{apc}{\listappendixname}

\usepackage{algorithmic}
\usepackage{wrapfig}
\usepackage{enumitem}
\usepackage{multirow}
\usepackage{bbm}
\usepackage{ragged2e}  
\usepackage{bbding}
\usepackage{ulem}

\usepackage{graphics}
\usepackage{mathtools}
\usepackage[linesnumbered, ruled]{algorithm2e}

\usepackage{amsmath}
\usepackage{amssymb}
\usepackage[utf8]{inputenc} 
\usepackage[T1]{fontenc}    
\usepackage{url}            
\usepackage{booktabs}       
\usepackage{amsfonts}       
\usepackage{nicefrac}       
\usepackage{microtype}      
\usepackage{xcolor}         
\definecolor{darkgreen}{rgb}{0, 0.5, 0}
\usepackage[colorlinks=true, citecolor=blue, linkcolor=blue]{hyperref}
\usepackage{cleveref}
\Crefname{equation}{Eq.}{Eqs.}
\PassOptionsToPackage{numbers,sor&compress}{natbib}
\usepackage{amsthm}
\usepackage{threeparttable}

\newtheorem{theorem}{Theorem}[section]
\newtheorem{lemma}[theorem]{Lemma}

\newtheorem{proposition}[theorem]{Proposition}
\newtheorem{corollary}[theorem]{Corollary}
\newtheorem{definition}{Definition}[section]

\newtheorem*{lemma*}{Lemma}
\newtheorem*{proposition*}{Proposition}
\newtheorem*{proof*}{Proof}
\newtheorem*{corollary*}{Corollary}

\title{ 
Rethinking the Diffusion Models for Numerical Tabular Data Imputation from the Perspective of Wasserstein Gradient Flow 
}

\author{
  Zhichao Chen/Ziciu Can$^{1}$ \quad Haoxuan Li$^{2}$ \quad Fangyikang Wang$^{1}$ \quad Odin Zhang$^{1}$ \quad Hu Xu$^{1}$ \\ 
  \textbf{Xiaoyu Jiang$^{1}$ \quad Zhihuan Song$^{1}$ \quad Eric H. Wang$^{1}$\thanks{Corresponding author.}}\\
  $^1$Zhejiang University \quad
  $^2$Peking University \\
}

\begin{document}

\maketitle

\begin{abstract}
Diffusion models (DMs) have gained attention in Missing Data Imputation (MDI), but there remain two long-neglected issues to be addressed:
(1). Inaccurate Imputation, which arises from inherently sample-diversification-pursuing generative process of DMs.
(2). Difficult Training, which stems from intricate design required for the mask matrix in model training stage.
To address these concerns within the realm of numerical tabular datasets\footnote{Datasets organized in a table format where each entry is a numerical value.}, we introduce a novel principled approach termed \uline{K}ernelized \uline{N}egative \uline{E}ntropy-regularized \uline{W}asserstein gradient flow \uline{Imp}utation (KnewImp). 
Specifically, based on Wasserstein gradient flow (WGF) framework, we first prove that issue (1) stems from the cost functionals implicitly maximized in DM-based MDI are equivalent to the MDI's objective plus diversification-promoting non-negative terms. 
Based on this, we then design a novel cost functional with diversification-discouraging negative entropy and derive our KnewImp approach within WGF framework and reproducing kernel Hilbert space. After that, we prove that the imputation procedure of KnewImp can be derived from another cost functional related to the joint distribution, eliminating the need for the mask matrix and hence naturally addressing issue (2). Extensive experiments demonstrate that our proposed KnewImp approach significantly outperforms existing state-of-the-art methods. 

\end{abstract}
\vspace{-0.4cm}
\section{Introduction}
The imputation of missing values from observational data is crucial for constructing machine learning models with broad applications across various fields, including e-commerce~\cite{li2024removing,li2023multiple,wang2022escm2}, healthcare~\cite{tashiro2021csdi}, and process industry~\cite{10508098}. Recently, diffusion models (DMs) have emerged as a powerful tool for missing data imputation (MDI), celebrated for their excellent capability to model data distributions and generate high-quality synthetic data~\cite{ouyang2023missdiff,song2020score,yang2023diffusion}. These models excel by approximating the (Stein) score function of the conditional distribution between missing and observed data, thereby reformulating the imputation problem as a generative task grounded in the learned score function. 

Although DMs have shown considerable success in MDI tasks, they face significant challenges that result from model inference and training: (1). Inaccurate Imputation: While DM-based approaches treat MDI as a conditional generative task by sampling from the learned score function, it is important to note that the primary evaluation metric for MDI focuses on accuracy~\cite{jarrett2022hyperimpute,muzellec2020missing}, rather than the sample diversification typically emphasized in generative tasks. Consequently, the inference objectives implicitly pursued for DMs may not align well with the specific needs of the MDI task. (2). Difficult Training: The training of diffusion models is complicated by the unknown nature of the ground-truth values of missing data. Previous methods~\cite{tashiro2021csdi,devlin-etal-2019-bert} have sought to address this by masking parts of the available observational data and then imputing these masked entries as a means to construct model training. In this procedure, the design of mask matrices is essential, as pointed out by reference~\cite{tashiro2021csdi}, and hence results in a training difficulty due to complex mask matrix selection mechanism.

The cornerstone for mitigating issue (1) lies in identifying the cost functional that is `secretly maximized' during the DM-based MDI procedure, delineating its relationship to the `vanilla' MDI's cost functional, refining it where it proves deficient, and redeveloping the corresponding imputation procedure for this enhanced functional. Building on this, the resolution to issue (2) involves bypassing the use of the conditional distribution in the imputation phase. In other words, maintaining the existing imputation procedures while transforming the cost functional into an equivalent form that merely contains the joint distribution.
To tackle these challenges in the realm of numerical tabular data, we introduce our approach named \uline{K}ernelized \uline{N}egative \uline{E}ntropy-regularized \uline{W}asserstein Gradient Flow \uline{Imp}utation (KnewImp). Specifically, to address issue (1), we unify the DMs' generative processes into the Wasserstein Gradient Flow (WGF) framework, recover their cost functionals, and validate their connections between MDI's objective. Based on this analysis, we then introduce a novel negative entropy-regularized (NER) cost functional, and establish a new easy-to-implement and closed-form MDI procedure similar to DMs' generative processes within the WGF framework and reproducing kernel Hilbert space (RKHS). After that, to circumvent the use of the mask matrix and address issue (2), we further develop a novel cost functional concerned with joint distribution, proving it serves as a lower bound to the NER functional, with a constant gap that preserves the imputation procedure within the WGF framework, and consequently streamlining the model training procedure.

In summary, the contributions of this manuscript can be summarized as follows:
\begin{itemize}[leftmargin=*]
  \item{
    We elucidate the inaccurate imputation issue for DM-based MDI approaches by revealing the relationship between their cost functional and identical MDI's objective. 
    Based on this, we propose KnewImp by designing a novel NER functional and obtaining corresponding 
    imputation procedure. 
  }
  \item{We demonstrate that the conditional distribution-related NER functional, 
  can be seamlessly transformed into another joint distribution-related NER functional, maintaining a constant gap. Consequently, we bypass the design of a mask matrix during the model training stage of KnewImp without explicitly altering its imputation procedure. }
  \item{We empirically validate the superiority of our proposed KnewImp method over state-of-the-art models through rigorous testing on various numerical tabular datasets.}
\end{itemize}

\section{Preliminaries}
\subsection{Missing Data Imputation }\label{subsec:pre-MDI}
Building upon previous works in MDI~\cite{muzellec2020missing,zhao2023transformed}, our objective is articulated as follows: Consider a numerical tabular dataset represented by a matrix $\boldsymbol{X} \in \mathbb{R}^{\mathrm{N} \times \mathrm{D}}$, comprising $\mathrm{N}$ samples each of dimension $\mathrm{D}$, where certain entries are missing. Accompanying $\boldsymbol{X}$ is a binary mask matrix $\boldsymbol{M} \in \{0,1\}^{\mathrm{N} \times \mathrm{D}}$, where entry $\boldsymbol{M}[i,d]$ is set as $0$ if $\boldsymbol{X}[i,d]$ is missing, thereby assigned as $\boldsymbol{X}[i,d] = \texttt{NaN}$ (denoting `not a number'), and $1$ otherwise. Hence, the matrix $\boldsymbol{X}$ is expressed as $\boldsymbol{X} = \boldsymbol{X}^{(\text{obs})} \odot \boldsymbol{M} + \texttt{NaN} \odot (\mathbbm{1}_{\mathrm{N} \times \mathrm{D}} - \boldsymbol{M})$,
where $\boldsymbol{X}^{(\text{obs})}$ represents the matrix with observed entries, $\odot$ denotes the Hadamard product, and $\mathbbm{1}_{\mathrm{N} \times \mathrm{D}}$ is a matrix of ones sized ${\mathrm{N} \times \mathrm{D}}$. The task at hand is to impute the missing entries in $\boldsymbol{X}$, yielding an estimation $\hat{\boldsymbol{X}}$ using the imputed matrix ${\boldsymbol{X}}^{(\text{imp})}$, formulated as $\hat{\boldsymbol{X}} = \boldsymbol{X}^{(\text{obs})} \odot \boldsymbol{M} + {\boldsymbol{X}}^{(\text{imp})} \odot (\mathbbm{1}_{\mathrm{N} \times \mathrm{D}} - \boldsymbol{M})$. Notably, according to reference~\cite{rubin1976inference}, missing data can be classified into three categories: Missing Completely at Random (MCAR), Missing at Random (MAR), and Missing Not at Random (MNAR)\footnote{Detailed information about these missing mechanisms is given in~\Cref{appendix-subsec:missingSimulation}.}, and we mainly restrict our discussion scope on numerical tabular data with MAR and MCAR settings.

\subsection{Diffusion Models and its application for Missing Data Imputation}\label{subsec:diffusionKnowledge}
According to reference~\cite{song2020score}, DMs begin by corrupting data towards a tractable noise distribution, typically a standard Gaussian, and then reverse this process to generate samples. Specifically, the forward corruption or diffusion process can be described as a discretization of the stochastic differential equation (SDE) along time $\tau$: $\mathrm{d}\boldsymbol{X}_\tau=f(\boldsymbol{X}_\tau,\tau)\mathrm{d}\tau + g_\tau\mathrm{d}W_{\tau}$, where $f(\boldsymbol{X}_\tau,\tau)$ is drift term, $g_\tau$ is volatility term, and $\mathrm{d}W{\tau}$ is standard Wiener process. The solution to this SDE forms a continuous trajectory of random variables ${\boldsymbol{X}_\tau}\vert_{\tau=0}^{\mathrm{T}}$. The density function $q_\tau$ of these variables is governed by the Fokker-Planck-Kolmogorov (FPK) equation $\frac{\partial q_\tau}{\partial \tau}=-\nabla\cdot(q_\tau f(\boldsymbol{X}\tau,\tau))+\frac{1}{2}g^2_\tau\nabla\cdot\nabla q_\tau$, as per Theorem 5.4 in reference~\cite{sarkka2019applied}. According to reference~\cite{anderson1982reverse}, the reverse process for sample generation is described by: $\mathrm{d}\boldsymbol{X}_\tau=[f(\boldsymbol{X}_\tau,\tau)-g^2_\tau\nabla{\log{p(\boldsymbol{X}_\tau)}}]\mathrm{d}\tau + g_\tau\mathrm{d}W_{\tau}$, where $\nabla{\log{p(\boldsymbol{X}_\tau)}}$ represents the score function and learned via neural networks during DM training phase.

Based on this, 
DMs approach MDI as a generative problem, and the score function $\nabla{\log{p(\boldsymbol{X})}}$ in the reverse process is replaced with conditional distribution $\nabla_{\boldsymbol{X}^{\text{(miss)}}}\log{p(\boldsymbol{X}^{\text{(miss)}}\vert\boldsymbol{X}^{\text{(obs)}} )}$~\cite{tashiro2021csdi}. Therefore, the challenge in DM-based MDI is to obtain an estimation $\nabla_{\boldsymbol{X}^{\text{(miss)}}}\log\hat{p}(\boldsymbol{X}^{\text{(miss)}}\vert\boldsymbol{X}^{\text{(obs)}} )$ that effectively approximates $\nabla_{\boldsymbol{X}^{\text{(miss)}}}\log{p(\boldsymbol{X}^{\text{(miss)}}\vert\boldsymbol{X}^{\text{(obs)}} )}$. However, constructing model training remains challenging due to ground truth $\boldsymbol{X}^{\text{(miss)}}$ is unknown. To alleviate this issue, previous DM-based MDI approaches necessitate the design of a mask matrix to obscure portions of the observational data; despite practical efficacy, the selection of mask mechanism, which determines the effectiveness of $\nabla\log\hat{p}(\boldsymbol{X}^{\text{(miss)}}\vert\boldsymbol{X}^{\text{(obs)}})$, remain challenges and hence may result in training difficulty.

\subsection{Wasserstein Gradient Flow}\label{subsec:GradinWasserstein}
The Wasserstein space $\mathcal{P}_2(\mathbb{R}^\mathrm{D})\coloneqq\{p\in\mathcal{P}_2(\mathbb{R}^\mathrm{D}): \int{\Vert x \Vert^2\mathrm{d}p(x)} <\infty \}$ is a metric space where distances between probability distributions are quantified using the 2-Wasserstein distance, defined as $(\mathop{\inf}_{\gamma\in\Gamma(p, q)}\int{\Vert x - y \Vert^2\mathrm{d}\gamma(x,y)})^{\frac{1}{2}}$. In this space, gradient flows resemble the steepest descent curves similar to those in classical Euclidean spaces. Specifically, for a cost functional $\mathcal{F}_{\text{cost}}: \mathcal{P}_2(\mathbb{R}^\mathrm{D})\rightarrow\mathbb{R}$, a gradient flow in Wasserstein space is an absolute continuous trajectory $(q_\tau)_{\tau>0}$ that seeks to minimize $\mathcal{F}_{\text{cost}}$ as efficiently as possible, as described in \cite{santambrogio2017euclidean}. This dynamic process is governed by the celebrated continuity equation $\frac{\partial q_\tau}{\partial \tau}=-\nabla\cdot(v_\tau q_\tau)$, where $v_\tau: \mathbb{R}^{\mathrm{D}}\rightarrow\mathbb{R}^{\mathrm{D}}$ is a time-dependent velocity field~\cite{ambrosio2005gradient}. Additionally, the evolution of sample $\boldsymbol{X}$ over time $\tau$ in $\mathcal{P}_2(\mathbb{R}^\mathrm{D})$ can be delineated by the ordinary differential equation (ODE) 
expressed as $\frac{\mathrm{d}\boldsymbol{X}}{\mathrm{d}\tau}=v_\tau(\boldsymbol{X})$.

\vspace{-0.1cm}
\section{Proposed Approach}
This section proposes our \uline{K}ernelized \uline{N}egative \uline{E}ntropy-regularized \uline{W}asserstein gradient flow \uline{Imp}utation (KnewImp) approach. 
We first define the MDI's objective oriented towards maximization and unify the DM-based MDI approaches within WGF framework. In this procedure, we prove that addressing MDI through the generative processes of DMs maximizes a diversification-promoting upper bound of MDI's objective. 
Building on this foundational analysis, we introduce a novel diversification-discouraging negative entropy-regularized (NER) cost functional that acts as a lower bound for the MDI's objective, ensuring precise imputation through maximizing the MDI's lower bound.
Based on this, we then develop the imputation procedure that features a closed-form, easily implementable expression within the WGF framework and RKHS.
Further, we establish that our NER functional, associated with the conditional distribution $p(\boldsymbol{X}^{\text{(miss)}}\vert\boldsymbol{X}^{\text{(obs)}} )$, is equivalent to another functional concerned with the joint distribution $p(\boldsymbol{X}^{\text{(miss)}},\boldsymbol{X}^{\text{(obs)}} )$, adjusted by a constant. This equivalence maintains the same velocity field but effectively eliminates the need for a mask matrix in the training stage. Finally, we conclude this section by outlining the procedure of KnewImp approach that encapsulates these innovations.


\subsection{Unifying DM-based MDI within WGF framework}\label{subsec:DiffMDIRelationship}

Drawing from previous works~\cite{mattei2019miwae,tashiro2021csdi}, the objective function for the MDI task can be defined as:
\begin{equation}\label{eq:impProblem}
  \boldsymbol{X}^{(\text{imp})} = \mathop{\arg\max}_{\boldsymbol{X}^{(\text{miss})}} \indent \log{ {p}( \boldsymbol{X}^{(\text{miss})}\vert \boldsymbol{X}^{(\text{obs})}) }.
\end{equation}
From the perspective of generative models, $\boldsymbol{X}^{(\text{imp})}$ is considered as samples drawn from a certain distribution $r(\boldsymbol{X}^{(\text{miss})})$ (we name it `proposal distribution'), and results in the following reformulation:
\begin{equation}\label{eq:objectiveMDIProcess}
  \mathop{\arg\max}_{\boldsymbol{X}^{(\text{miss})}} \indent{\log{ {p}( \boldsymbol{X}^{(\text{miss})}\vert \boldsymbol{X}^{(\text{obs})}) }},\boldsymbol{X}^{(\text{miss})}\sim r(\boldsymbol{X}^{(\text{miss})}) \Rightarrow \mathop{\arg\max}_{\boldsymbol{X}^{(\text{miss})}} \indent \mathbb{E}_{ 
    r(\boldsymbol{X}^{(\text{miss})})}[{\log{ {p}( \boldsymbol{X}^{(\text{miss})}\vert \boldsymbol{X}^{(\text{obs})}) }}] .
\end{equation}
Consequently, DM-based MDI approaches address this optimization problem by generating samples from the estimated conditional score function: $\nabla_{\boldsymbol{X}^{(\text{miss})}}\log{\hat{p}(\boldsymbol{X}^{(\text{miss})}\vert \boldsymbol{X}^{(\text{obs})})}$.


Note that, according to Section~\ref{subsec:diffusionKnowledge}, the generative process of DMs satisfies the FPK equation, a specific instance of the continuity equation that underpins the WGF according to Section~\ref{subsec:GradinWasserstein}. Meanwhile, WGF framework is central to functional optimization, indicating that the divergence of the objective between MDI and the DM-based MDI approaches can be effectively analyzed within the WGF framework. In support of this, we first give the following proposition, which elucidates the relationship between the objective of DM-based MDI approaches and MDI~\footnote{according to reference~\cite{song2020score}, we mainly consdier the Variance Preserving-SDE (VP-SDE), Variance Exploding-SDE (VE-SDE), and sub-Variance Preserving-SDE (sub-VP-SDE)}:
\begin{proposition}\label{prop:inEffectiveSampling}
Within WGF framework, DM-based MDI approaches can be viewed as finding the imputed values $\boldsymbol{X}^{(\text{imp})}$ that maximize the following objective:
\begin{equation}\label{eq:samplingFunctional}
  \mathop{\arg\max}_{\boldsymbol{X}^{(\text{miss})}} \quad  \mathbb{E}_{ r(\boldsymbol{X}^{(\text{miss})})}[{\log{ \hat{p}( \boldsymbol{X}^{(\text{miss})}\vert \boldsymbol{X}^{(\text{obs})}) }}] +\psi(\boldsymbol{X}^{(\text{miss})}) + \text{const},
\end{equation}
where const is the abbreviation of constant, and $\psi(\boldsymbol{X}^{(\text{miss})})$ is a scalar function determined by the type of SDE underlying the DMs.
\begin{itemize}[leftmargin=*]
  \item{\textbf{VP-SDE:} $\psi(\boldsymbol{X}^{(\text{miss})})=\mathbb{E}_{ r(\boldsymbol{X}^{(\text{miss})})}\{\frac{1}{4}[\boldsymbol{X}^{\text{(miss)}}]^\top[\boldsymbol{X}^{\text{(miss)}}]  -\frac{1}{2}\log{r(\boldsymbol{X}^{\text{(miss)}})} \}$}
  \item{\textbf{VE-SDE:} $\psi(\boldsymbol{X}^{(\text{miss})})=\mathbb{E}_{r(\boldsymbol{X}^{(\text{miss})})}\{ -\frac{1}{2}\log{r(\boldsymbol{X}^{\text{(miss)}})} \}$}
  \item{\textbf{sub-VP-SDE:} $\psi(\boldsymbol{X}^{(\text{miss})})=\mathbb{E}_{ r(\boldsymbol{X}^{(\text{miss})})}\{  \frac{1}{4\gamma_\tau}[\boldsymbol{X}^{\text{(miss)}}]^\top[\boldsymbol{X}^{\text{(miss)}}] -\frac{1}{2}\log{r(\boldsymbol{X}^{\text{(miss)}})}\}$, where $\gamma_\tau$ is determined by noise scale $\beta_\tau$: $\gamma_\tau\coloneqq(1-\exp(-2\int_{0}^{\tau}{\beta_s\mathrm{d}s}))>0,0<\beta_1<\beta_2<\dots<\beta_{\mathrm{T}}<1$.}
\end{itemize}
It is important to note that in DMs, the condition $\psi(\boldsymbol{X}^{(\text{miss})}) \ge 0$ consistently holds. This assertion is supported by the fact that the inner product $[\boldsymbol{X}^{\text{(miss)}}]^\top[\boldsymbol{X}^{\text{(miss)}}] \ge 0$, and the entropy function defined as $\mathbb{H}[r(\boldsymbol{X}^{\text{(miss)}})] \coloneqq -\int{r(\boldsymbol{X}^{\text{(miss)}}) \log{r(\boldsymbol{X}^{\text{(miss)}})} \mathrm{d}\boldsymbol{X}^{\text{(miss)}}}$ is also non-negative.
\end{proposition}

Based on this proposition, it becomes evident that `inaccurate imputation' issue may arise from the misalignment in the optimization objectives: By comparing~\Cref{eq:objectiveMDIProcess,eq:samplingFunctional}, we observe that DM-based MDI methods are optimizing an upper bound of the MDI. Additionally, it is important to note that the gaps involving inner products and entropy, which promote maximization, inherently encourage sample diversification~\cite{wang2019nonlinear}. This diversification is fundamentally at odds with the precision required in MDI tasks. To achieve a more accurate imputation, it is crucial to reformulate $\psi(\boldsymbol{X}^{(\text{miss})})$ to ensure that $\psi(\boldsymbol{X}^{(\text{miss})}) \leq 0$, thereby aligning the objective of the imputation procedure with the accuracy-oriented goal of MDI.

\subsection{Negative Entropy Regularized \& Closed-form Velocity Field Expression}\label{subsec:RegularClosedIter}
Based on previous subsection, we adopt negative entropy as $\psi(\boldsymbol{X}^{(\text{miss})})$ intuitively:
\begin{equation}\label{eq:NEGReg}
  \psi(\boldsymbol{X}^{(\text{miss})}) = \lambda\int{r(\boldsymbol{X}^{\text{(miss)}})\log{r(\boldsymbol{X}^{\text{(miss)}})}\mathrm{d}\boldsymbol{X}^{\text{(miss)}}}=-\lambda \mathbb{H}[r(\boldsymbol{X}^{\text{(miss)}})], \lambda > 0,
\end{equation}
where positive constant $\lambda$ is a predefined regularization strength, and consequently we can formulate our NER cost functional for MDI as follows:
\begin{equation}\label{eq:NEGFunctional}
  \mathcal{F}_{\text{NER}} \coloneqq \mathbb{E}_{ r(\boldsymbol{X}^{(\text{miss})})}[{\log{ \hat{p}( \boldsymbol{X}^{(\text{miss})}\vert \boldsymbol{X}^{(\text{obs})}) }}] -\lambda \mathbb{H}[r(\boldsymbol{X}^{\text{(miss)}})]. 
\end{equation}
From a theoretical perspective, the NER term serves a critical role: The optimal $r(\boldsymbol{X}^{(\text{miss})})$ inherently allows for infinite possibilities, potentially leading to a diversification of imputed values that could adversely affect accuracy.
By incorporating this regularization term, we not only keep the objective direction but also reduce the diversification of samples by eliminating the entropy of $r(\boldsymbol{X}^{(\text{miss})})$, which may result in an improvement in accuracy. 

As demonstrated in~\cite{santambrogio2017euclidean,ijcai2022p679}, we can directly incorporate~\Cref{eq:NEGFunctional} into the WGF framework and result in the following velocity field that drives the ODE in~\Cref{subsec:GradinWasserstein}:
\begin{equation}\label{eq:velocityNEG}
  v_\tau = -\nabla_{\boldsymbol{X}^{\text{(miss)}}}\frac{\delta  (- \mathcal{F}_{\text{NER}}) }{\delta r(\boldsymbol{X}^{\text{(miss)}})}=[\nabla_{\boldsymbol{X}^{\text{(miss)}}}\log{ \hat{p}( \boldsymbol{X}^{(\text{miss})}\vert \boldsymbol{X}^{(\text{obs})}) } +\lambda\nabla_{\boldsymbol{X}^{\text{(miss)}}}\log{r(\boldsymbol{X}^{\text{(miss)}})} ],
\end{equation}
where $\frac{\delta \mathcal{F}_{\text{NER}} }{\delta r(\boldsymbol{X}^{\text{(miss)}})}$ represents the first variation of $\mathcal{F}_{\text{NER}}$ with respect to $r(\boldsymbol{X}^{\text{(miss)}})$. However, implementing this velocity field poses substantial challenges within both the ODE framework of the WGF and the SDE contexts.~\footnote{Detailed analysis regarding the implementation challenges is provided in~\Cref{appendix-subsec:impleDiff}.} To alleviate this issue, 
%
we attempt to derive the expressions for model implementation based on the steepest ascent direction of functional gradient~\cite{friedman2001greedy,8744312,dong2022particle}. On this basis, we first derive the following ODE for the evolution of $\mathcal{F}_{\text{NER}}$ along time $\tau$:
\begin{proposition}\label{prop:simulationLoss}
  The evolution of $\mathcal{F}_{\text{NER}}$ along $\tau$ can be characterized by the following ODE, assuming that the boundary condition $\mathbb{E}_{ r(\boldsymbol{X}^{(\text{miss})},\tau)}\{\nabla_{\boldsymbol{X}^{(\text{miss})}}\cdot [u(\boldsymbol{X}^{\text{(miss)}}, \tau) \log{\hat{p}( \boldsymbol{X}^{(\text{miss})} \vert \boldsymbol{X}^{(\text{obs})})} ]\}=0$ is satisfied for the velocity field $u(\boldsymbol{X}^{\text{(miss)}}, \tau)$:
  \begin{equation}\label{eq:iterativeLossImputation}
    \frac{\mathrm{d} \mathcal{F}_{\text{NER}}}{\mathrm{d} \tau} = \mathbb{E}_{ r(\boldsymbol{X}^{(\text{miss})},\tau)}[ 
    u^{\top}(\boldsymbol{X}^{\text{(miss)}}, \tau)\nabla_{\boldsymbol{X}^{(\text{miss})}} \log{\hat{p}( \boldsymbol{X}^{(\text{miss})} \vert \boldsymbol{X}^{(\text{obs})})}  -\lambda\nabla_{\boldsymbol{X}^{(\text{miss})}} \cdot u(\boldsymbol{X}^{\text{(miss)}}, \tau)
    ].
  \end{equation}
This boundary condition is achievable, for instance, when $\hat{p}( \boldsymbol{X}^{(\text{miss})} \vert \boldsymbol{X}^{(\text{obs})})$ is bounded, and the limit of the velocity field as the norm of $\boldsymbol{X}^{(\text{miss})}$ approaches zero is zero ($\lim_{\Vert \boldsymbol{X}^{\text{(miss)}}\Vert\rightarrow 0}{u(\boldsymbol{X}^{\text{(miss)}}, \tau)}=0$).
\end{proposition}
Despite the clarity provided by~\Cref{eq:iterativeLossImputation}, practical model implementation remains a significant challenge due to the potential variability in the velocity field $u(\boldsymbol{X}^{\text{(miss)}}, \tau)$. To address this, we propose constraining the velocity within some specified function class $\mathcal{V}$~\cite{liu2019understanding,wild2024rigorous,dong2022particle}, such that $u(\boldsymbol{X}^{\text{(miss)}}, \tau) \in \mathcal{V}$, which allows us to explore the steepest ascent direction systematically. \textit{To obtain a closed-form and easily implementable expression}, 
we choose RKHS denoted by $\mathcal{H}$ with RKHS norm $\Vert \cdot\Vert_{\mathcal{H}}$ to represent $\mathcal{V}$. On this basis, we have the following proposition for the expression of $ u(\boldsymbol{X}^{\text{(miss)}}, \tau)$: 
\begin{proposition}\label{prop:steinMap}
When the velocity field $u(\boldsymbol{X}^{\text{(miss)}}, \tau)$ is regularized by RKHS norm, the problem of finding the steepest gradient ascent direction can be formulated as follows:
\begin{equation}\label{eq:objectiveRKHSReg}
  u(\boldsymbol{X}^{\text{(miss)}}, \tau) = \mathop{\arg\max}_{v(\boldsymbol{X}^{\text{(miss)}}, \tau)\in\mathcal{H}^d}
  \begin{aligned}
  & \{\mathbb{E}_{ r(\boldsymbol{X}^{(\text{miss})},\tau)}[ 
    v^{\top}(\boldsymbol{X}^{\text{(miss)}}, \tau)\nabla_{\boldsymbol{X}^{(\text{miss})}} \log{\hat{p}( \boldsymbol{X}^{(\text{miss})} \vert \boldsymbol{X}^{(\text{obs})})}  \\
  & \quad \quad- \lambda\nabla_{\boldsymbol{X}^{(\text{miss})}} \cdot v(\boldsymbol{X}^{\text{(miss)}}, \tau)
      ]\}-\frac{1}{2}\Vert v(\boldsymbol{X}^{\text{(miss)}}, \tau)\Vert_{\mathcal{H}}^2.
  \end{aligned}
  \end{equation}
  The corresponding optimal solution is given by:
\begin{equation}\label{eq:RKHSVelocityField}
u(\boldsymbol{X}^{\text{(miss)}}, \tau) = \mathbb{E}_{r(\tilde{\boldsymbol{{X}}}^{(\text{miss})},\tau)}\left\{ 
      \begin{aligned}  
      & { -\lambda \nabla_{{\tilde{\boldsymbol{X}}}^{(\text{miss})}} \mathcal{K}(\boldsymbol{X}^{(\text{miss})}  , \tilde{\boldsymbol{X}}^{(\text{miss})})  } \\
      &\quad \quad+ [\nabla_{{\tilde{\boldsymbol{X}}}^{(\text{miss})}}\log{\hat{p}( \tilde{\boldsymbol{X}}^{(\text{miss})} \vert \boldsymbol{X}^{(\text{obs})})}]^\top\mathcal{K}(\boldsymbol{X}^{(\text{miss})}  , \tilde{\boldsymbol{X}}^{(\text{miss})}) 
      \end{aligned} 
      \right\},
\end{equation}
where $\mathcal{K}(\boldsymbol{X}^{(\text{miss})}  , \tilde{\boldsymbol{X}}^{(\text{miss})})$ is kernel function.
\end{proposition}
Importantly, since the missing value dimension is undefined, we did not specify the type signature of $\mathcal{K}(\boldsymbol{X}^{(\text{miss})}  , \tilde{\boldsymbol{X}}^{(\text{miss})})$, and the expectation term $\mathbb{E}_{r(\tilde{\boldsymbol{{X}}}^{(\text{miss})},\tau)}$ can be efficiently estimated using Monte Carlo approximation. Leveraging this approach, the velocity field as outlined in~\Cref{eq:RKHSVelocityField} does not require explicit computation of proposal distribution $r(\boldsymbol{{X}}^{(\text{miss})},\tau)$. Consequently, we finally derive an easily implementable and closed-form imputation procedure for our KnewImp approach. 

\subsection{Modeling ${p}( \boldsymbol{X}^{(\text{miss})}\vert \boldsymbol{X}^{(\text{obs})})$ by ${p}( \boldsymbol{X}^{(\text{miss})}, \boldsymbol{X}^{(\text{obs})})$}\label{subsec:JointDistModeling}
As discussed in the previous subsection, accurately defining the conditional distribution ${p}( \boldsymbol{X}^{(\text{miss})}\vert \boldsymbol{X}^{(\text{obs})})$ is crucial for effectively simulating the ODE in~\Cref{eq:iterativeLossImputation} using the velocity field specified in \Cref{eq:RKHSVelocityField}. However, as previously noted, precise modeling of ${p}( \boldsymbol{X}^{(\text{miss})}\vert \boldsymbol{X}^{(\text{obs})})$ presents substantial challenges due to the diverse choices of masking matrices, which critically influence the efficacy of model training.

To circumvent this difficulty, a practical approach involves substituting the conditional distribution ${p}( \boldsymbol{X}^{(\text{miss})}\vert \boldsymbol{X}^{(\text{obs})})$ with the joint distribution ${p}( \boldsymbol{X}^{(\text{miss})}, \boldsymbol{X}^{(\text{obs})})$. By denoting $\boldsymbol{X}^{(\text{joint})} = (\boldsymbol{X}^{(\text{miss})}, \boldsymbol{X}^{(\text{obs})})$, we can accordingly redefine the velocity field $u$ as follows:

\begin{equation}\label{eq:jointVelocityField}
  {u}( \boldsymbol{X}^{\text{(joint)}}, \tau) = \mathbb{E}_{r({\tilde{\boldsymbol{X}}}^{\text{(joint)}},\tau)}\left\{ 
    \begin{aligned}  
    & { - \lambda \nabla_{  {\tilde{\boldsymbol{X}}^{\text{(miss)}}}} \mathcal{K}(  \boldsymbol{X}^{\text{(joint)}}  ,    {\tilde{\boldsymbol{X}}^{\text{(joint)}}})  }\\
    &\quad \quad+ [\nabla_{  {\tilde{\boldsymbol{X}}^{\text{(miss)}}}}\log{\hat{p}(  \tilde{\boldsymbol{X}}^{\text{(joint)}})}]^\top\mathcal{K}(  \boldsymbol{X}^{\text{(joint)}}  ,   \tilde{\boldsymbol{X}}^{\text{(joint)}}) 
    \end{aligned} 
    \right\},
\end{equation}
where $\mathcal{K}(\boldsymbol{X}^{(\text{joint})}  , \tilde{\boldsymbol{X}}^{(\text{joint})}):\mathbb{R}^\mathrm{D}\rightarrow\mathbb{R}^\mathrm{D}$ is kernel function, and we use radial basis function kernel defined as $\mathcal{K}(\boldsymbol{X}  , \tilde{\boldsymbol{X}})\coloneqq \exp(-\frac{\Vert \boldsymbol{X}  - \tilde{\boldsymbol{X}}\Vert^2}{2h^2})$ with bandwidth $h$ in this paper~\cite{zhou2008derivative,liu2016stein,10005854}. Based on this, $\nabla_{  {\boldsymbol{X}^{\text{(miss)}}}}\log{\hat{p}(  \boldsymbol{X}^{\text{(joint)}})}$ can be obtained according to the following equation: 
\begin{equation}\label{eq:CondByJoint}
    \nabla_{  {\tilde{\boldsymbol{X}}^{\text{(miss)}}}}\log{\hat{p}(  \tilde{\boldsymbol{X}}^{\text{(joint)}})} = \nabla_{  {\tilde{\boldsymbol{X}}^{\text{(joint)}}}}\log{\hat{p}(  \tilde{\boldsymbol{X}}^{\text{(joint)}})} \odot (\mathbbm{1}_{\mathrm{N} \times \mathrm{D}}-\boldsymbol{M}) +   0 \times \boldsymbol{M}.
\end{equation}
As such, a pertinent question arises: What's the relationship between \Cref{eq:jointVelocityField,eq:RKHSVelocityField}? Interestingly, these formulations are identical.  In light of this, the remainder of this subsection is dedicated to demonstrating that the velocity field associated with the cost functional in \Cref{eq:jointVelocityField} does not compromise the optimization of $\mathcal{F}_{\text{NER}}$. To support this assertion, we present the following proposition:
\begin{proposition}\label{prop:JointCondDist}
Assume that the proposal distribution $r({\boldsymbol{X}}^{\text{(joint)}})$ is factorized by $r({\boldsymbol{X}}^{\text{(joint)}})\coloneqq r({\boldsymbol{X}}^{\text{(miss)}}) p(\boldsymbol{X}^{\text{(obs)}})$. The cost functional associated with the joint distribution is defined by the following equation:
 \begin{equation}\label{eq:jointNERCostFunctional}
  \mathcal{F}_{\text{joint-NER}} \coloneqq \mathbb{E}_{ r({\boldsymbol{X}}^{\text{(joint)}})}[{\log{ \hat{p}( \boldsymbol{X}^{(\text{joint})}) }}] -\lambda\mathbb{H}[r({\boldsymbol{X}}^{\text{(joint)}})],
 \end{equation}
 which leads to the velocity field delineated in \Cref{eq:jointVelocityField} and establishes $\mathcal{F}_{\text{joint-NER}}$ as a lower bound for $\mathcal{F}_{\text{NER}}$, with the difference being a constant (i.e., $\mathcal{F}_{\text{joint-NER}} = \mathcal{F}_{\text{NER}} - \text{const},\text{const}\ge0$).
\end{proposition}

Based on this, the following corollary can be obtained:
\begin{corollary}\label{eq:velocityEqual}
The following equation holds:
\begin{equation}
  {u}( \boldsymbol{X}^{\text{(joint)}}, \tau) = u( \boldsymbol{X}^{\text{(miss)}}, \tau).
\end{equation}
\end{corollary}

Building on these foundations, the imputed value is obtained by simulating the following ODE:
\begin{equation}
  \frac{\mathrm{d}\boldsymbol{X}^{\text{(miss)}}}{\mathrm{d}\tau} =  {u}( \boldsymbol{X}^{\text{(joint)}}, \tau).
\end{equation}
For simplicity, in this paper, we simulate this ODE by forward Euler's method with step size $\eta$. \footnote{please see~\Cref{appendix-forwardEuler} for more detailed information about forward Euler's method.}

To date, our primary objective has been to determine the estimation of score function $\nabla_{{\boldsymbol{X}}^{\text{(joint)}}}\log\hat{p}({\boldsymbol{X}}^{\text{(joint)}})$. To achieve this, we employ Denoising Score Matching (DSM)~\cite{hyvarinen2005estimation,vincent2011connection} to parameterize $\nabla_{{\boldsymbol{X}}^{\text{(joint)}}}\log\hat{p}({\boldsymbol{X}}^{\text{(joint)}})$ using a neural network with $\theta$ as parameter set. We design the learning objective to minimize the discrepancy between the actual score and the model's predicted score after introducing Gaussian noise to the clean ${\boldsymbol{X}}^{\text{(joint)}}$ as $\hat{\boldsymbol{X}}^{\text{(joint)}}$ :
\begin{equation}
  \mathcal{L}_{\text{DSM}}\coloneqq\frac{1}{2}\mathbb{E}_{q_{\sigma}(\hat{\boldsymbol{X}}^{\text{(joint)}}\vert {\boldsymbol{X}}^{\text{(joint)}})}[\Vert 
  \nabla_{{\hat{\boldsymbol{X}}}^{\text{(joint)}}}\log\hat{p}({\boldsymbol{X}}^{\text{(joint)}})
  - \nabla_{\hat{\boldsymbol{X}}^{\text{(joint)}}} \log{q_{\sigma}(\hat{\boldsymbol{X}}^{\text{(joint)}}\vert {\boldsymbol{X}}^{\text{(joint)}})}  \Vert^2].
\end{equation}
Notably, $\sigma$ is variance scale, $\hat{\boldsymbol{X}}^{\text{(joint)}}$ is obtained by $\hat{\boldsymbol{X}}^{\text{(joint)}}={\boldsymbol{X}}^{\text{(joint)}} + \epsilon,\epsilon\sim\mathcal{N}(\mathbf{0}, \sigma^2\mathbf{I})$, and $ \nabla_{\hat{\boldsymbol{X}}^{\text{(joint)}}} \log{q_{\sigma}(\hat{\boldsymbol{X}}^{\text{(joint)}}\vert {\boldsymbol{X}}^{\text{(joint)}})} =-\frac{\hat{\boldsymbol{X}}^{\text{(joint)}} -\boldsymbol{X}^{\text{(joint)}} }{\sigma^2}$. Once $\nabla_{{{\boldsymbol{X}}}^{\text{(joint)}}}\log\hat{p}({\boldsymbol{X}}^{\text{(joint)}})$ is trained, we can obtain the imputation value by simulating the differential equation based on~\Cref{eq:jointVelocityField}.
\subsection{Overall Architecture of KnewImp}\label{subsec:OverallArchitecture}
Fig.~\ref{fig:illustration} presents the architecture of our KnewImp approach, which consists of two parts namely `Impute' and `Estimate'. The `Impute' part alleviates the missing data imputation as an ODE simulation problem within WGF framework, and the imputed matrix is obtained by simulating the velocity field as per~\Cref{eq:jointVelocityField}. Meanwhile, since the velocity field requires the computation of the score function of the joint data, the `Estimate' part serves for estimating the score function. By alternatively repeating these two parts, we can finally obtain the imputed value. To better delineate the KnewImp approach, we summarize the corresponding algorithms in~\Cref{appendix-Algorithm}. 
\begin{figure}[htbp]
  \centering
  \includegraphics[width=1.0\textwidth]{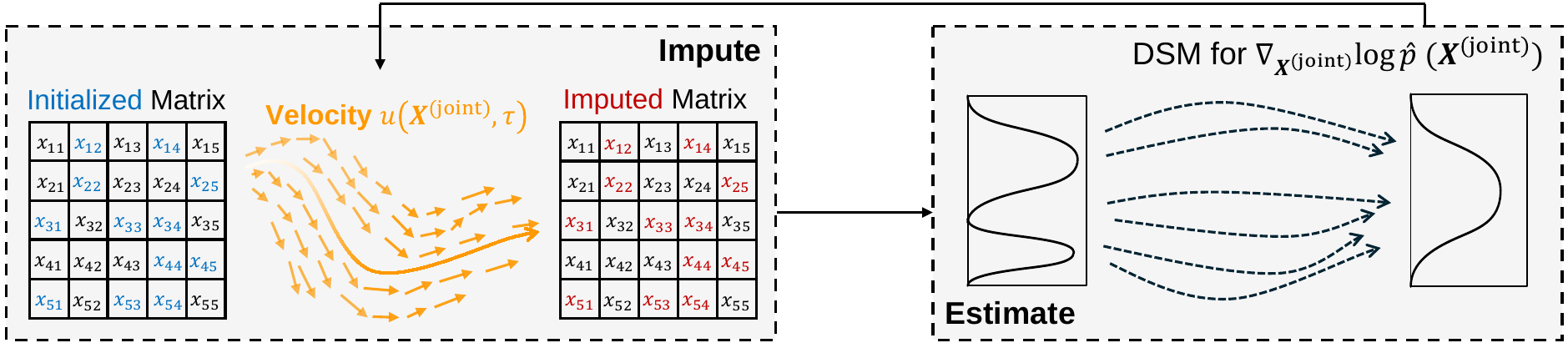} 
\vspace{-0.4cm}
  \caption{The illustration of KnewImp. The left part indicates we impute the missing value by WGF, and the right part indicates we use DSM to estimate $\log{p(\boldsymbol{X}^{\text{(miss)}})}$.}
  \label{fig:illustration}
  \vspace{-0.2cm}
\end{figure}
\section{Experiments}
\subsection{Experimental Setup}
\textbf{Datasets:} Eight real-world datasets from \href{https://archive.ics.uci.edu}{UCI repository} are chosen to validate the efficacy of our KnewImp approach. Detailed information for these datasets and the missing scenario simulation method is provided in Appendix~\ref{appendix-subsec:missingSimulation}.

\textbf{Baselines:} For fairness, we mainly consider the following models as baseline models: DMs-based approaches: conditional score-based diffusion models for Tabular Data (CSDI\_T)~\cite{tashiro2021csdi}, MissDiff~\cite{ouyang2023missdiff}; Non-DMs: Optimal Transport Imputer (Sink)~\cite{muzellec2020missing}, Transform Distribution Matching (TDM, state-of-the-art)~\cite{zhao2023transformed},  Generative Adversarial Imputation Nets (GAIN)~\cite{yoon2018gain}, Missing Data Importance-Weighted Autoencoder (MIWAE)~\cite{mattei2019miwae}, 
and  Missing data Imputation Refinement And Causal
LEarning (MIRACLE)~\cite{kyono2021miracle}. Concerning experimental details are given in~\Cref{appendix-TrainingProtocols}.

\textbf{Evaluation Metric:} We choose the mean absolute error (abbreviated to `MAE') and squared Wasserstein distance (abbreviated to `Wass') as evaluation metrics.\footnote{Detailed information about these two metrics is given in~\Cref{appendix-subsec:evaluationProtocols}. We also report concerning results under the MNAR scenario in~\Cref{appendix-sec:additionalEmpiricalEvidence} for completeness. 
}


\subsection{Baseline Comparison Results}
Baseline comparison results are given in~\Cref{tab:ComparisonResults}, and the following observations can be given
\begin{itemize}[leftmargin=*]
  \item{Models with neural architectures such as MIRACLE, MIWAE, and TDM demonstrate superior performance compared to models lacking such architectures. This observation suggests that integrating neural networks into MDI tasks can significantly enhance model performance. }
  \item{DM-based imputation approaches generally perform worse than other MDI methods. This outcome indicates that despite the incorporation of complex nonlinear neural architectures to boost performance, employing diversification-oriented generative approaches may not align well with the precision requirements of MDI tasks.}
  \item{Our proposed KnewImp method consistently ranks as the best or second-best across most comparisons. Notably, KnewImp significantly outperforms other DM-based MDI approaches, underscoring the effectiveness of our analytical enhancements and innovations in~\Cref{subsec:DiffMDIRelationship,subsec:RegularClosedIter,subsec:JointDistModeling}.}
\end{itemize}
\flushleft
\begin{table}[htbp]
  \caption{Performance of MAE and Wass metrics at 30\% missing rate, and `$^*$' marks that KnewImp outperforms significantly at $p$-value < 0.05 over paired samples $t$-test. Best results are \textbf{bolded} and the second best results are \uline{underliend}. Other results like standard deviation are given in appendix.}
  \label{tab:ComparisonResults}
  \resizebox{1.0\linewidth}{!}{
\begin{tabular}{c|l|ll|ll|ll|ll|ll|ll|ll|ll}
\toprule
\multirow{2}{*}{Scenario} & \multicolumn{1}{c|}{\multirow{2}{*}{Model}} & \multicolumn{2}{c|}{BT} & \multicolumn{2}{c|}{BCD} & \multicolumn{2}{c|}{CC} & \multicolumn{2}{c|}{CBV} & \multicolumn{2}{c|}{IS} & \multicolumn{2}{c|}{PK} & \multicolumn{2}{c|}{QB} & \multicolumn{2}{c}{WQW} \\ \cmidrule{3-18} 
                          & \multicolumn{1}{c|}{}                       & MAE       & Wass       & MAE       & Wass        & MAE       & Wass       & MAE       & Wass        & MAE      & Wass        & MAE       & Wass       & MAE       & Wass        & MAE       & Wass        \\ \midrule
\multirow{8}{*}{MAR}      & CSDI\_T                                     & 0.93$^*$  & 3.44$^*$   & 0.92$^*$  & 18.2$^*$    & 0.85$^*$  & 2.82$^*$   & 0.81$^*$  & 3.86$^*$    & 0.70$^*$  & 16.9$^*$   & 0.99$^*$  & 15.9$^*$  & 0.65$^*$  & 20.1$^*$    & 0.77$^*$  & 4.13$^*$    \\
                          & MissDiff                                    & 0.85$^*$   & 2.20$^*$   & 0.91$^*$   & 16.5$^*$  & 0.87$^*$   & 1.59$^*$  & 0.83$^*$   & 3.87$^*$   & 0.72$^*$  & 13.3$^*$  & 0.92$^*$  & 17.1$^*$  & 0.63$^*$   & 26.3$^*$  & 0.75$^*$   & 6.88$^*$   \\
                          & GAIN                                        & 0.75$^*$  & 0.65$^*$   & 0.54$^*$  & 1.64$^*$    & 0.75$^*$  & 0.67$^*$   & 0.68$^*$  & 0.68$^*$    & 0.56$^*$ & 1.88$^*$    & \uline{0.59}$^*$  & \uline{1.90}$^*$    & 0.65$^*$  & 5.05$^*$    & 0.68$^*$  & 0.87$^*$    \\
                          & MIRACLE                                     & \uline{0.62}$^*$  & \uline{0.38}       & 0.55$^*$  & 1.92$^*$    & \textbf{0.43}      & \textbf{0.25}       & \uline{0.55}$^*$  & \uline{0.46}$^*$    & 3.39$^*$ & 35.1$^*$   & 4.14$^*$  & 34.1$^*$  & \uline{0.46}      & \textbf{2.87}$^*$    & \uline{0.51}$^*$  & \uline{0.56}        \\
                          & MIWAE                                       & 0.64      & 0.53       & \uline{0.52}$^*$  & \uline{1.54}$^*$    & 0.76$^*$  & 0.64$^*$   & 0.82$^*$  & 0.92$^*$    & \uline{0.50}$^*$  & \uline{1.87}$^*$    & 0.65$^*$  & 1.98$^*$   & 0.55$^*$  & 5.05$^*$    & 0.62$^*$  & 0.75$^*$    \\
                          & Sink                                        & 0.87$^*$  & 0.92$^*$   & 0.92$^*$  & 3.84$^*$    & 0.88$^*$  & 0.83$^*$   & 0.84$^*$  & 0.98$^*$    & 0.75$^*$ & 2.43$^*$    & 0.94$^*$  & 3.61$^*$   & 0.65$^*$  & 4.71$^*$    & 0.76$^*$  & 1.04$^*$    \\
                          & TDM                                         & 0.83$^*$  & 0.89$^*$   & 0.83$^*$  & 3.47$^*$    & 0.81$^*$  & 0.73$^*$   & 0.76$^*$  & 0.85$^*$    & 0.62$^*$ & 1.96$^*$    & 0.86$^*$  & 3.36$^*$   & 0.59$^*$  & 4.46$^*$    & 0.73$^*$  & 0.99$^*$    \\
                          & \textbf{KnewImp}                                     & \textbf{0.52}      & \textbf{0.38}       & \textbf{0.34}      & \textbf{0.82}        & \uline{0.35}      & \uline{0.25}       & \textbf{0.31}      & \textbf{0.20}         & \textbf{0.39}     & \textbf{1.31}        & \textbf{0.44}      & \textbf{1.21}       & \textbf{0.45}      & \uline{3.50}        & \textbf{0.46}      & \textbf{0.55}        \\ \midrule
\multirow{8}{*}{MCAR}     & CSDI\_T                                     & 0.73$^*$  & 1.93$^*$   & 0.73$^*$  & 15.5$^*$   & 0.85$^*$  & 2.71$^*$   & 0.83$^*$  & 3.79$^*$    & 0.76$^*$ & 15.2$^*$   & 0.72$^*$  & 12.4$^*$  & 0.57$^*$  & 19.9$^*$   & 0.78$^*$  & 4.11$^*$    \\
                          & MissDiff                                    & 0.72$^*$   & 1.62$^*$  & 0.73$^*$   & 14.4$^*$  & 0.84$^*$   & 1.23$^*$  & 0.82$^*$   & 3.31$^*$   & 0.75$^*$  & 13.01$^*$  & 0.71$^*$  & 14.1$^*$  & 0.56$^*$   & 19.7$^*$  & 0.76$^*$   & 4.95$^*$   \\
                          & GAIN                                        & 0.72$^*$  & 0.39$^*$   & \uline{0.38}$^*$  & \uline{1.41}$^*$    & 0.78$^*$  & 0.73$^*$   & 0.72$^*$  & 0.99$^*$    & \uline{0.57}$^*$ & \uline{3.72}$^*$    & \uline{0.46}$^*$  & \uline{1.70}        & 0.42$^*$  & \uline{3.62}        & 0.73$^*$  & 1.14$^*$    \\
                          & MIRACLE                                     & \uline{0.52}      & \textbf{0.15}$^*$   & 0.44$^*$  & 1.94$^*$    & \uline{0.53}$^*$  & \uline{0.35}       & \uline{0.61}$^*$  & \uline{0.72}$^*$    & 2.99$^*$ & 52.9$^*$   & 3.38$^*$  & 42.8$^*$  & \uline{0.35}      & \textbf{2.71}$^*$    & \uline{0.56}$^*$  & \textbf{0.75}        \\
                          & MIWAE                                       & 0.58$^*$  & 0.24       & 0.50$^*$   & 2.55$^*$    & 0.76$^*$  & 0.69$^*$   & 0.83$^*$  & 1.24$^*$    & 0.64$^*$ & 4.95$^*$    & 0.51$^*$  & 2.05$^*$   & 0.48$^*$  & 5.87$^*$    & 0.67$^*$  & 0.95$^*$    \\
                          & Sink                                        & 0.73$^*$  & 0.48$^*$   & 0.75$^*$  & 4.39$^*$    & 0.84$^*$  & 0.85$^*$   & 0.82$^*$  & 1.27$^*$    & 0.75$^*$ & 4.94$^*$    & 0.74$^*$  & 3.36$^*$   & 0.61$^*$  & 5.92$^*$    & 0.76$^*$  & 1.25$^*$    \\
                          & TDM                                         & 0.68$^*$  & 0.42$^*$   & 0.63$^*$  & 3.57$^*$    & 0.77$^*$  & 0.75$^*$   & 0.77$^*$  & 1.15$^*$    & 0.66$^*$ & 4.20$^*$     & 0.64$^*$  & 2.89$^*$   & 0.52$^*$  & 5.34$^*$    & 0.74$^*$  & 1.20$^*$     \\
                          & \textbf{KnewImp}                                     & \textbf{0.48}      & \uline{0.18}       & \textbf{0.25}      & \textbf{0.80}         & \textbf{0.47}      & \textbf{0.34}       & \textbf{0.42}      & \textbf{0.44}        & \textbf{0.44}     & \textbf{3.05}        & \textbf{0.32}      & \textbf{1.01}       & \textbf{0.34}      & 3.66        & \textbf{0.53}      & \uline{0.76}        \\ \bottomrule
\end{tabular}
  }
  \vspace{-0.1cm}
  \end{table}

\begin{table}[htbp]
\caption{Ablation Study Results with missing rate at 30\%, and `$^*$' marks that KnewImp outperforms significantly at $p$-value < 0.05 over paired samples $t$-test. Best results are \textbf{bolded}.}
\label{tab:AblationStudy}
  \resizebox{1.0\linewidth}{!}{
\begin{tabular}{c|l|l|ll|ll|ll|ll|ll|ll|ll|ll}
\toprule
\multirow{2}{*}{Missing} & \multicolumn{1}{c|}{\multirow{2}{*}{NER}} & \multicolumn{1}{c|}{\multirow{2}{*}{Joint}} & \multicolumn{2}{c|}{BT}              & \multicolumn{2}{c|}{BCD}             & \multicolumn{2}{c|}{CC}              & \multicolumn{2}{c|}{CBV}             & \multicolumn{2}{c|}{IS}              & \multicolumn{2}{c|}{PK}              & \multicolumn{2}{c|}{QB} & \multicolumn{2}{c}{WQW} \\ \cmidrule{4-19} 
                         & \multicolumn{1}{c|}{}                     & \multicolumn{1}{c|}{}                       & MAE      & \multicolumn{1}{l|}{Wass} & MAE      & \multicolumn{1}{l|}{Wass} & MAE      & \multicolumn{1}{l|}{Wass} & MAE      & \multicolumn{1}{l|}{Wass} & MAE      & \multicolumn{1}{l|}{Wass} & MAE      & \multicolumn{1}{l|}{Wass} & MAE        & Wass       & MAE        & Wass       \\ \midrule
\multirow{4}{*}{MAR}     & \XSolidBrush                                     & \XSolidBrush                                        & 0.96$^*$ & 3.82$^*$                  & 1.05$^*$ & 20.2$^*$                 & 1.04$^*$ & 5.47$^*$                  & 0.86$^*$ & 5.81$^*$                  & 0.67$^*$ &20.2$^*$                 & 1.06$^*$ & 15.6$^*$                 & 0.72$^*$   & 22.5$^*$  & 0.79$^*$   & 6.49$^*$   \\
                         & \XSolidBrush                                      & \Checkmark                                       & 0.54     & 0.42                      & 0.34     & 0.82                      & 0.61$^*$ & 0.40$^*$                   & 0.58$^*$ & 0.47$^*$                  & 0.43$^*$ & 1.34                      & 0.46$^*$ & 1.25$^*$                  & 0.47$^*$   & 3.56$^*$   & 0.55$^*$   & 0.64$^*$   \\
                         & \Checkmark                                      & \XSolidBrush                                       & 0.96$^*$ & 3.83$^*$                  & 1.05$^*$ & 20.3$^*$                 & 1.04$^*$ & 5.49$^*$                  & 0.86$^*$ & 5.83$^*$                  & 0.67$^*$ & 20.2$^*$                 & 1.06$^*$ & 15.7$^*$                 & 0.72$^*$   & 22.5$^*$  & 0.79$^*$   & 6.51$^*$   \\
                         & \Checkmark                                      & \Checkmark                                        & \textbf{0.52}     & \textbf{0.38}                      & \textbf{0.34}     & \textbf{0.82}                      & \textbf{0.35}     & \textbf{0.25}                      & \textbf{0.31}     & \textbf{0.20}                      & \textbf{0.39}     & \textbf{1.31}                      & \textbf{0.44}     & \textbf{1.21}                      & \textbf{0.45}       & \textbf{3.50}       & \textbf{0.46}       & \textbf{0.55}       \\ \midrule
\multirow{4}{*}{MCAR}    & \XSolidBrush                                     & \XSolidBrush                                        & 0.72$^*$ & 2.11$^*$                  & 0.74$^*$ & 16.7$^*$                 & 0.85$^*$ & 3.72$^*$                  & 0.83$^*$ & 5.22$^*$                  & 0.74$^*$ & 18.4$^*$                 & 0.71$^*$ & 12.7$^*$                 & 0.58$^*$   & 20.1$^*$  & 0.76$^*$   & 5.57$^*$   \\
                         & \XSolidBrush                                      & \Checkmark                                       & 0.52$^*$ & 0.17$^*$                  & 0.25     & 0.79                      & 0.62$^*$ & 0.46$^*$                  & 0.61$^*$ & 0.71$^*$                  & 0.46     & 3.05                      & 0.34     & 1.09                      & 0.36$^*$   & 3.74$^*$   & 0.58$^*$   & 0.82$^*$   \\
                         & \Checkmark                                      & \XSolidBrush                                       & 0.72$^*$ & 2.12$^*$                  & 0.73$^*$ & 16.8$^*$                 & 0.86$^*$ & 3.73$^*$                  & 0.83$^*$ & 5.24$^*$                  & 0.74$^*$ & 18.4$^*$                 & 0.71$^*$ & 12.8$^*$                 & 0.58$^*$   & 20.1$^*$  & 0.76$^*$   & 5.60$^*$    \\
                         & \Checkmark                                      & \Checkmark                                        & \textbf{0.48}     & \textbf{0.18}                      & \textbf{0.25}     & \textbf{0.80}                      & \textbf{0.47}     & \textbf{0.34}                      & \textbf{0.42}     & \textbf{0.44}                      & \textbf{0.44}     & \textbf{3.05}                      & \textbf{0.32}     & \textbf{1.01}                      & \textbf{0.34}       & \textbf{3.66}       & \textbf{0.53}       & \textbf{0.76}       \\ \bottomrule
\end{tabular}
  }
\vspace{-0.3cm}
\end{table}
\justifying
\subsection{Ablation Study Results}\label{subsec:abResults}
In this subsection, we conduct the ablation study to assess the contributions of two key components in our KnewImp approach: the NER term and the joint modeling strategy (referred to as 'Joint'). The results of this study are detailed in \Cref{tab:AblationStudy}. Analysis of the data between the second and last rows of~\Cref{tab:AblationStudy} reveals that, in the absence of the NER, the proposal distribution $r(\boldsymbol{X}^{\text{(miss)}})$ may become pathological, leading to diminished model performance. Additionally, when comparing results from the first, third, and last rows, it becomes evident that modeling the joint distribution directly, rather than inferring it from the conditional distribution, significantly enhances model performance. This finding underscores the effectiveness of the strategies we have implemented, as discussed in~\Cref{subsec:JointDistModeling}. Overall, the ablation study underscores the critical roles of both the NER term and the joint distribution learning strategy in promoting the performance of KnewImp.

\subsection{Sensitivity Analysis}\label{subsec:sensResults}
In this subsection, we analyze the impact of key hyperparameters within the KnewImp approach, including the bandwidth $h$ of the RBF kernel function, the hidden units $\mathrm{HU}_{\text{score}}$ in the score network, the weight $\lambda$ of the NER term, and the discretization step size $\eta$ for simulating the ODE defined in~\Cref{eq:jointVelocityField}. The profound influence of these hyperparameters on learning objectives and overall performance is substantiated by the experimental results presented in~\Cref{fig:sensResult}. Initially, we explore the effects of varying the bandwidth $h$. We observe that an increase in bandwidth correlates with a decrease in imputation accuracy. For instance, as the bandwidth increases from 0.5 to 2.0, the MAE and Wasserstein distance escalate from 0.35 and 0.25 to 0.82 and 0.74, respectively. This trend suggests that excessive bandwidth can lead to an over-smoothed velocity field, expanding the exploration space of the distribution $r(\boldsymbol{X}^{\text{(joint)}})$ excessively and failing to adequately `concentrate' this distribution, ultimately diminishing performance. Subsequently, we examine changes in the score network's hidden units. Increasing the hidden units from 256 to 512 appears to decrease imputation accuracy, likely due to overfitting issues associated with larger neural networks. Next, we adjust the strength of the NER term and find that increasing its intensity generally improves imputation accuracy. This supports the necessity of the NER term, further validating its effectiveness. Lastly, we investigate the discretization step size for the ODE. We find that accuracy initially increases with smaller step sizes but then decreases. This pattern is consistent with ODE simulation behavior, where smaller step sizes require longer to converge, potentially resulting in lower accuracy within a predefined time. Conversely, larger step sizes increase discretization errors, adversely affecting accuracy as well.

\begin{figure}[htbp]
  \vspace{-1.0cm}
  \centering
  \subfigure[MAR with 30\% missing rate at CC dataset.]{\includegraphics[width=0.48\linewidth]{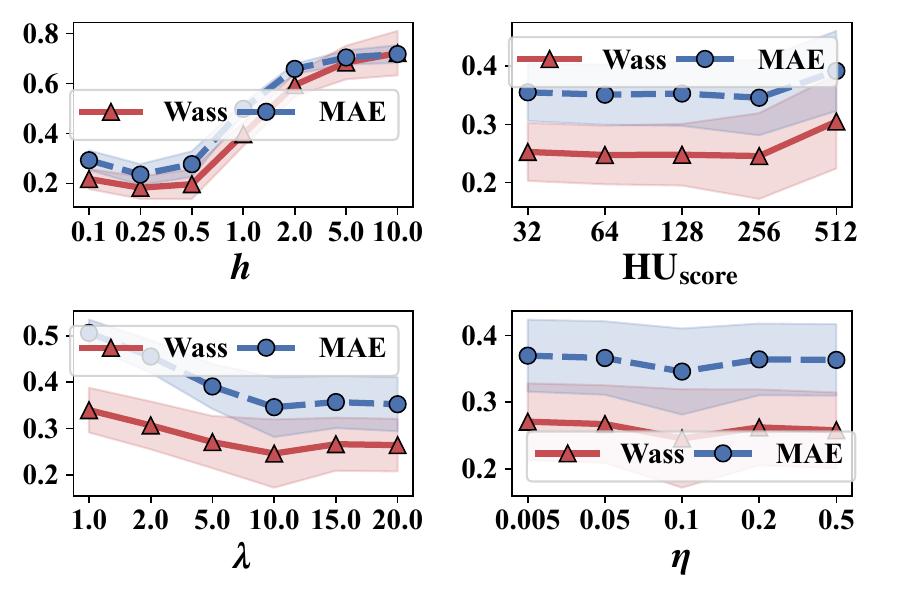}}
  \subfigure[MCAR with 30\% missing rate at CC dataset.]{\includegraphics[width=0.48\linewidth]{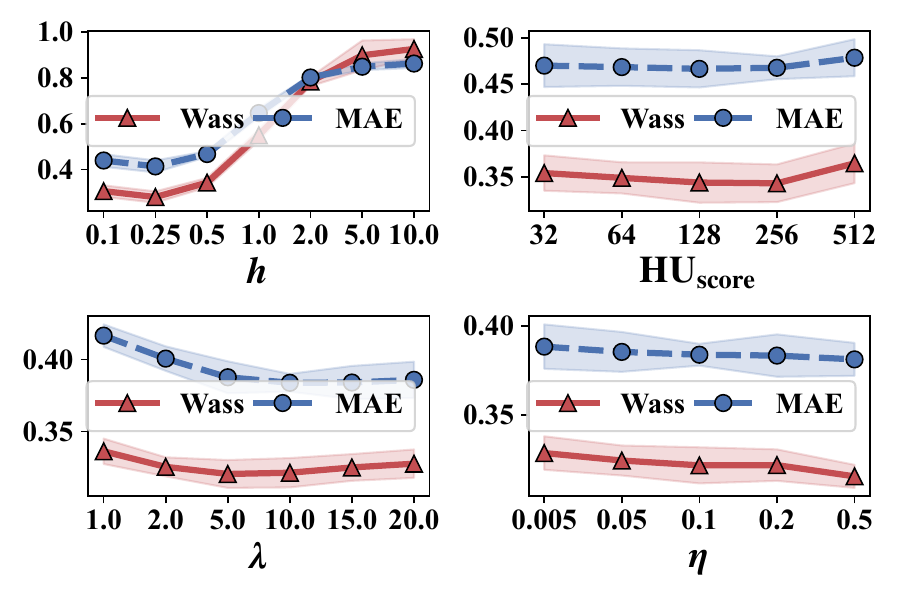}}
  \vspace{-0.3cm}
  \caption{Parameter sensitivity of KnewImp on bandwidth for kernel function ($h$), hidden unit of score network $\text{HU}_{\text{score}}$, NER weight $\lambda$, and discretization step $\eta$ for~\Cref{eq:jointVelocityField} on CC dataset. Mean values and one standard deviations from mean are represented by scatters and shaded area, respectively.
  }
  \label{fig:sensResult}
  \vspace{-0.5cm}
\end{figure}


\section{Related Works}
\subsection{Diffusion Models for Missing Data Imputation} 
The impressive ability of DMs to synthesize data has inspired extensive research into their application for MDI tasks~\cite{wang2024deep,yang2024survey}. Among the pioneering efforts, the Conditional Score-based Diffusion models for Imputation (CSDI)~\cite{tashiro2021csdi} was the first to adapt diffusion models for time-series MDI, substituting the score function with a conditional distribution and pioneering a novel model training strategy by masking parts of the observational data. Building on this, to address categorical data in tabular datasets, CSDI\_T~\cite{zheng2022diffusion} introduced an embedding layer within the feature extractor. To enhance inference efficiency, the conditional Schr\"odinger bridge method for probabilistic time series imputation proposed modeling the diffusion process as a Schr\"odinger bridge~\cite{chen2023provably}. Meanwhile, MissDiff~\cite{ouyang2023missdiff} sought to bypass the masking mechanism typically used during score function training by using missing data information as a mask matrix to improve the training process.

Despite these advancements from the perspective of feature extraction module~\cite{alcaraz2022diffusion,xu2023density}, loss function~\cite{ouyang2023missdiff}, and model inference approach~\cite{wang2023observed}, the reconciliation of the inherent diversity-seeking nature of DMs' generative processes and the accuracy-centric demands of MDI task remains underexplored. To our knowledge, this paper is the first to elucidate the relationship between DM generative processes and MDI tasks from an optimization perspective (\Cref{subsec:DiffMDIRelationship}). Based on these insights, we further propose our KnewImp approach, which prioritizes MDI accuracy (\Cref{subsec:RegularClosedIter}).


\subsection{Modeling Conditional Distribution by Joint Distribution}
Modeling conditional distribution as joint distribution remains an opening question and has a broad potential for application~\cite{yu2006supervised,chemseddine2024conditional,8892406}. Conditional sliced WGF~\cite{du2023nonparametric} first empirically validated that the velocity field of joint distribution and conditional distribution are identical when choosing sliced Wasserstein distance as cost functional. After that, reference~\cite{9944913} extended this relationship and derived the relationship between conditional and joint distribution in various discrepancy metrics like f-divergence, Wasserstein distance, and integral probability metrics. On this basis, reference~\cite{hagemann2023posterior} further theoretically proved the equivalence of velocity fields for conditional and joint distribution.

However, the objective of KnewImp does not belong to any kind of discrepancy metric~\cite{9944913}. The most similar discrepancy metric is Kullback Leiber (KL) divergence ($-\int{r(x)\frac{\log{r(x)}}{p(x)}}\mathrm{d}x$). Notably, KL divergence contains diversification-encouraging `positive' entropy $\mathbb{H}[r(x)]$ as the regularization term, and the regularization term in our study is diversification-discouraging `negative' entropy (i.e., $-\mathbb{H}[r(x)]$), and thus more than directly applying these results to our research is needed. On this basis, our theoretical contribution proves that this joint distribution modeling approach can still be applied when the functional is regularized by the negative entropy (\Cref{subsec:JointDistModeling}).

\section{Conclusions}
Existing DM-based MDI approaches face two critical issues that may hinder model performance: inaccurate imputation and difficult training. The first issue arises from the inherent conflict between the diversification-oriented generative process of DMs and the accuracy-focused demands of the MDI tasks. The second issue stems from the selection complexities of the masking matrix to facilitate conditional distribution between missing and observed data. To this end, this study initially applied the WGF framework to analyze DM-based MDI tasks, elucidating the relationship between the optimization objectives of DMs' generative process and the MDI task, and answered the reason for inaccurate imputation issue from the perspective of optimization. On this basis, we proposed our KnewImp approach by redesigning a novel effect cost functional and developing the corresponding DM-like imputation procedure within WGF and RKHS. Furthermore, we proved that another joint-distribution-related cost functional can result in the same imputation procedure, which naturally copes with the need for a masking matrix during model training. Finally, we conducted extensive experiments and demonstrated that the KnewImp approach can mitigate the abovementioned issues and achieve better performance than prevalent baseline models.

\newpage
\bibliographystyle{plainnat}
\bibliography{ref}

\newpage
\appendix
\counterwithin{equation}{section}
\setcounter{equation}{0}
\counterwithin{table}{section}
\setcounter{table}{0}
\counterwithin{figure}{section}
\setcounter{figure}{0}
\begin{appendices}
    \let\oldsection\section
    \renewcommand{\section}[1]{
        \oldsection{\textcolor{black}{#1}}
        \addcontentsline{apc}{section}{\protect\numberline{\thesection}\textcolor{blue}{#1}}
    }

    \let\oldsubsection\subsection
    \renewcommand{\subsection}[1]{
        \oldsubsection{\textcolor{black}{#1}}
        \addcontentsline{apc}{subsection}{\protect\numberline{\thesubsection}\textcolor{blue}{#1}}
    }


\addcontentsline{toc}{chapter}{Appendices} 
\listofappendix 

\newpage

\section{Detailed Preliminaries of Wasserstein Gradient Flow}\label{sec:PreliminariesDetailed}
In this section, we want to introduce the WGF technique and its application scenarios to better understand this paper. Before introduction, the following concepts are listed to better understand the WGF framework:
\begin{enumerate}[leftmargin=*]
\item{\textbf{Wasserstein Metric:} Let $\mathcal{P}_2(\mathbb{R}^{\mathrm{D}})$ represent the space of probability measures on $\mathbb{R}^{\mathrm{D}}$ that possess finite second moments. Formally, this is expressed as $\mathcal{P}_2(\mathbb{R}^{\mathrm{D}}) = \{\mu \in \mathcal{M}(\mathbb{R}^{\mathrm{D}}) \mid \int{\Vert x\Vert^2\mathrm{d}\mu(x)} < \infty \}$, where $\mathcal{M}(\mathbb{R}^{\mathrm{D}})$ denotes the set of all probability measures on $\mathbb{R}^{\mathrm{D}}$. Considering any two probability measures $\mu, \nu \in \mathcal{P}_2(\mathbb{R}^{\mathrm{D}})$, we define the Wasserstein-$p$ distance between them as follows:
\begin{equation}
\mathcal{W}_{p}=\left( \mathop{\inf}_{\pi\in\Gamma(\mu,\nu)} \int_{\mathbb{R}^{\mathrm{D}}\times \mathbb{R}^{\mathrm{D}}}{\Vert x-y \Vert^p\mathrm{d}\pi(x,y)} \right)^{\frac{1}{p}}.
\end{equation}
Here, $\Gamma(\mu, \nu)$ represents the collection of all joint distributions (couplings) between $\mu$ and $\nu$. For every joint distribution $\pi \in \Gamma(\mu, \nu)$, it holds that $\mu(x) = \int_{\mathbb{R}^{\mathrm{D}}} \pi(x, y) \, \mathrm{d}y$ and $\nu(y) = \int_{\mathbb{R}^{\mathrm{D}}} \pi(x, y) \, \mathrm{d}x$. The integral on the right-hand side encapsulates the transportation cost in the optimal transport (OT) problem, framed by Kantorovich's formulation, where $\pi^*$ denotes the optimal transportation plan.

Furthermore, leveraging Jensen's inequality facilitates demonstrating the monotonicity of the Wasserstein-$p$ distance, affirming that for $1 \le p \le q$, the relationship $\mathcal{W}_p(\mu, \nu) \le \mathcal{W}_q(\mu, \nu)$ invariably holds. Building on this principle, we can articulate the inner product within the measurable space $(\mathcal{P}_2(\mathbb{R}^{\mathrm{D}}), \mathcal{W})$ as delineated below:

\begin{equation}
\langle \mu, \nu \rangle_{\mu_\tau} = \int_{\mathbb{R}^{\mathrm{D}}}{\langle \mu, \nu \rangle_{\mathbb{R}^{\mathrm{D}}}\mathrm{d}\mu_\tau}
\end{equation}
  }
\item{\textbf{Gradient Flow in Wasserstein Space:} 
Consider a functional $\mathcal{F}$ associated with $\mu \in \mathcal{P}_2(\mathbb{R}^{\mathrm{D}})$. Our objective is to identify the optimal $\mu$ that minimizes $\mathcal{F}$:
\begin{equation}
  \mathop{\min}_{\mu \in \mathcal{P}_2(\mathbb{R}^{\mathrm{D}})} \mathcal{F}(\mu) +\text{const}.
\end{equation}
To facilitate the decrease of $\mathcal{F}(\mu)$, we introduce a velocity field $v_{\mu_\tau} : \mathbb{R}^{\mathrm{D}} \rightarrow \mathbb{R}^{\mathrm{D}}$ designed to expedite the reduction of $\mathcal{F}(\mu)$ as $\mu$ evolves under this field. Utilizing the chain rule yields:
\begin{equation}
  \frac{\mathrm{d}\mathcal{F}(\mu)}{\mathrm{d}\tau} = \int \left< \nabla \frac{\delta \mathcal{F}}{\delta \mu_\tau}, v_{\mu_\tau} \right> \, \mathrm{d}\mu_\tau,
\end{equation}
where $\delta$ represents the first variation operator. To ensure the decrease of $\mathcal{F}(\mu)$, i.e., $ \frac{\mathrm{d}\mathcal{F}(\mu)}{\mathrm{d}\tau} \le 0$, the velocity field is defined as:
\begin{equation}\label{eq:gradFlowEquation}
  v_{\mu_\tau} = - \nabla \frac{\delta \mathcal{F}}{\delta \mu_\tau}.
\end{equation}
The decline of $\mathcal{F}(\mu)$ aligns with the continuity equation:
\begin{equation}\label{eq:continuityEquation}
  \frac{\partial \mu_\tau}{\partial \tau} = -\nabla \cdot (\mu_\tau v_{\mu_\tau}).
\end{equation}
Hence, the continuity equation~\Cref{eq:continuityEquation}, coupled with the velocity field articulated in~\Cref{eq:gradFlowEquation}, is recognized as the \textit{Wasserstein Gradient Flow}, delineating the steepest descent in the Wasserstein space.
}
\item{\textbf{Simulation of WGF \& Sampling:} 
There are primarily two discretization techniques for the WGF: the forward scheme and the backward scheme.
\begin{itemize}[leftmargin=*]
\item{\textbf{Forward Scheme:}
The forward scheme applies gradient descent within the Wasserstein space to identify the direction of the steepest descent. For an energy functional $\mathcal{F}(\mu_\tau)$ with a specified step size $\eta$, the update rule in the forward scheme is formulated as:
\begin{equation}
  \mu_{t+1} = (\mathrm{Id} - \nabla\frac{\delta \mathcal{F}}{\delta \mu_\tau})_{\#}\mu_\tau,
\end{equation}
facilitating an intuitive and direct update mechanism that emulates the gradient flow in the Euclidean space but transposed into the Wasserstein space.
}
\item{\textbf{Backward Scheme:}
Conversely, the backward scheme, often referred to as the Jordan-Kinderlehrer-Otto (JKO) scheme~\cite{jordan1998variational}, represents a more implicit discretization approach. It defines the subsequent distribution $\mu_{\tau+1}$ by solving an optimization problem that balances the energy decrease and the transportation cost. This scheme is mathematically denoted as:
\begin{equation}
  \mu_{\tau+1} = \mathop{\arg\min}_{\mu\in\mathcal{P}_2(\mathbb{R}^{\mathrm{D}})} \mathcal{F}(\mu) + \frac{1}{2\eta}\mathcal{W}^2_2(\mu, \mu_\tau),
\end{equation}
thereby integrating the energy minimization and transport efficiency into a single variational problem that reflects the inherent structure of the Wasserstein space.
}
\end{itemize}
These schemes provide distinct yet complementary approaches to discretizing the dynamics defined by WGFs, offering different perspectives and tools for the analysis and computation of these flows.
}
\end{enumerate}
Leveraging the WGF framework, if we designate the functional $\mathcal{F}$ to be the KL divergence, it yields a particular formulation for the velocity field. 
\begin{equation}\label{eq:KLdivVariation}
  v_{\mu_\tau} = -\nabla\frac{\delta \mathbb{D}_{\text{KL}}(\mu_\tau\Vert p)}{\delta \mu_\tau} = \nabla\log{p} - \nabla\log{\mu_\tau}.
\end{equation}
On this basis, plug~\Cref{eq:KLdivVariation} into~\Cref{eq:continuityEquation}, we can get the following PDE:
\begin{equation}\label{eq:KLVelocity}
  \frac{\partial \mu_\tau}{\partial \tau} = -\nabla \cdot (\mu_\tau \nabla\log{p}) + \nabla\cdot\nabla\mu_\tau.
\end{equation}
According to Theorem 5.4 in reference~\cite{sarkka2019applied}, denote the random sample from distribution $p$ as $x$, we can obtain the following SDE called Langevin equation~\cite{welling2011bayesian} for implementing this gradient flow easily:
\begin{equation}\label{eq:LangevinDynamicsExpress}
  \mathrm{d}x = \nabla_{x}{\log{p}(x)}\mathrm{d}\tau +\sqrt{2}\mathrm{d}W_\tau,
\end{equation}
where $\mathrm{d}W_\tau$ is the Wiener process (also known as Brownian motion).
\section{Theoretical Analysis}\label{appendix-sec:theoryAnalysis}
\subsection{Implementation Difficulty of Velocity Field}\label{appendix-subsec:impleDiff}
The difficulty of implementing velocity can be given from two perspectives, namely ODE-based implementation and SDE-based implementation, to the best of our knowledge. In this subsection, we want to discuss these two implementation approaches in detail.

\textbf{ODE-based Implementation:}
\begin{enumerate}[leftmargin=*]
  \item{\textbf{WGF framework:} 
According to the continuity equation, we can obtain the following velocity field:
\begin{equation}
  \frac{\mathrm{d}\boldsymbol{X}^{\text{(miss)}}}{\mathrm{d}\tau}\overset{\text{(i)}}{=}v_\tau(\boldsymbol{X}^{\text{(miss)}}) \overset{\text{(ii)}}{=} -[\nabla_{\boldsymbol{X}^{\text{(miss)}}}\log{ \hat{p}( \boldsymbol{X}^{(\text{miss})}\vert \boldsymbol{X}^{(\text{obs})}) } +\lambda\nabla_{\boldsymbol{X}^{\text{(miss)}}}\log{r(\boldsymbol{X}^{\text{(miss)}})} ],
\end{equation}
where (i) is based on~\Cref{subsec:GradinWasserstein}, and (ii) is based on~\Cref{eq:velocityNEG}. The expression of velocity field involves the computation of density term $r(\boldsymbol{X}^{\text{(miss)}})$~\cite{li2018gradient,chen2018unified}, which is intractable during practice. Based on this, we conclude that implementing this velocity field within the WGF framework is difficult.
  }
\item{\textbf{Probability flow ODE:} According to reference~\cite{song2020score}, if we directly plug~\Cref{eq:velocityNEG} into the FPK equation, we can get the following PDE: 
\begin{equation}
\begin{aligned}
  &\frac{\partial r(\boldsymbol{X}^{\text{(miss)}})}{\partial \tau}  \\
  =& -(\boldsymbol{X}^{\text{(miss)}} r(\boldsymbol{X}^{\text{(miss)}})) \\
 = & -\left[\nabla_{\boldsymbol{X}^{\text{(miss)}}}\log{ \hat{p}( \boldsymbol{X}^{(\text{miss})}\vert \boldsymbol{X}^{(\text{obs})}) } r(\boldsymbol{X}^{\text{(miss)}})\right] -\lambda\nabla\cdot\nabla r(\boldsymbol{X}^{\text{(miss)}}) \\
 = & 
 \begin{aligned}
 -[\nabla_{\boldsymbol{X}^{\text{(miss)}}}\log{ \hat{p}( \boldsymbol{X}^{(\text{miss})}\vert \boldsymbol{X}^{(\text{obs})}) } & r(\boldsymbol{X}^{\text{(miss)}})] -\lambda\nabla\cdot\nabla r(\boldsymbol{X}^{\text{(miss)}}) \\
 & +\frac{1}{2}\sigma_\tau^2 \nabla\cdot\nabla r(\boldsymbol{X}^{\text{(miss)}}) -\frac{1}{2}\sigma_\tau^2 \nabla\cdot\nabla r(\boldsymbol{X}^{\text{(miss)}})  
 \end{aligned}
 \\
 = & 
 \begin{aligned}
 -&\left\{\left[\nabla_{\boldsymbol{X}^{\text{(miss)}}}\log{ \hat{p}( \boldsymbol{X}^{(\text{miss})}\vert \boldsymbol{X}^{(\text{obs})}) } + (\lambda+\frac{1}{2}\sigma_\tau^2) \nabla\log{ r(\boldsymbol{X}^{\text{(miss)}}) }\right] r(\boldsymbol{X}^{\text{(miss)}})\right\} \\
  &+\frac{1}{2}\sigma_\tau^2 \nabla\cdot\nabla r(\boldsymbol{X}^{\text{(miss)}}) .
\end{aligned}
\end{aligned}
\end{equation}
When we set $\sigma_\tau$ as $0$, we can find that the corresponding ODE is~\Cref{eq:velocityNEG}, where we are obliged to compute the intractable density $r(\boldsymbol{X}^{\text{(miss)}})$.
}
\end{enumerate}
\textbf{SDE-based Implementation:}\\
If we plug~\Cref{eq:velocityNEG} into the FPK equation, the corresponding PDE can be given as follows:
\begin{equation}
  \begin{aligned}
    &\frac{\partial r(\boldsymbol{X}^{\text{(miss)}})}{\partial \tau}  \\
    =& -(v_{\tau}(\boldsymbol{X}^{\text{(miss)}}) r(\boldsymbol{X}^{\text{(miss)}})) \\
   = & -\left[\nabla_{\boldsymbol{X}^{\text{(miss)}}}\log{ \hat{p}( \boldsymbol{X}^{(\text{miss})}\vert \boldsymbol{X}^{(\text{obs})}) } r(\boldsymbol{X}^{\text{(miss)}})\right] -\lambda\nabla\cdot\nabla r(\boldsymbol{X}^{\text{(miss)}}), 
  \end{aligned}
\end{equation}
where the coefficient before Laplacian operator $\nabla\cdot\nabla$ is $-1$. To the best of our knowledge, this structure makes deriving a corresponding SDE impossible by current approaches.


\subsection{Proof \& Discussions of Propositions \& Corollaries}\label{sec:PropProof}

\begin{proposition*}[\ref{prop:inEffectiveSampling}]
Within WGF framework, DM-based MDI approaches can be viewed as finding the imputed values $\boldsymbol{X}^{(\text{imp})}$ that maximize the following objective:
\begin{equation}\label{appendix-eq:samplingFunctional}
  \mathop{\arg\max}_{\boldsymbol{X}^{(\text{miss})}} \quad  \mathbb{E}_{ r(\boldsymbol{X}^{(\text{miss})})}[{\log{ \hat{p}( \boldsymbol{X}^{(\text{miss})}\vert \boldsymbol{X}^{(\text{obs})}) }}] +\psi(\boldsymbol{X}^{(\text{miss})}) + \text{const},
\end{equation}
where const is the abbreviation of constant, and $\psi(\boldsymbol{X}^{(\text{miss})})$ is a scalar function determined by the type of SDE underlying the DMs.
\begin{itemize}[leftmargin=*]
  \item{\textbf{VP-SDE:} $\psi(\boldsymbol{X}^{(\text{miss})})=\mathbb{E}_{ r(\boldsymbol{X}^{(\text{miss})})}\{\frac{1}{4}[\boldsymbol{X}^{\text{(miss)}}]^\top[\boldsymbol{X}^{\text{(miss)}}]  -\frac{1}{2}\log{r(\boldsymbol{X}^{\text{(miss)}})} \}$}
  \item{\textbf{VE-SDE:} $\psi(\boldsymbol{X}^{(\text{miss})})=\mathbb{E}_{r(\boldsymbol{X}^{(\text{miss})})}\{ -\frac{1}{2}\log{r(\boldsymbol{X}^{\text{(miss)}})} \}$}
  \item{\textbf{sub-VP-SDE:} $\psi(\boldsymbol{X}^{(\text{miss})})=\mathbb{E}_{ r(\boldsymbol{X}^{(\text{miss})})}\{  \frac{1}{4\gamma_\tau}[\boldsymbol{X}^{\text{(miss)}}]^\top[\boldsymbol{X}^{\text{(miss)}}] -\frac{1}{2}\log{r(\boldsymbol{X}^{\text{(miss)}})}\}$, 
  where $\gamma_\tau$ is determined by noise scale $\beta_\tau$: $\gamma_\tau\coloneqq(1-\exp(-2\int_{0}^{\tau}{\beta_s\mathrm{d}s}))>0,0<\beta_1<\beta_2<\dots<\beta_{\mathrm{T}}<1$.
  }
\end{itemize}
It is important to note that in DMs, the condition $\psi(\boldsymbol{X}^{(\text{miss})}) \ge 0$ consistently holds. This assertion is supported by the fact that the inner product $[\boldsymbol{X}^{\text{(miss)}}]^\top[\boldsymbol{X}^{\text{(miss)}}] \ge 0$, and the entropy function defined as $\mathbb{H}[r(\boldsymbol{X}^{\text{(miss)}})] \coloneqq -\int{r(\boldsymbol{X}^{\text{(miss)}}) \log{r(\boldsymbol{X}^{\text{(miss)}})} \mathrm{d}\boldsymbol{X}^{\text{(miss)}}}$ is also non-negative.
\end{proposition*}
\begin{proof}
  Since there are various approaches for reversing the sampling procedure of DMs, for simplicity, we mainly consider the VP-SDE, VE-SDE, and sub-VP-SDE as analysis objects. 
  \begin{itemize}[leftmargin=*]
    \item{\textbf{VP-SDE:}
  According to reference~\cite{song2020score}, the density evolution of the generative process for VP-SDE can be delineated by the following PDE:
\begin{equation}
\begin{aligned}
  \frac{\partial r(\boldsymbol{X}^{\text{(miss)}})}{\partial \tau}  =& -\nabla_{\boldsymbol{X}^{\text{(miss)}}} \cdot \left\{ 
    r(\boldsymbol{X}^{\text{(miss)}}) \left[\beta_\tau\right]\left[\frac{1}{2}\boldsymbol{X}^{\text{(miss)}} + \nabla_{\boldsymbol{X}^{\text{(miss)}}}{\log{\hat{p}(\boldsymbol{X}^{\text{(miss)}}\vert \boldsymbol{X}^{\text{(obs)}})}}\right]
   \right\}\\
     +&\frac{\beta_\tau}{2}\nabla_{\boldsymbol{X}^{\text{(miss)}}} \cdot\nabla_{\boldsymbol{X}^{\text{(miss)}}} r(\boldsymbol{X}^{\text{(miss)}}) 
\end{aligned}
\end{equation}
where $\beta_\tau\in(0, 1)$ is the noise scale.
On this basis, by chaning the variable as $\mathrm{d}\tau\coloneqq\frac{\beta_\tau}{2}\mathrm{d}\tau$~\cite{khrulkov2022understanding}, we can get the following equation:
\begin{equation}\label{eq:VPFPK}
  \frac{\partial r(\boldsymbol{X}^{\text{(miss)}})}{\partial \tau} = -\nabla_{\boldsymbol{X}^{\text{(miss)}}} \cdot \left\{
      \begin{aligned}  
        r(\boldsymbol{X}^{\text{(miss)}})  [&\frac{1}{2}\boldsymbol{X}^{\text{(miss)}} + \nabla_{\boldsymbol{X}^{\text{(miss)}}}{\log{\hat{p}(\boldsymbol{X}^{\text{(miss)}}\vert \boldsymbol{X}^{\text{(obs)}})}}\\
    & -\frac{1}{2}\nabla_{\boldsymbol{X}^{\text{(miss)}}}\log{r(\boldsymbol{X}^{\text{(miss)}})} ]
      \end{aligned}
   \right\}.
\end{equation}
Comparing~\Cref{eq:VPFPK} with~\Cref{eq:gradFlowEquation,eq:continuityEquation}, the cost functional to be minimized of this simulation procedure can be given as follows:
\begin{equation}
\begin{aligned}
  \mathcal{F}_{\text{VP-SDE}}& =-\int{r(\boldsymbol{X}^{\text{(miss)}})\left\{
  \begin{aligned}  
  \frac{1}{4}&[\boldsymbol{X}^{\text{(miss)}}]^\top[\boldsymbol{X}^{\text{(miss)}}] +\log{\hat{p}(\boldsymbol{X}^{\text{(miss)}}\vert \boldsymbol{X}^{\text{(obs)}})}\\
  & -\frac{1}{2}\log{r(\boldsymbol{X}^{\text{(miss)}})} + \text{const}
  \end{aligned}  
  \right\}\mathrm{d}\boldsymbol{X}^{\text{(miss)}}}\\
  & = -\mathbb{E}_{ r(\boldsymbol{X}^{\text{(miss)}})}\left\{
  \begin{aligned}  
  \frac{1}{4}&[\boldsymbol{X}^{\text{(miss)}}]^\top[\boldsymbol{X}^{\text{(miss)}}] +\log{\hat{p}(\boldsymbol{X}^{\text{(miss)}}\vert \boldsymbol{X}^{\text{(obs)}})}\\
  & -\frac{1}{2}\log{r(\boldsymbol{X}^{\text{(miss)}})} + \text{const}
  \end{aligned} 
  \right\}.
\end{aligned}
\end{equation}
Note that $\frac{1}{4}[\boldsymbol{X}^{\text{(miss)}}]^\top[\boldsymbol{X}^{\text{(miss)}}] \ge 0$ and $-\frac{1}{2}\int{r(\boldsymbol{X}^{\text{(miss)}})\log{r(\boldsymbol{X}^{\text{(miss)}})}\mathrm{d}\boldsymbol{X}^{\text{(miss)}}} \ge 0$ hold, and thus the proposition for VP-SDE is proved by taking the negative of the abovementioned equation.
    }
\item{\textbf{VE-SDE:} Similarly, based on reference~\cite{song2020score}, the following PDE can be given to delineate the density evolution of the generative process for VE-SDE:
\begin{equation}\label{eq:smldReverse}
\begin{aligned}
  \frac{\partial r(\boldsymbol{X}^{\text{(miss)}})}{\partial \tau} &= -\nabla_{\boldsymbol{X}^{\text{(miss)}}} \cdot \left\{ 
    r(\boldsymbol{X}^{\text{(miss)}})\left[-\frac{\mathrm{d}\sigma^2_\tau}{\mathrm{d}\tau}\right]\nabla_{\boldsymbol{X}^{\text{(miss)}}}{\log{\hat{p}(\boldsymbol{X}^{\text{(miss)}}\vert \boldsymbol{X}^{\text{(obs)}})}}  
   \right\}\\
   & + \frac{1}{2}\frac{\mathrm{d}\sigma^2_\tau}{\mathrm{d}\tau}\nabla_{\boldsymbol{X}^{\text{(miss)}}} \cdot\nabla_{\boldsymbol{X}^{\text{(miss)}}}r(\boldsymbol{X}^{\text{(miss)}}),
\end{aligned}
\end{equation} 
where $\sigma^2_\tau$ is a time varying noise scale.

As such, by chaning the variable as $\mathrm{d}\tau\coloneqq\left[\frac{\mathrm{d}\sigma^2_\tau}{\mathrm{d}\tau}\right]\mathrm{d}\tau$~\cite{khrulkov2022understanding},~\Cref{eq:smldReverse} can be reformulated as follows:
\begin{equation}\label{eq:VEFPK}
  \frac{\partial r(\boldsymbol{X}^{\text{(miss)}})}{\partial \tau} = -\nabla_{\boldsymbol{X}^{\text{(miss)}}} \cdot \left\{ 
    r(\boldsymbol{X}^{\text{(miss)}})
    \left[
      \nabla_{\boldsymbol{X}^{\text{(miss)}}}{\log{\hat{p}(\boldsymbol{X}^{\text{(miss)}}\vert \boldsymbol{X}^{\text{(obs)}})}}-\frac{1}{2}\nabla_{\boldsymbol{X}^{\text{(miss)}}}\log{r(\boldsymbol{X}^{\text{(miss)}})}
    \right] 
   \right\}.
\end{equation}
Comparing~\Cref{eq:VEFPK} with~\Cref{eq:gradFlowEquation,eq:continuityEquation}, the cost functional to be minimized of this simulation procedure can be given as follows:
\begin{equation}
\begin{aligned}
 \mathcal{F}_{\text{VE-SDE}}& =\int{r(\boldsymbol{X}^{\text{(miss)}})\left\{\frac{1}{2}\log{r(\boldsymbol{X}^{\text{(miss)}})} -\log{\hat{p}(\boldsymbol{X}^{\text{(miss)}}\vert \boldsymbol{X}^{\text{(obs)}})}+\text{const}\right\}\mathrm{d}\boldsymbol{X}^{\text{(miss)}}}\\
  & =- \mathbb{E}_{ r(\boldsymbol{X}^{\text{(miss)}})}\left\{-\frac{1}{2}\log{r(\boldsymbol{X}^{\text{(miss)}})} +\log{\hat{p}(\boldsymbol{X}^{\text{(miss)}}\vert \boldsymbol{X}^{\text{(obs)}})}+\text{const}\right\}.
\end{aligned}
\end{equation}
Note that the entropy function $-\frac{1}{2}\int{r(\boldsymbol{X}^{\text{(miss)}})\log{r(\boldsymbol{X}^{\text{(miss)}})}\mathrm{d}\boldsymbol{X}^{\text{(miss)}}} \ge 0$ holds, and thus the proposition for VE-SDE is proved by taking the negative of the abovementioned equation.

}
\item{\textbf{sub-VP-SDE:} 
Based on reference~\cite{song2020score}, the following PDE can be given to delineate the density evolution of the generative process for sub-VP-SDE:
\begin{equation}
\begin{aligned}
  \frac{\partial r(\boldsymbol{X}^{\text{(miss)}})}{\partial \tau}  =& -\nabla_{\boldsymbol{X}^{\text{(miss)}}} \cdot \left\{ 
    r(\boldsymbol{X}^{\text{(miss)}}) \left[\beta_\tau\right]\left[\frac{1}{2}\boldsymbol{X}^{\text{(miss)}} +  \gamma_\tau\nabla_{\boldsymbol{X}^{\text{(miss)}}}{\log{\hat{p}(\boldsymbol{X}^{\text{(miss)}}\vert \boldsymbol{X}^{\text{(obs)}})}}\right]
   \right\}\\
     +&\frac{\beta_\tau}{2}\gamma_\tau\nabla_{\boldsymbol{X}^{\text{(miss)}}} \cdot\nabla_{\boldsymbol{X}^{\text{(miss)}}} r(\boldsymbol{X}^{\text{(miss)}}) ,
\end{aligned}
\end{equation}
where $\gamma_\tau\coloneqq(1-\exp(-2\int_{0}^{\tau}{\beta_s\mathrm{d}s}))>0$. On this basis, by chaning the variable as $\mathrm{d}\tau\coloneqq\frac{\beta_\tau}{2}\mathrm{d}\tau$, we can get the following equation:
\begin{equation}\label{eq:subVPFPK}
  \frac{\partial r(\boldsymbol{X}^{\text{(miss)}})}{\partial \tau} = -\nabla_{\boldsymbol{X}^{\text{(miss)}}} \cdot \left\{
      \begin{aligned}  
        r(\boldsymbol{X}^{\text{(miss)}})  [&\frac{1}{2}\boldsymbol{X}^{\text{(miss)}} +\gamma_\tau \nabla_{\boldsymbol{X}^{\text{(miss)}}}{\log{\hat{p}(\boldsymbol{X}^{\text{(miss)}}\vert \boldsymbol{X}^{\text{(obs)}})}}\\
    & -\frac{\gamma_\tau}{2}\nabla_{\boldsymbol{X}^{\text{(miss)}}}\log{r(\boldsymbol{X}^{\text{(miss)}})} ]
      \end{aligned}
   \right\}.
\end{equation}
}
Comparing~\Cref{eq:subVPFPK} with~\Cref{eq:gradFlowEquation,eq:continuityEquation}, the cost functional to be minimized of this simulation procedure can be given as follows:
\begin{equation}
\begin{aligned}
  \mathcal{F}_{\text{sub-VP-SDE}}& =-\int{r(\boldsymbol{X}^{\text{(miss)}})\left\{
  \begin{aligned}  
  \frac{1}{4}&[\boldsymbol{X}^{\text{(miss)}}]^\top[\boldsymbol{X}^{\text{(miss)}}] +\gamma_\tau\log{\hat{p}(\boldsymbol{X}^{\text{(miss)}}\vert \boldsymbol{X}^{\text{(obs)}})}\\
  & -\frac{\gamma_\tau}{2}\log{r(\boldsymbol{X}^{\text{(miss)}})} +\text{const}
  \end{aligned}  
  \right\}\mathrm{d}\boldsymbol{X}^{\text{(miss)}}}\\
  & = -\mathbb{E}_{ r(\boldsymbol{X}^{\text{(miss)}})}\left\{
  \begin{aligned}  
  \frac{1}{4}&[\boldsymbol{X}^{\text{(miss)}}]^\top[\boldsymbol{X}^{\text{(miss)}}] +\gamma_\tau\log{\hat{p}(\boldsymbol{X}^{\text{(miss)}}\vert \boldsymbol{X}^{\text{(obs)}})}\\
  & -\frac{\gamma_\tau}{2}\log{r(\boldsymbol{X}^{\text{(miss)}})} +\text{const}
  \end{aligned} 
  \right\}\\
  & = -\mathbb{E}_{ r(\boldsymbol{X}^{\text{(miss)}})}\left\{
    \begin{aligned}  
    \frac{1}{4\gamma_\tau}&[\boldsymbol{X}^{\text{(miss)}}]^\top[\boldsymbol{X}^{\text{(miss)}}] +\log{\hat{p}(\boldsymbol{X}^{\text{(miss)}}\vert \boldsymbol{X}^{\text{(obs)}})}\\
    & -\frac{1}{2}\log{r(\boldsymbol{X}^{\text{(miss)}})} + \text{const}
    \end{aligned} 
    \right\}.
\end{aligned}
\end{equation}
Note that $\frac{1}{4\gamma_\tau}[\boldsymbol{X}^{\text{(miss)}}]^\top[\boldsymbol{X}^{\text{(miss)}}] \ge 0$ and $-\frac{1}{2}\int{r(\boldsymbol{X}^{\text{(miss)}})\log{r(\boldsymbol{X}^{\text{(miss)}})}\mathrm{d}\boldsymbol{X}^{\text{(miss)}}} \ge 0$ hold, and thus the proposition for sub-VP-SDE is proved by taking the negative of the abovementioned equation.
\end{itemize}
In summary, the regularization term $\psi(\boldsymbol{X}^{(\text{miss})})$ for VP-SDE is $\mathbb{E}_{\boldsymbol{X}^{(\text{miss})}\sim r(\boldsymbol{X}^{\text{(miss)}})}\{\frac{1}{4}[\boldsymbol{X}^{\text{(miss)}}]^\top[\boldsymbol{X}^{\text{(miss)}}]-\frac{1}{2}\log{r(\boldsymbol{X}^{\text{(miss)}})} \}$, for VE-SDE is $\frac{1}{2}\mathbb{H}(r(\boldsymbol{X}^{\text{(miss)}}))$, and for sub-VP-SDE is $\mathbb{E}_{\boldsymbol{X}^{(\text{miss})}\sim r(\boldsymbol{X}^{\text{(miss)}})}\{\frac{1}{4\gamma_\tau}[\boldsymbol{X}^{\text{(miss)}}]^\top[\boldsymbol{X}^{\text{(miss)}}]-\frac{1}{2}\log{r(\boldsymbol{X}^{\text{(miss)}})} \}$.
\end{proof}

\begin{proposition*}[\ref{prop:simulationLoss}]
  The evolution of $\mathcal{F}_{\text{NER}}$ along $\tau$ can be characterized by the following ODE, assuming that the boundary condition $\mathbb{E}_{ r(\boldsymbol{X}^{(\text{miss})},\tau)}\{\nabla_{\boldsymbol{X}^{(\text{miss})}}\cdot [u(\boldsymbol{X}^{\text{(miss)}}, \tau) \log{\hat{p}( \boldsymbol{X}^{(\text{miss})} \vert \boldsymbol{X}^{(\text{obs})})} ]\}=0$ is satisfied for the velocity field $u(\boldsymbol{X}^{\text{(miss)}}, \tau)$:
  \begin{equation}\label{appendix-eq:iterativeLossImputation}
    \frac{\mathrm{d} \mathcal{F}_{\text{NER}}}{\mathrm{d} \tau} = \mathbb{E}_{ r(\boldsymbol{X}^{(\text{miss})},\tau)}[ 
    u^{\top}(\boldsymbol{X}^{\text{(miss)}}, \tau)\nabla_{\boldsymbol{X}^{(\text{miss})}} \log{\hat{p}( \boldsymbol{X}^{(\text{miss})} \vert \boldsymbol{X}^{(\text{obs})})}  -\lambda\nabla_{\boldsymbol{X}^{(\text{miss})}} \cdot u(\boldsymbol{X}^{\text{(miss)}}, \tau)
    ].
  \end{equation}
This boundary condition is achievable, for instance, when $\hat{p}( \boldsymbol{X}^{(\text{miss})} \vert \boldsymbol{X}^{(\text{obs})})$ is bounded, and the limit of the velocity field as the norm of $\boldsymbol{X}^{(\text{miss})}$ approaches zero is zero ($\lim_{\Vert \boldsymbol{X}^{\text{(miss)}}\Vert\rightarrow 0}{u(\boldsymbol{X}^{\text{(miss)}}, \tau)}=0$). 
\end{proposition*}
\begin{proof}
Before proving this proposition, we should recognize that the evolution of $\boldsymbol{X}^{(\text{miss})}$ should promise the probability density function $r(\boldsymbol{X}^{(\text{miss})},\tau)$ unchanged. In other words, the following continuity equation should be satisfied during the optimization of $r(\boldsymbol{X}^{(\text{miss})},\tau)$:
\begin{equation}
  \frac{\partial r(\boldsymbol{X}^{(\text{miss})},\tau)}{\partial \tau} = -\nabla_{\boldsymbol{X}^{(\text{miss})}}\cdot[r(\boldsymbol{X}^{(\text{miss})},\tau)u(\boldsymbol{X}^{\text{(miss)}}, \tau)].
\end{equation}
On this basis, the evolution of $\mathcal{F}_{\text{NER}}$ along time $\tau$, $ \frac{\mathrm{d}\mathcal{F}_{\text{NER}}}{\mathrm{d} \tau}$, can be given as follows based on the chain rule:
\begin{equation}\label{eq:evoluationLossfunc}
\begin{aligned}
   & \frac{\mathrm{d}\mathcal{F}_{\text{NER}}}{\mathrm{d} \tau} 
   \\
 = & \int{ \frac{\partial r(\boldsymbol{X}^{(\text{miss})},\tau)}{\partial \tau} \left[
  \log{ \hat{p}( \boldsymbol{X}^{(\text{miss})}\vert \boldsymbol{X}^{(\text{obs})}) } + \lambda \log{r(\boldsymbol{X}^{(\text{miss})},\tau)} + \lambda
 \right] \mathrm{d}\boldsymbol{X}^{(\text{miss})}}  
 \\
 = &  \int{ -\{\nabla_{\boldsymbol{X}^{(\text{miss})}}\cdot[r(\boldsymbol{X}^{(\text{miss})},\tau)u(\boldsymbol{X}^{\text{(miss)}}, \tau)]\}[
  \log{ \hat{p}( \boldsymbol{X}^{(\text{miss})}\vert \boldsymbol{X}^{(\text{obs})}) } +  \lambda\log{r(\boldsymbol{X}^{(\text{miss})},\tau)} +  \lambda] \mathrm{d}\boldsymbol{X}^{(\text{miss})}}   
 \\
 \overset{\text{(i)}}{=} &  \int{ [r(\boldsymbol{X}^{(\text{miss})},\tau)u(\boldsymbol{X}^{\text{(miss)}}, \tau)]^{\top} \nabla_{\boldsymbol{X}^{(\text{miss})}}[
  \log{ \hat{p}( \boldsymbol{X}^{(\text{miss})}\vert \boldsymbol{X}^{(\text{obs})}) } +  \lambda\log{r(\boldsymbol{X}^{(\text{miss})},\tau)} +  \lambda] \mathrm{d}\boldsymbol{X}^{(\text{miss})}}  
 \\
 = & \int{ [r(\boldsymbol{X}^{(\text{miss})},\tau)u(\boldsymbol{X}^{\text{(miss)}}, \tau)]^{\top} \{\nabla_{\boldsymbol{X}^{(\text{miss})}}[
  \log{ \hat{p}( \boldsymbol{X}^{(\text{miss})}\vert \boldsymbol{X}^{(\text{obs})}) } +  \lambda\log{r(\boldsymbol{X}^{(\text{miss})},\tau)}] \} \mathrm{d}\boldsymbol{X}^{(\text{miss})}} 
 \\
 = & \int{ [u(\boldsymbol{X}^{\text{(miss)}}, \tau)]^{\top} [
 r(\boldsymbol{X}^{(\text{miss})},\tau)\nabla_{\boldsymbol{X}^{(\text{miss})}} 
  \log{ \hat{p}( \boldsymbol{X}^{(\text{miss})}\vert \boldsymbol{X}^{(\text{obs})}) } 
  + \lambda  r(\boldsymbol{X}^{(\text{miss})},\tau)\nabla_{\boldsymbol{X}^{(\text{miss})}}\log{r(\boldsymbol{X}^{(\text{miss})},\tau)}  
  ] \mathrm{d}\boldsymbol{X}^{(\text{miss})}}
 \\
 = &
 \int{ [u(\boldsymbol{X}^{\text{(miss)}}, \tau)]^{\top} [r(\boldsymbol{X}^{(\text{miss})},\tau)\nabla_{\boldsymbol{X}^{(\text{miss})}}
  \log{ \hat{p}( \boldsymbol{X}^{(\text{miss})}\vert \boldsymbol{X}^{(\text{obs})}) }  + \lambda \nabla_{\boldsymbol{X}^{(\text{miss})}}r(\boldsymbol{X}^{(\text{miss})},\tau) ] \mathrm{d}\boldsymbol{X}^{(\text{miss})}}
 \\
 \overset{\text{(ii)}}{=} & 
 \int{ r(\boldsymbol{X}^{(\text{miss})},\tau)[u^{\top}(\boldsymbol{X}^{\text{(miss)}}, \tau)\nabla_{\boldsymbol{X}^{(\text{miss})}}
  \log{ \hat{p}( \boldsymbol{X}^{(\text{miss})}\vert \boldsymbol{X}^{(\text{obs})}) }  -  \lambda
  \nabla_{\boldsymbol{X}^{(\text{miss})}}\cdot u(\boldsymbol{X}^{\text{(miss)}}, \tau)
  ] 
  \mathrm{d}\boldsymbol{X}^{(\text{miss})}}
  \\
  = &  \mathbb{E}_{ r(\boldsymbol{X}^{(\text{miss})},\tau)}[ u^{\top}(\boldsymbol{X}^{\text{(miss)}}, \tau)\nabla_{\boldsymbol{X}^{(\text{miss})}} \log{\hat{p}( \boldsymbol{X}^{(\text{miss})} \vert \boldsymbol{X}^{(\text{obs})})}  - \lambda\nabla_{\boldsymbol{X}^{(\text{miss})}} \cdot u(\boldsymbol{X}^{\text{(miss)}}, \tau)],
\end{aligned}
\end{equation}
%
where (i) and (ii) are based on integration by parts.
\end{proof}
\begin{proposition*}[\ref{prop:steinMap}]
When the velocity field $u(\boldsymbol{X}^{\text{(miss)}}, \tau)$ is constrained by the norm of RKHS, the problem of finding the steepest gradient ascent direction can be formulated as follows:
\begin{equation}\label{appendix-eq:objectiveRKHSReg}
  u(\boldsymbol{X}^{\text{(miss)}}, \tau) = \mathop{\arg\max}_{v(\boldsymbol{X}^{\text{(miss)}}, \tau)\in\mathcal{H}^d}
  \begin{aligned}
  & \{\mathbb{E}_{ r(\boldsymbol{X}^{(\text{miss})},\tau)}[ 
    v^{\top}(\boldsymbol{X}^{\text{(miss)}}, \tau)\nabla_{\boldsymbol{X}^{(\text{miss})}} \log{\hat{p}( \boldsymbol{X}^{(\text{miss})} \vert \boldsymbol{X}^{(\text{obs})})}  \\
  & \quad \quad- \lambda\nabla_{\boldsymbol{X}^{(\text{miss})}} \cdot v(\boldsymbol{X}^{\text{(miss)}}, \tau)
      ]\}- \frac{1}{2}\Vert v(\boldsymbol{X}^{\text{(miss)}}, \tau)\Vert_{\mathcal{H}}^2.
  \end{aligned}
  \end{equation}
  The corresponding optimal solution is given by:
\begin{equation}\label{appendix-eq:RKHSVelocityField}
u(\boldsymbol{X}^{\text{(miss)}}, \tau) = \mathbb{E}_{r(\tilde{\boldsymbol{{X}}}^{(\text{miss})},\tau)}\left\{ 
      \begin{aligned}  
      & { - \lambda \nabla_{{\tilde{\boldsymbol{X}}}^{(\text{miss})}} \mathcal{K}(\boldsymbol{X}^{(\text{miss})}  , \tilde{\boldsymbol{X}}^{(\text{miss})})  } \\
      &\quad \quad+ [\nabla_{{\tilde{\boldsymbol{X}}}^{(\text{miss})}}\log{\hat{p}( \tilde{\boldsymbol{X}}^{(\text{miss})} \vert \boldsymbol{X}^{(\text{obs})})}]^\top\mathcal{K}(\boldsymbol{X}^{(\text{miss})}  , \tilde{\boldsymbol{X}}^{(\text{miss})}) 
      \end{aligned} 
      \right\},
\end{equation}
where $\mathcal{K}(\boldsymbol{X}^{(\text{miss})}  , \tilde{\boldsymbol{X}}^{(\text{miss})})$ is kernel function.
\end{proposition*}
\begin{proof}
Assume we have a map function $\phi(x)$, the kernel function can be given as follows:
\begin{equation}\label{eq:kernelMap}
  \mathcal{K}(x,y) = \left<\phi(x),\phi(y)\right>_{\mathcal{H}}.
\end{equation}
Based on this, the regularization term that control the magnitude of $v(\boldsymbol{X}^{\text{(miss)}}, \tau)$ can be given by $\frac{1}{2}\Vert v(\boldsymbol{X}^{\text{(miss)}}, \tau) \Vert_{\mathcal{H}}$, and the spectral decomposition of kernel function can be given as follows:
\begin{equation}\label{eq:kernelExpress}
  \mathcal{K}(x, y) = \sum_{i=1}^{\infty}{\xi_i \phi_{i}(x)\phi_i(y)},
\end{equation}
where $\phi_i(\cdot)$ indicates the orthonormal basis and $\xi_i$ is the corresponding eigen-value. For any function $v(\boldsymbol{X}^{\text{(miss)}}, \tau)\in\mathcal{H}$, the following decomposition is given:
\begin{equation}\label{eq:velocityDecomp}
  v(\boldsymbol{X}^{\text{(miss)}}, \tau)=\sum_{i=1}^{\infty}{v_i\sqrt{\xi_i}\phi_i(\boldsymbol{X}^{\text{(miss)}}, \tau)},
\end{equation}
where $v_i$ and $\sum_{i=1}^{\infty}{\Vert v_i \Vert^2}<\infty$.

The learning objective defined in~\Cref{appendix-eq:objectiveRKHSReg} can be reformulated as follows:
\begin{equation}
\begin{aligned}
  &v^*(\boldsymbol{X}^{\text{(miss)}}, \tau) \\
   = &\mathop{\arg\max}_{v(\boldsymbol{X}^{\text{(miss)}}, \tau)\in\mathcal{H}^d}
    \begin{aligned}
    & \{\mathbb{E}_{ r(\boldsymbol{X}^{(\text{miss})},\tau)}[ 
      v^{\top}(\boldsymbol{X}^{\text{(miss)}}, \tau)\nabla_{\boldsymbol{X}^{(\text{miss})}} \log{\hat{p}( \boldsymbol{X}^{(\text{miss})} \vert \boldsymbol{X}^{(\text{obs})})}  \\
    & \quad \quad- \lambda\nabla_{\boldsymbol{X}^{(\text{miss})}} \cdot v(\boldsymbol{X}^{\text{(miss)}}, \tau)
        ]\}- \frac{1}{2}\Vert v(\boldsymbol{X}^{\text{(miss)}}, \tau)\Vert_{\mathcal{H}^d},
    \end{aligned}\\
     \overset{\text{(i)}}{=} &\mathop{\arg\max}_{v(\boldsymbol{X}^{\text{(miss)}}, \tau)\in\mathcal{H}^d} 
    \begin{aligned}
    & \{\mathbb{E}_{ r(\tilde{\boldsymbol{X}}^{(\text{miss})},\tau)}[ 
      \sum_{i=1}^{\infty}{\sqrt{\xi_i}\nabla_{\tilde{\boldsymbol{X}}^{(\text{miss})}} \log{\hat{p}( \tilde{\boldsymbol{X}}^{(\text{miss})} \vert \boldsymbol{X}^{(\text{obs})})^{\top}}v_i\phi_i(\tilde{\boldsymbol{X}}^{\text{(miss)}}, \tau)}\\
    & \quad \quad- \lambda\nabla_{\tilde{\boldsymbol{X}}^{(\text{miss})}} \cdot \sum_{i=1}^{\infty}{v_i\sqrt{\xi_i}\phi_i(\tilde{\boldsymbol{X}}^{\text{(miss)}}, \tau)}
          ]\}- \frac{1}{2}\sum_{i=1}^{\infty}{\Vert v_i\Vert^2},
    \end{aligned}\\
  \end{aligned}
\end{equation}
Take the right-hand-side of (i) with-respect-to $v_i$, and set it to $0$, we can get:
\begin{equation}
   \sqrt{\xi_i}\{\mathbb{E}_{ r(\tilde{\boldsymbol{X}}^{(\text{miss})},\tau)}[ 
      [{\nabla_{\tilde{\boldsymbol{X}}^{(\text{miss})}} \log{\hat{p}( \tilde{\boldsymbol{X}}^{(\text{miss})} \vert \boldsymbol{X}^{(\text{obs})})}}]^\top\phi_i(\tilde{\boldsymbol{X}}^{\text{(miss)}}, \tau)
   - \lambda\nabla_{\tilde{\boldsymbol{X}}^{(\text{miss})}}{\phi_i(\tilde{\boldsymbol{X}}^{\text{(miss)}}, \tau)}
          ]\}- v_i =0.
\end{equation}
On this basis, $v_i^*$ can be given as follows:
\begin{equation}
  v_i^* =  \sqrt{\xi_i}\{\mathbb{E}_{ r(\tilde{\boldsymbol{X}}^{(\text{miss})},\tau)}[ 
    [\nabla_{\tilde{\boldsymbol{X}}^{(\text{miss})}} \log{\hat{p}( \tilde{\boldsymbol{X}}^{(\text{miss})} \vert \boldsymbol{X}^{(\text{obs})})} ]^\top \phi_i(\boldsymbol{X}^{\text{(miss)}}, \tau)
 - \lambda\nabla_{\tilde{\boldsymbol{X}}^{(\text{miss})}}{\phi_i(\tilde{\boldsymbol{X}}^{\text{(miss)}}, \tau)}]\},
\end{equation}
and hence, $u(\boldsymbol{X}^{\text{(miss)}}, \tau)$ can be given as follows:
\begin{equation}
  \begin{aligned}
  & u(\boldsymbol{X}^{\text{(miss)}}, \tau) \\
=& \sum_{i=1}^{\infty}{\sqrt{\xi_i} v_i^* \phi_i(\boldsymbol{X}^{\text{(miss)}}, \tau)}\\
=&\mathbb{E}_{r(\tilde{\boldsymbol{X}}^{(\text{miss})},\tau)}\left[ 
\begin{aligned}  
& { -  \lambda \nabla_{\tilde{\boldsymbol{X}}^{(\text{miss})}} \mathcal{K}(\boldsymbol{X}^{(\text{miss})}  , \tilde{\boldsymbol{X}}^{(\text{miss})})  } \\
&\quad \quad+ [\nabla_{\tilde{\boldsymbol{X}}^{(\text{miss})}}\log{\hat{p}( \tilde{\boldsymbol{X}}^{(\text{miss})} \vert \boldsymbol{X}^{(\text{obs})})}]^\top\mathcal{K}(\boldsymbol{X}^{(\text{miss})}  , \tilde{\boldsymbol{X}}^{(\text{miss})}) 
\end{aligned} 
\right].
  \end{aligned}
\end{equation}
\end{proof}

\begin{proposition*}[\ref{prop:JointCondDist}]
Suppose the proposal distribution $r({\boldsymbol{X}}^{\text{(joint)}})$ is factorized by $r({\boldsymbol{X}}^{\text{(joint)}})\coloneqq r({\boldsymbol{X}}^{\text{(miss)}}) p(\boldsymbol{X}^{\text{(obs)}})$. The cost functional concerned with joint distribution defined by the following equation:
 \begin{equation}
  \mathcal{F}_{\text{joint-NER}} \coloneqq \mathbb{E}_{ r({\boldsymbol{X}}^{\text{(joint)}})}[{\log{ \hat{p}( \boldsymbol{X}^{(\text{joint})}) }}] - \lambda\mathbb{H}[r({\boldsymbol{X}}^{\text{(joint)}})],
 \end{equation}
 results in the velocity field defined in~\Cref{eq:jointVelocityField}, and is a lower bound of $\mathcal{F}_{\text{NER}}$ where the gap is a constant. (i.e. $\mathcal{F}_{\text{joint-NER}} = \mathcal{F}_{\text{NER}}-\text{const}, \text{const}\ge 0$.)
\end{proposition*}

\begin{proof}
Our proof will be divided into two parts namely `velocity field derivation' and `upper bound acquirement'.

\textbf{Velocity Field Derivation:}

the following continuity equation should be satisfied during the optimization of $r(\boldsymbol{X}^{(\text{miss})},\tau)$:
\begin{equation}
\begin{aligned}
  & \frac{\partial r(\boldsymbol{X}^{(\text{miss})},\tau)}{\partial \tau} = -\nabla_{\boldsymbol{X}^{(\text{miss})}}\cdot[r(\boldsymbol{X}^{(\text{miss})},\tau){u}(\boldsymbol{X}^{\text{(miss)}}, \tau)]\\
  \Rightarrow& \frac{\partial r(\boldsymbol{X}^{(\text{miss})},\tau)}{\partial \tau}\times p(\boldsymbol{X}^{\text{(obs)}}) = -\nabla_{\boldsymbol{X}^{(\text{miss})}}\cdot[r(\boldsymbol{X}^{(\text{miss})},\tau){u}(\boldsymbol{X}^{\text{(miss)}}, \tau)]\times p(\boldsymbol{X}^{\text{(obs)}}) \\
  \overset{\text{(i)}}{\Rightarrow}& \frac{\partial r(\boldsymbol{X}^{(\text{joint})},\tau)}{\partial \tau} = -\nabla_{\boldsymbol{X}^{(\text{miss})}}\cdot[r(\boldsymbol{X}^{(\text{joint})},\tau){u}(\boldsymbol{X}^{\text{(joint)}}, \tau)],
\end{aligned}
\end{equation}
where (i) is based on the fact that $\boldsymbol{X}^{\text{(obs)}}$ remains unchanged during imputation process. And thus, according to~\Cref{eq:evoluationLossfunc}, the evolution of $\mathcal{F}_{\text{joint-NER}}$ along time $\tau$, $ \frac{\mathrm{d}\mathcal{F}_{\text{joint-NER}}}{\mathrm{d} \tau}$, can be given as follows based on the chain rule:
\begin{equation}\label{eq:evoluationJointLossfunc}
\begin{aligned}
   & \frac{\mathrm{d}\mathcal{F}_{\text{joint-NER}}}{\mathrm{d} \tau} 
   \\
 = & \int{ \frac{\partial r(\boldsymbol{X}^{(\text{joint})},\tau)}{\partial \tau} \left[
  \log{ \hat{p}( \boldsymbol{X}^{(\text{joint})}) } +  \lambda\log{r(\boldsymbol{X}^{(\text{joint})},\tau)} +  \lambda
 \right] \mathrm{d}\boldsymbol{X}^{(\text{joint})}}  
 \\
 = &  \int{ -\{\nabla_{\boldsymbol{X}^{(\text{miss})}}\cdot[r(\boldsymbol{X}^{(\text{joint})},\tau)u(\boldsymbol{X}^{\text{(joint)}}, \tau)]\}[
  \log{ \hat{p}( \boldsymbol{X}^{(\text{joint})}) } + \lambda \log{r(\boldsymbol{X}^{(\text{joint})},\tau)} +  \lambda] \mathrm{d}\boldsymbol{X}^{(\text{joint})}}   
 \\
 \overset{\text{(i)}}{=} &  \int{ [r(\boldsymbol{X}^{(\text{joint})},\tau){u}(\boldsymbol{X}^{\text{(joint)}}, \tau)]^{\top} \nabla_{\boldsymbol{X}^{(\text{miss})}}[
  \log{ \hat{p}( \boldsymbol{X}^{(\text{joint})}) } +  \lambda\log{r(\boldsymbol{X}^{(\text{joint})},\tau)} +  \lambda] \mathrm{d}\boldsymbol{X}^{(\text{joint})}}  
 \\
 = & \int{ [r(\boldsymbol{X}^{(\text{joint})},\tau){u}(\boldsymbol{X}^{\text{(joint)}}, \tau)]^{\top} \{\nabla_{\boldsymbol{X}^{(\text{miss})}}[
  \log{ \hat{p}( \boldsymbol{X}^{(\text{joint})}) } + \lambda \log{r(\boldsymbol{X}^{(\text{joint})},\tau)}] \} \mathrm{d}\boldsymbol{X}^{(\text{joint})}} 
 \\
 = & \int{ [{u}(\boldsymbol{X}^{\text{(joint)}}, \tau)]^{\top} [
 r(\boldsymbol{X}^{(\text{joint})},\tau)\nabla_{\boldsymbol{X}^{(\text{miss})}} 
  \log{ \hat{p}( \boldsymbol{X}^{(\text{joint})}) } 
  +  \lambda r(\boldsymbol{X}^{(\text{joint})},\tau)\nabla_{\boldsymbol{X}^{(\text{miss})}}\log{r(\boldsymbol{X}^{(\text{joint})},\tau)}  
  ] \mathrm{d}\boldsymbol{X}^{(\text{joint})}}
 \\
 = &
 \int{ [{u}(\boldsymbol{X}^{\text{(joint)}}, \tau)]^{\top} [r(\boldsymbol{X}^{(\text{joint})},\tau)\nabla_{\boldsymbol{X}^{(\text{miss})}}
  \log{ \hat{p}( \boldsymbol{X}^{(\text{joint})}) }  + \lambda \nabla_{\boldsymbol{X}^{(\text{miss})}}r(\boldsymbol{X}^{(\text{joint})},\tau) ] \mathrm{d}\boldsymbol{X}^{(\text{joint})}}
 \\
 \overset{\text{(ii)}}{=} & 
 \int{ r(\boldsymbol{X}^{(\text{joint})},\tau)[{u}^{\top}(\boldsymbol{X}^{\text{(joint)}}, \tau)\nabla_{\boldsymbol{X}^{(\text{miss})}}
  \log{ \hat{p}( \boldsymbol{X}^{(\text{joint})}) }  -  \lambda
  \nabla_{\boldsymbol{X}^{(\text{miss})}}\cdot {u}(\boldsymbol{X}^{\text{(joint)}}, \tau)
  ] 
  \mathrm{d}\boldsymbol{X}^{(\text{joint})}}
  \\
  = &  \mathbb{E}_{ r(\boldsymbol{X}^{(\text{joint})},\tau)}[ {u}^{\top}(\boldsymbol{X}^{\text{(joint)}}, \tau)\nabla_{\boldsymbol{X}^{(\text{miss})}} \log{\hat{p}( \boldsymbol{X}^{(\text{joint})} )}  - \lambda\nabla_{\boldsymbol{X}^{(\text{miss})}} \cdot {u}(\boldsymbol{X}^{\text{(joint)}}, \tau)],
\end{aligned}
\end{equation}
where (i) and (ii) are based on integration by parts. 

Similar to the proof of proposition~\ref{prop:steinMap}, we can restrict the velocity field in RKHS and find the steepest gradient boosting direction as follows according to~\Cref{eq:kernelMap,eq:kernelExpress,eq:velocityDecomp}:
\begin{equation}
\begin{aligned}
 & v^*(\boldsymbol{X}^{\text{(joint)}}, \tau)\\ 
   =& \mathop{\arg\max}_{v(\boldsymbol{X}^{\text{(joint)}}, \tau)\in\mathcal{H}^d}
    \begin{aligned}
    & \{\mathbb{E}_{ r(\boldsymbol{X}^{(\text{joint})},\tau)}[ 
      v^{\top}(\boldsymbol{X}^{\text{(joint)}}, \tau)\nabla_{\boldsymbol{X}^{(\text{miss})}} \log{\hat{p}( \boldsymbol{X}^{(\text{joint})} )}  \\
    & \quad \quad- \lambda\nabla_{\boldsymbol{X}^{(\text{miss})}} \cdot v(\boldsymbol{X}^{\text{(miss)}}, \tau)
        ]\}- \frac{1}{2}\Vert v(\boldsymbol{X}^{\text{(joint)}}, \tau)\Vert_{\mathcal{H}^d},
    \end{aligned}\\
  \overset{\text{(i)}}{=}& \mathop{\arg\max}_{v(\boldsymbol{X}^{\text{(joint)}}, \tau)\in\mathcal{H}^d} 
    \begin{aligned}
    & \{\mathbb{E}_{ r(\tilde{\boldsymbol{X}}^{(\text{joint})},\tau)}[ 
      \sum_{i=1}^{\infty}{\sqrt{\xi_i}\nabla_{\tilde{\boldsymbol{X}}^{(\text{miss})}} \log{\hat{p}( \tilde{\boldsymbol{X}}^{(\text{joint})} )^{\top}}v_i\phi_i(\tilde{\boldsymbol{X}}^{\text{(joint)}}, \tau)}\\
    & \quad \quad- \lambda\nabla_{\tilde{\boldsymbol{X}}^{(\text{miss})}} \cdot \sum_{i=1}^{\infty}{v_i\sqrt{\xi_i}\phi_i(\tilde{\boldsymbol{X}}^{\text{(joint)}}, \tau)}
          ]\}- \frac{1}{2}\sum_{i=1}^{\infty}{\Vert v_i\Vert^2},
    \end{aligned}\\
  \end{aligned}
\end{equation}
Take the right-hand-side of (i) with-respect-to $v_i$, and set it to $0$, we can get:
\begin{equation}
   \sqrt{\xi_i}\{\mathbb{E}_{ r(\tilde{\boldsymbol{X}}^{(\text{joint})},\tau)}[ 
      [{\nabla_{\tilde{\boldsymbol{X}}^{(\text{miss})}}} \log{\hat{p}( \tilde{\boldsymbol{X}}^{(\text{joint})})}]^\top\phi_i(\tilde{\boldsymbol{X}}^{\text{(joint)}}, \tau)
   - \lambda\nabla_{\tilde{\boldsymbol{X}}^{(\text{miss})}}{\phi_i(\tilde{\boldsymbol{X}}^{\text{(joint)}}, \tau)}
          ]\}- v_i =0.
\end{equation}
On this basis, $v_i^*$ can be given as follows:
\begin{equation}
  v_i^* =  \sqrt{\xi_i}\{\mathbb{E}_{ r(\tilde{\boldsymbol{X}}^{(\text{joint})},\tau)}[ 
    {[\nabla_{\tilde{\boldsymbol{X}}^{(\text{miss})}} \log{\hat{p}( \tilde{\boldsymbol{X}}^{(\text{joint})} )}]^\top\phi_i(\tilde{\boldsymbol{X}}^{\text{(joint)}}, \tau)}
 - \lambda\nabla_{\tilde{\boldsymbol{X}}^{(\text{miss})}}{\phi_i(\tilde{\boldsymbol{X}}^{\text{(joint)}}, \tau)}]\},
\end{equation}
and hence, $u(\boldsymbol{X}^{\text{(joint)}}, \tau)$ can be given as follows:
\begin{equation}
  \begin{aligned}
  & u(\boldsymbol{X}^{\text{(joint)}}, \tau) \\
=& \sum_{i=1}^{\infty}{\sqrt{\xi_i} v_i^* \phi_i(\boldsymbol{X}^{\text{(joint)}}, \tau)}\\
=&\mathbb{E}_{r(\boldsymbol{X}^{(\text{joint})},\tau)}\left[ 
\begin{aligned}  
& { - \lambda \nabla_{\tilde{\boldsymbol{X}}^{(\text{joint})}} \mathcal{K}(\boldsymbol{X}^{(\text{joint})}  , \tilde{\boldsymbol{X}}^{(\text{joint})})  } \\
&\quad \quad+ \nabla_{\tilde{\tilde{\boldsymbol{X}}}^{(\text{miss})}}\log{\hat{p}( \tilde{\boldsymbol{X}}^{(\text{joint})} )}\mathcal{K}(\boldsymbol{X}^{(\text{miss})}  , \tilde{\boldsymbol{X}}^{(\text{miss})}) 
\end{aligned} 
\right].
  \end{aligned}
\end{equation}

\textbf{Lower Bound Acquirement:}

Before proving the proposition, we should notice that given the unchanged observational data $\boldsymbol{X}^{\text{(obs)}}$, the distribution $p(\boldsymbol{X}^{\text{(obs)}})$ is a constant. On this basis, consider the definition of $\mathcal{F}_{\text{NER}}$ (right-hand-side of~\Cref{eq:NEGFunctional}), the first term and the second term are denoted by `term 1' and `term 2' for simplicity:
\begin{equation}
  \underbrace{ \mathbb{E}_{
    r(\boldsymbol{X}^{(\text{miss})})}[{\log{ \hat{p}( \boldsymbol{X}^{(\text{miss})}\vert \boldsymbol{X}^{(\text{obs})}) }}]}_{\text{term 1}}  +\lambda\times \left[\underbrace{-  \mathbb{H}[r(\boldsymbol{X}^{(\text{miss})})]}_{\text{term 2}}\right].
\end{equation}
For term 1, we can obtain the following derivation:
\begin{equation}\label{eq:term1Upper}
\begin{aligned}
  & \int{ r(\boldsymbol{X}^{(\text{miss})})\log{ \hat{p}( \boldsymbol{X}^{(\text{miss})}\vert \boldsymbol{X}^{(\text{obs})}) } \mathrm{d}\boldsymbol{X}^{(\text{miss})} }\\
   & \ge\int{ r(\boldsymbol{X}^{(\text{miss})})\log{ \hat{p}( \boldsymbol{X}^{(\text{miss})}\vert \boldsymbol{X}^{(\text{obs})}) } \mathrm{d}\boldsymbol{X}^{(\text{miss})} } +\underbrace{ \int{p(\boldsymbol{X}^{\text{(obs)}})\log{p(\boldsymbol{X}^{\text{(obs)}})}\mathrm{d}\boldsymbol{X}^{\text{(obs)}}}}_{\text{negative entropy (negative constant)}} \\
  =&  \iint{ p(\boldsymbol{X}^{(\text{obs})})r(\boldsymbol{X}^{(\text{miss})})\log{ \hat{p}( \boldsymbol{X}^{(\text{miss})}\vert \boldsymbol{X}^{(\text{obs})}) } \mathrm{d}\boldsymbol{X}^{(\text{miss})} \mathrm{d}\boldsymbol{X}^{(\text{obs})}} \\
  &\quad  +\underbrace{ \int{p(\boldsymbol{X}^{\text{(obs)}})\log{p(\boldsymbol{X}^{\text{(obs)}})}\mathrm{d}\boldsymbol{X}^{\text{(obs)}}}}_{\text{negative constant}} \\
  =&  \iint{ p(\boldsymbol{X}^{(\text{obs})})r(\boldsymbol{X}^{(\text{miss})})\log{ \hat{p}( \boldsymbol{X}^{(\text{miss})}\vert \boldsymbol{X}^{(\text{obs})}) } \mathrm{d}\boldsymbol{X}^{(\text{miss})} \mathrm{d}\boldsymbol{X}^{(\text{obs})}} \\
  &\quad  +\underbrace{ \iint{r(\boldsymbol{X}^{(\text{miss})})p(\boldsymbol{X}^{\text{(obs)}})\log{p(\boldsymbol{X}^{\text{(obs)}})}\mathrm{d}\boldsymbol{X}^{(\text{miss})} \mathrm{d}\boldsymbol{X}^{\text{(obs)}}}}_{\text{negative constant}} \\
  = &  \iint{ \underbrace{ p(\boldsymbol{X}^{(\text{obs})})r(\boldsymbol{X}^{(\text{miss})})}_{r( \boldsymbol{X}^{(\text{miss})}, \boldsymbol{X}^{(\text{obs})})}[\underbrace{\log{ \hat{p}( \boldsymbol{X}^{(\text{miss})}\vert \boldsymbol{X}^{(\text{obs})}) } + \log{p( \boldsymbol{X}^{(\text{obs})}) }}_{\log{\hat{p}( \boldsymbol{X}^{(\text{miss})}, \boldsymbol{X}^{(\text{obs})})}}]\mathrm{d}\boldsymbol{X}^{(\text{miss})} \mathrm{d}\boldsymbol{X}^{(\text{obs})}}\\
  = & \mathbb{E}_{r( \boldsymbol{X}^{(\text{miss})}, \boldsymbol{X}^{(\text{obs})})}[\log{\hat{p}( \boldsymbol{X}^{(\text{miss})}, \boldsymbol{X}^{(\text{obs})})}].
\end{aligned}
\end{equation}
Similarly, the term 2 can be reformulated as follows:
\begin{equation}\label{eq:term2Upper}
\begin{aligned}
 & -\mathbb{H}[r(\boldsymbol{X}^{(\text{miss})})] \\
 \ge  & -\mathbb{H}[r(\boldsymbol{X}^{(\text{miss})})] +\underbrace{ \int{p(\boldsymbol{X}^{\text{(obs)}})\log{p(\boldsymbol{X}^{\text{(obs)}})}\mathrm{d}\boldsymbol{X}^{\text{(obs)}}}}_{\text{negative entropy (negative constant)}} \\
 = & \iint{ p(\boldsymbol{X}^{(\text{obs})})r(\boldsymbol{X}^{(\text{miss})})\log{r(\boldsymbol{X}^{(\text{miss})})\mathrm{d}\boldsymbol{X}^{(\text{miss})} \mathrm{d}\boldsymbol{X}^{(\text{obs})}}} \\
 &\quad +\underbrace{ \iint{ p(\boldsymbol{X}^{(\text{obs})})r(\boldsymbol{X}^{(\text{miss})})\log{p(\boldsymbol{X}^{\text{(obs)}})} \mathrm{d}\boldsymbol{X}^{(\text{miss})} \mathrm{d}\boldsymbol{X}^{(\text{obs})} }}_{\text{negative entropy (negative constant)}} \\
 =& \iint{ \underbrace{p(\boldsymbol{X}^{(\text{obs})})r(\boldsymbol{X}^{(\text{miss})})}_{r(\boldsymbol{X}^{(\text{obs})}, \boldsymbol{X}^{(\text{miss})})}[ \underbrace{\log{r(\boldsymbol{X}^{(\text{miss})}) } + \log{p(\boldsymbol{X}^{\text{(obs)}})}}_{r(\boldsymbol{X}^{(\text{obs})}, \boldsymbol{X}^{(\text{miss})})}] \mathrm{d}\boldsymbol{X}^{(\text{miss})} \mathrm{d}\boldsymbol{X}^{(\text{obs})}}  \\
= & -\mathbb{H}[r(\boldsymbol{X}^{(\text{obs})}, \boldsymbol{X}^{(\text{miss})})] .
\end{aligned}
\end{equation}

Combine~\Cref{eq:term1Upper,eq:term2Upper}, we can obtain the following relationship:
\begin{equation}\label{eq:jointCondConstant}
  \mathcal{F}_{\text{NER}} -\text{const}= \mathcal{F}_{\text{joint-NER}} ,
\end{equation}
and constant $\text{const}$ is greater than $0$.

\end{proof}

\begin{corollary*}[\ref{eq:velocityEqual}]
The following equation can be satisfied:
\begin{equation}
   {u}( \boldsymbol{X}^{\text{(joint)}}, \tau) = u( \boldsymbol{X}^{\text{(miss)}}, \tau).
\end{equation}
\end{corollary*}
\begin{proof}
  This corollary can be easily proven by according to~\Cref{eq:jointCondConstant}:
  \begin{equation}\label{eq:IdenticalContinuity}
  \begin{aligned}
   & \mathcal{F}_{\text{NER}} = \mathcal{F}_{\text{joint-NER}} +\text{const}\\
\Rightarrow & \nabla_{\boldsymbol{X}^{\text{(miss)}}} \frac{\delta  \mathcal{F}_{\text{NER}} }{\delta r(\boldsymbol{X}^{\text{(miss)}})} 
=  \nabla_{\boldsymbol{X}^{\text{(miss)}}} \frac{\delta  \mathcal{F}_{\text{joint-NER}} +\text{const} }{\delta r(\boldsymbol{X}^{\text{(miss)}})} \\
\Rightarrow & \nabla_{\boldsymbol{X}^{\text{(miss)}}} \frac{\delta  \mathcal{F}_{\text{NER}} }{\delta r(\boldsymbol{X}^{\text{(miss)}})} 
=  \nabla_{\boldsymbol{X}^{\text{(miss)}}} \frac{\delta  \mathcal{F}_{\text{joint-NER}} }{\delta r(\boldsymbol{X}^{\text{(miss)}})}.
  \end{aligned}
  \end{equation}
Plug~\Cref{eq:IdenticalContinuity} into~\Cref{eq:continuityEquation,eq:gradFlowEquation}, we can see that the density functions for $\boldsymbol{X}^{\text{(miss)}}$ within functional $\mathcal{F}_{\text{NER}}$ and $\mathcal{F}_{\text{joint-NER}}$ are identical.
\end{proof}

\section{Detailed Information for KnewImp Implementation}\label{appendix-sec:ExperimentalDetailed}

\subsection{Forward Euler's Method for ODE Simulation}\label{appendix-forwardEuler}

Suppose we have the following ODE:
\begin{equation}
  \frac{\mathrm{d} x_\tau}{\mathrm{d}\tau} = f(x_\tau,\tau),
\end{equation}
and the initial value at $\tau=0$ is given $x_0 = x_{\text{init}}$, the value at time $\eta$ can be derived as follows:
\begin{equation}
  x_\eta = x_0 + \int_{0}^{\eta}{f(x_\tau,\tau)\mathrm{d}\tau}.
\end{equation}
To alleviate the intergal term, the forward Euler's method attempts to convert the integral term to summation term as follows:
\begin{equation}\label{eq:ODEEuler}
  x_\mathrm{\eta} = x_0 + f(x_\tau,\tau)\times(\eta-0).
\end{equation}
On this basis, the value at time $\mathrm{T}$ can be obtained by repeating~\Cref{eq:ODEEuler} from $\tau=0$ to $\tau=\mathrm{T}$, which is is the forward Euler's method.
\begin{algorithm}
\caption{Algorithm for Forward Euler's Method }\label{algoEulerMethod}
\SetKwInput{KwInput}{Input}
\SetKwInput{KwOutput}{Output}
\SetKwInput{Hyperparams}{Hyperparameters}
\SetKwInput{Initialization}{Initialization}
\SetKwInput{Parameter}{Parameter}
\SetKwInput{Training}{Training}
\SetKwInput{Testing}{Testing}
\KwInput{ODE $f(x_\tau,\tau)$; start point $\tau_0$; end point $\tau_\mathrm{T}$; step size $\eta$; initial value $x_{\tau_0}$.}
\KwOutput{Predicted value $x_{\tau_\mathrm{T}}$ at $\tau_\mathrm{T}$.}
Repeating times $j$ calculation: $j\leftarrow (\tau_\mathrm{T} - \tau_0)/{\eta}$\\
\For{$t = \tau_0+\eta, \tau_0+2\eta, ..., \tau_0+j\eta $}{
$x_{\tau_\mathrm{T}} \xleftarrow[]{} x({t-\eta}) + f(x_{{t-\eta}},{t-\eta})\times \eta $
}
\end{algorithm}
\newpage
\subsection{Algorithms for KnewImp}\label{appendix-Algorithm}

As we pointed out in Fig.~\ref{fig:illustration}, the KnewImp mainly consists of two parts, namely `Impute' and `Estimate'. Based on this, we first give the algorithm for the `Impute' and `Estimate' parts in this subsection. 
\begin{algorithm}[htbp]
  \caption{Impute Part Algorithm for KnewImp}\label{algo:impWGFAlgo}
  \SetKwInput{KwInput}{Input}
  \SetKwInput{KwOutput}{Output}
  \SetKwInput{Hyperparams}{Hyperparameters}
  \SetKwInput{Training}{Training}
  \SetKwInput{Testing}{Testing}
  \DontPrintSemicolon
  
  \KwInput{
    Initialized Missing Data $\boldsymbol{X}^{\text{(imp)}}$, Score Function: $\nabla_{{\boldsymbol{X}}^{\text{(joint)}}}\log\hat{p}({\boldsymbol{X}}^{\text{(joint)}})$, and Mask Matrix $\boldsymbol{M}$.
  }
  \Hyperparams{\\
    Simulation Time: $\mathrm{T}$, Discretization Step Size $\eta$, and Bandwidth of RBF kernel $h$.  }
  \KwOut{Imputed Result: $\boldsymbol{X}^{\text{(imp)}}$.}
  Set $\tau=0$,  \\
  \While {$\tau < \mathrm{T}$}{
    \tcc{Velocity Field Acquirement}
  $ \nabla_{  {\boldsymbol{X}^{\text{(miss)}}}}\log{\hat{p}(  \boldsymbol{X}^{\text{(joint)}})} \leftarrow \nabla_{  {\boldsymbol{X}^{\text{(joint)}}}}\log{\hat{p}(  \boldsymbol{X}^{\text{(joint)}})} \odot (\mathbbm{1}_{\mathrm{N} \times \mathrm{D}}-\boldsymbol{M}) +   0 \times \boldsymbol{M},$      \;
  $  {u}( \boldsymbol{X}^{\text{(joint)}}, \tau) \leftarrow \mathbb{E}_{r({\boldsymbol{X}}^{\text{(joint)}},\tau)}\left\{ 
    \begin{aligned}  
    & { - \lambda \nabla_{  {\tilde{\boldsymbol{X}}^{\text{(miss)}}}} \mathcal{K}(  \boldsymbol{X}^{\text{(joint)}}  ,    {\tilde{\boldsymbol{X}}^{\text{(joint)}}})  }\\
    &\quad \quad+ [\nabla_{  {\boldsymbol{X}^{\text{(miss)}}}}\log{\hat{p}(  \boldsymbol{X}^{\text{(joint)}})}]^\top\mathcal{K}(  \boldsymbol{X}^{\text{(joint)}}  ,   \tilde{\boldsymbol{X}}^{\text{(joint)}}) 
    \end{aligned} 
    \right\},$ \;
  \tcc{ODE Simulation By Forward Euler's Method}
  $\boldsymbol{X}^{\text{(imp)}}\leftarrow\boldsymbol{X}^{\text{(imp)}} + \eta \times {u}( \boldsymbol{X}^{\text{(joint)}}, \tau) ,$         \;
  $\tau \leftarrow \tau + 1.$                         \;
  }  
\end{algorithm}
\begin{algorithm}[htbp]
  \caption{Estimate Part Algorithm for KnewImp}\label{algo:impScoreEstimate}
  \SetKwInput{KwInput}{Input}
  \SetKwInput{KwOutput}{Output}
  \SetKwInput{Hyperparams}{Hyperparameters}
  \SetKwInput{Training}{Training}
  \SetKwInput{Testing}{Testing}
  \DontPrintSemicolon
  \KwInput{Imputed Data $\boldsymbol{X}^{\text{(imp)}}$, 
  and $\nabla_{{\boldsymbol{X}}^{\text{(joint)}}}\log\hat{p}({\boldsymbol{X}}^{\text{(joint)}})$ parameterized by Neural Network with Parameter $\theta$.
  }
  \Hyperparams{\\Network Learning Rate $lr$, Training Epoch $\mathcal{E}$, and Network Hidden Unit $\text{HU}_{\text{score}}$. }
  \KwOutput{Score Function: $\nabla_{{\boldsymbol{X}}^{\text{(joint)}}}\log\hat{p}({\boldsymbol{X}}^{\text{(joint)}})$.}
  \While {$e \leq \mathcal{E}$}{
    \tcc{Data Noising}
    $\hat{\boldsymbol{X}}^{\text{(joint)}}\leftarrow{\boldsymbol{X}}^{\text{(joint)}} + \epsilon,\epsilon\sim\mathcal{N}(\mathbf{0}, \sigma^2\mathbf{I}),$\;
   $ \nabla_{\hat{\boldsymbol{X}}^{\text{(joint)}}} \log{q_{\sigma}(\hat{\boldsymbol{X}}^{\text{(joint)}}\vert {\boldsymbol{X}}^{\text{(joint)}})} \leftarrow-\frac{\hat{\boldsymbol{X}}^{\text{(joint)}} -\boldsymbol{X}^{\text{(joint)}} }{\sigma^2},$\;
   \tcc{Score Function Training}
   $ \mathcal{L}_{\text{DSM}}\leftarrow\frac{1}{2}\mathbb{E}_{q_{\sigma}(\hat{\boldsymbol{X}}^{\text{(joint)}}\vert {\boldsymbol{X}}^{\text{(joint)}})}[\Vert 
   \nabla_{{\hat{\boldsymbol{X}}}^{\text{(joint)}}}\log\hat{p}({\boldsymbol{X}}^{\text{(joint)}})
   - \nabla_{\hat{\boldsymbol{X}}^{\text{(joint)}}} \log{q_{\sigma}(\hat{\boldsymbol{X}}^{\text{(joint)}}\vert {\boldsymbol{X}}^{\text{(joint)}})}  \Vert^2],$\;
   $\theta \leftarrow \theta - lr\times \nabla_{\theta}{\mathcal{L}_{\text{DSM}}}.$\;

  }
\end{algorithm}
\\On this basis, the algorithm for KnewImp is summarized as follows:
\begin{algorithm}[htbp]
  \caption{KnewImp Algorithm for MDI}\label{algo:knewImpAlgorithm}
  \SetKwInput{KwInput}{Input}
  \SetKwInput{KwOutput}{Output}
  \SetKwInput{Hyperparams}{Hyperparameters}
  \SetKwInput{Training}{Training}
  \SetKwInput{Testing}{Testing}
  \DontPrintSemicolon

  \KwInput{
    Missing Data $\boldsymbol{X}^{\text{(miss)}}$, 
    and Mask Matrix $\boldsymbol{M}$.
  }
  \Hyperparams{\\
  Loop Time: $\mathcal{T}$,
    Simulation Time: $\mathrm{T}$, Discretization Step Size $\eta$, Bandwidth of RBF kernel $h$, Network Learning Rate $lr$, Training Epoch $\mathcal{E}$, and Network Hidden Unit $\text{HU}_{\text{score}}$. }
  $\boldsymbol{X}^{\text{(imp)}} \leftarrow \text{Initialize}(\boldsymbol{X}^{\text{(imp)}})$ \\
  \While {$t < \mathcal{T}$}{
    \tcc{`Estimate' Part}
    $\nabla_{{\boldsymbol{X}}^{\text{(joint)}}}\log\hat{p}({\boldsymbol{X}}^{\text{(joint)}})\leftarrow\text{\Cref{algo:impScoreEstimate}}$\;
    \tcc{`Impute' Part}
    $\boldsymbol{X}^{\text{(imp)}}\leftarrow\text{\Cref{algo:impWGFAlgo}}$\;

  }

\end{algorithm}
\newpage
\section{Detailed Information for Experiments}\label{appendix-datasetForExperiments}

\subsection{Background \& Simulation of Missing Data}\label{appendix-subsec:missingSimulation}
According to reference~\cite{rubin1976inference}, missing data can be classified into three categories: Missing Completely at Random (MCAR), where the absence of data is completely unrelated to any observed or unobserved variables; Missing at Random (MAR), where the likelihood of missing data depends solely on observed data; and Missing Not at Random (MNAR), where missingness is influenced by unobserved data. In the cases of MCAR and MAR, the patterns of missing data are considered `ignorable' because it is unnecessary to explicitly model the distribution of the missing values. Conversely, MNAR scenarios, where missing data can introduce significant biases that are not easily corrected without imposing domain-specific assumptions, constraints, or parametric forms on the missingness mechanism, present more complex challenges~\cite{muzellec2020missing,jarrett2022hyperimpute}. Therefore, our discussion is primarily focused on numerical tabular data within the MCAR and MAR contexts.

To simulate missing data, we adopt the methodologies outlined in reference~\cite{jarrett2022hyperimpute}:
\begin{itemize}[leftmargin=*]
\item{\textbf{MAR:}  Initially, a random subset of features is selected to remain non-missing. The masking of the remaining features is conducted using a logistic model, which employs the non-missing features as predictors. This model is parameterized with randomly selected weights, and the bias is adjusted to achieve the desired missingness rate.}
\item{\textbf{MCAR:} For each data point, the masking variable is generated from a Bernoulli distribution with a predetermined fixed mean, ensuring that the probability of missingness is the same across all data points.}
\item{\textbf{MNAR:} Although MNAR scenarios are not the primary focus of this manuscript, we include experiments in this context. Missingness is introduced either by additional masking of the MAR-selected features using a Bernoulli process with a fixed mean, or through direct self-masking of values using interval-censoring techniques. In this paper, we mainly consider the former strategy. In other words, the mechanism of MNAR we used in this paper is identical to the previously described MAR mechanism, but the inputs of the logistic model are then masked by an MCAR mechanism. }
\end{itemize}

Based on this, the datasets listed in~\Cref{tab:DatasetsDescription} are adopted in this paper.
\begin{table}[htbp]
  \caption{Detailed dataset descriptions, where `Dimension' denotes the variate number of each dataset. `Numer' denotes the total number of item.}
  \label{tab:DatasetsDescription}
  \centering
  \begin{tabular}{l|l|l|l}
  \toprule
  Abbreviation & Dataset Name               & Numer ($\mathrm{N}$) & Dimension ($\mathrm{D}$) \\ \midrule
  BT           & Blood Transfusion         & 748   & 4         \\
  BCD          & Breast Cancer Diagnostic  & 569   & 30        \\
  CC           & Concrete Compression      & 1030  & 7         \\
  CBV          & Connectionist Bench Vowel & 990   & 10        \\
  IS           & Ionosphere                & 351   & 34        \\
  PK           & Parkinsons                & 195   & 23        \\
  QB           & QSAR Biodegradation       & 1055  & 41        \\
  WQW          & Wine Quality White        & 4898  & 11        \\ \bottomrule
  \end{tabular}
  \end{table}
\newpage
\subsection{Training Protocols of Different Models }\label{appendix-TrainingProtocols}
In this study, we employ a two-layer Multi-Layer Perceptron (MLP) to model $\nabla_{\boldsymbol{X}^{\text{(miss)}}}\log{\hat{p}}(\boldsymbol{X}^{\text{(miss)}})$. Each layer is configured with 256 hidden units ($\textrm{HU}_{\text{score}}$). The activation function is set as `Swish' function~\cite{ramachandran2017searching}, and the variance scale $\sigma$ for DSM is set as 0.1. The network is trained using an Adam optimizer with a learning rate of $\text{1.0} \times \text{10}^{\text{-3}}$, and the batch size is dynamically set to $\mathrm{N}$. For the `impute' part, we specify a simulation time ($\mathrm{T}$) of 500, a step size of 0.1, and a bandwidth ($h$) of 0.5. The loop time $\mathcal{T}$ for KnewImp is set as 2. For baseline models, the batch size is uniformly set at 512. Models incorporating neural architectures are optimized with the Adam optimizer at a learning rate of $1.0 \times 10^{-2}$, in line with the practices recommended by Kingma and Ba~\cite{kingma2014adam}. The MIWAE model features a latent dimension of 16 and 32 hidden units. The settings for the TDM model include 16 hidden units per layer and two layers. 
For the CSDI\_T and MissDiff models, the parameters are set as follows: particle number at 50, diffusion embedding dimension at 128, batch size at 512 (for Sink and TDM, if $\mathrm{N}<512$, the batch size is set as $2^{\lfloor\frac{\mathrm{N}}{2}\rfloor}$), and learning rate at $1.0 \times 10^{-3}$, and diffusion steps at 100. 

To ensure fairness and reproducibility, all experiments are conducted on a workstation equipped with an Intel Xeon E5 processor with four cores, eight Nvidia GTX 1080 GPUs, and 128 GB of RAM. Each experiment is replicated at least five times, utilizing six distinct random seeds to guarantee robustness in the results.

\subsection{Evaluation Protocols}\label{appendix-subsec:evaluationProtocols}
Imputation methods are assessed using two metrics: the mean absolute error (MAE), which is a pointwise metric, and the squared Wasserstein distance (abbreviated as Wass), which evaluates empirical distributions. Based on reference~\cite{muzellec2020missing}, consider a dataset $\boldsymbol{X} \in \mathbb{R}^{\mathrm{N}\times\mathrm{D}}$ with missing values. For any entry $(i, j)$ identified as missing, let $\boldsymbol{X}^{\text{(imp)}}[i,d]$ represent the corresponding imputation, and $\boldsymbol{X}^{\text{(true)}}[i,d]$ denote the ground truth. Define $\boldsymbol{m}_0$ as the total number of missing entries, $\boldsymbol{m}_0 \coloneqq \#\{(i,d) , \boldsymbol{M}[i,d]=0\}$, and $\boldsymbol{m}_1$ as the number of data points that have at least one missing value, $\boldsymbol{m}_1 \coloneqq \#\{i : \exists d, \boldsymbol{M}[i,d]=0\}$. The set $\boldsymbol{M}_1$ encompasses indices of data points with any missing values, $\boldsymbol{M}_1 \coloneqq \{i : \exists d, \boldsymbol{M}[i,d]=0\}$. The metrics used to evaluate the accuracy of the imputation, MAE and Wass, are calculated as follows:
\begin{equation}\label{eq:MAEDef}
  \text{MAE} \coloneqq \frac{1}{\boldsymbol{m}_0}\sum_{(i,d):\boldsymbol{M}[i,d]=0}{\left| \boldsymbol{X}^{\text{(true)}}[i,d] -  \boldsymbol{X}^{\text{(imp)}}[i,d]  \right|},
\end{equation}
\begin{equation}\label{eq:WassDef}
  \text{Wass} \coloneqq \mathcal{W}^2_2\left[ \frac{1}{\boldsymbol{m}_1}\sum_{k=1}^{K}{\varDelta_{\boldsymbol{X}^{\text{(imp)}}_{\boldsymbol{M}_1}}}, \frac{1}{\boldsymbol{m}_1}\sum_{k=1}^{K}{\varDelta_{\boldsymbol{X}^{\text{(true)}}_{\boldsymbol{M}_1}}} \right], 
\end{equation}
where $\varDelta_{\boldsymbol{x}}$ is the Dirac distribution (measure) concentrated on $\boldsymbol{x}$.

\newpage
\section{Additional Empirical Evidence }\label{appendix-sec:additionalEmpiricalEvidence}
\subsection{Additional Experimental Results with MNAR Scenario}\label{appendix-sec:BaselineComparison}
In this subsection, we expand upon the results presented in~\Cref{tab:ComparisonResults} by including the MNAR scenario, as detailed in~\Cref{appendix-tab:ComparisonResults}. Additionally, we report on the outcomes of an ablation study and sensitivity analysis in~\Cref{appendix-tab:AblationStudy,appendix-tab:stdAblationStudy} and~\Cref{appendix-fig:sensResult}. These extended results lead to several pertinent observations:

\begin{itemize}[leftmargin=*]
  \item{Across three different missing data scenarios, the models consistently exhibit the poorest performance under the MNAR condition. For instance, in the MNAR scenario, nearly all models show a significant decrease in imputation accuracy and an increase in standard deviation. This supports the assertion made in~\Cref{appendix-subsec:missingSimulation} that addressing the MNAR scenario requires the incorporation of relevant domain knowledge to mitigate biases introduced by the pattern of missing data.}
  \item{The findings from the ablation study under the MNAR scenario are consistent with those observed in both MAR and MCAR scenarios in~\Cref{subsec:abResults}. This consistency underscores the importance of including the NER term and adopting the joint distribution modeling approach.}
  \item{Similarly, the results from the sensitivity analysis under the MNAR scenario align with those from MAR and MCAR scenarios in~\Cref{subsec:sensResults}. This alignment reinforces our interpretations of model performance across different groups of hyperparameters under MAR and MCAR scenarios.}
\end{itemize}

\begin{table}[htbp]
  \caption{Performance of MAE and Wass metrics at 30\% missing rate, and `$^*$' marks that KnewImp outperforms significantly at $p$-value < 0.05 over paired samples $t$-test. 
  }
  \label{appendix-tab:ComparisonResults}
  \resizebox{1.0\linewidth}{!}{
\begin{tabular}{c|l|ll|ll|ll|ll|ll|ll|ll|ll}
\toprule
\multirow{2}{*}{Scenario} & \multicolumn{1}{c|}{\multirow{2}{*}{Model}} & \multicolumn{2}{c|}{BT} & \multicolumn{2}{c|}{BCD} & \multicolumn{2}{c|}{CC} & \multicolumn{2}{c|}{CBV} & \multicolumn{2}{c|}{IS} & \multicolumn{2}{c|}{PK} & \multicolumn{2}{c|}{QB} & \multicolumn{2}{c}{WQW} \\ \cmidrule{3-18} 
                          & \multicolumn{1}{c|}{}                       & MAE       & Wass       & MAE       & Wass        & MAE       & Wass       & MAE       & Wass        & MAE      & Wass        & MAE       & Wass       & MAE       & Wass        & MAE       & Wass        \\ \midrule
\multirow{8}{*}{MAR}      & CSDI\_T                                     & 0.93$^*$  & 3.44$^*$   & 0.92$^*$  & 18.2$^*$    & 0.85$^*$  & 2.82$^*$   & 0.81$^*$  & 3.86$^*$    & 0.70$^*$  & 16.9$^*$   & 0.99$^*$  & 15.9$^*$  & 0.65$^*$  & 20.1$^*$    & 0.77$^*$  & 4.13$^*$    \\
                          & MissDiff                                    & 0.85$^*$   & 2.20$^*$   & 0.91$^*$   & 16.5$^*$  & 0.87$^*$   & 1.59$^*$  & 0.83$^*$   & 3.87$^*$   & 0.72$^*$  & 13.3$^*$  & 0.92$^*$  & 17.1$^*$  & 0.63$^*$   & 26.3$^*$  & 0.75$^*$   & 6.88$^*$   \\
                          & GAIN                                        & 0.75$^*$  & 0.65$^*$   & 0.54$^*$  & 1.64$^*$    & 0.75$^*$  & 0.67$^*$   & 0.68$^*$  & 0.68$^*$    & 0.56$^*$ & 1.88$^*$    & \uline{0.59}$^*$  & \uline{1.90}$^*$    & 0.65$^*$  & 5.05$^*$    & 0.68$^*$  & 0.87$^*$    \\
                          & MIRACLE                                     & \uline{0.62}$^*$  & \uline{0.38}       & 0.55$^*$  & 1.92$^*$    & \textbf{0.43}      & \textbf{0.25}       & \uline{0.55}$^*$  & \uline{0.46}$^*$    & 3.39$^*$ & 35.1$^*$   & 4.14$^*$  & 34.1$^*$  & \uline{0.46}      & \textbf{2.87}$^*$    & \uline{0.51}$^*$  & \uline{0.56}        \\
                          & MIWAE                                       & 0.64      & 0.53       & \uline{0.52}$^*$  & \uline{1.54}$^*$    & 0.76$^*$  & 0.64$^*$   & 0.82$^*$  & 0.92$^*$    & \uline{0.50}$^*$  & \uline{1.87}$^*$    & 0.65$^*$  & 1.98$^*$   & 0.55$^*$  & 5.05$^*$    & 0.62$^*$  & 0.75$^*$    \\
                          & Sink                                        & 0.87$^*$  & 0.92$^*$   & 0.92$^*$  & 3.84$^*$    & 0.88$^*$  & 0.83$^*$   & 0.84$^*$  & 0.98$^*$    & 0.75$^*$ & 2.43$^*$    & 0.94$^*$  & 3.61$^*$   & 0.65$^*$  & 4.71$^*$    & 0.76$^*$  & 1.04$^*$    \\
                          & TDM                                         & 0.83$^*$  & 0.89$^*$   & 0.83$^*$  & 3.47$^*$    & 0.81$^*$  & 0.73$^*$   & 0.76$^*$  & 0.85$^*$    & 0.62$^*$ & 1.96$^*$    & 0.86$^*$  & 3.36$^*$   & 0.59$^*$  & 4.46$^*$    & 0.73$^*$  & 0.99$^*$    \\
                          & \textbf{KnewImp}                                     & \textbf{0.52}      & \textbf{0.38}       & \textbf{0.34}      & \textbf{0.82}        & \uline{0.35}      & \uline{0.25}       & \textbf{0.31}      & \textbf{0.20}         & \textbf{0.39}     & \textbf{1.31}        & \textbf{0.44}      & \textbf{1.21}       & \textbf{0.45}      & \uline{3.50}        & \textbf{0.46}      & \textbf{0.55}        \\ \midrule
\multirow{8}{*}{MCAR}     & CSDI\_T                                     & 0.73$^*$  & 1.93$^*$   & 0.73$^*$  & 15.5$^*$   & 0.85$^*$  & 2.71$^*$   & 0.83$^*$  & 3.79$^*$    & 0.76$^*$ & 15.2$^*$   & 0.72$^*$  & 12.4$^*$  & 0.57$^*$  & 19.9$^*$   & 0.78$^*$  & 4.11$^*$    \\
                          & MissDiff                                    & 0.72$^*$   & 1.62$^*$  & 0.73$^*$   & 14.4$^*$  & 0.84$^*$   & 1.23$^*$  & 0.82$^*$   & 3.31$^*$   & 0.75$^*$  & 13.0$^*$  & 0.71$^*$  & 14.1$^*$  & 0.56$^*$   & 19.7$^*$  & 0.76$^*$   & 4.95$^*$   \\
                          & GAIN                                        & 0.72$^*$  & 0.39$^*$   & \uline{0.38}$^*$  & \uline{1.41}$^*$    & 0.78$^*$  & 0.73$^*$   & 0.72$^*$  & 0.99$^*$    & \uline{0.57}$^*$ & \uline{3.72}$^*$    & \uline{0.46}$^*$  & \uline{1.70}        & 0.42$^*$  & \uline{3.62}        & 0.73$^*$  & 1.14$^*$    \\
                          & MIRACLE                                     & \uline{0.52}      & \textbf{0.15}$^*$   & 0.44$^*$  & 1.94$^*$    & \uline{0.53}$^*$  & \uline{0.35}       & \uline{0.61}$^*$  & \uline{0.72}$^*$    & 2.99$^*$ & 52.9$^*$   & 3.38$^*$  & 42.8$^*$  & \uline{0.35}      & \textbf{2.71}$^*$    & \uline{0.56}$^*$  & \textbf{0.75}        \\
                          & MIWAE                                       & 0.58$^*$  & 0.24       & 0.50$^*$   & 2.55$^*$    & 0.76$^*$  & 0.69$^*$   & 0.83$^*$  & 1.24$^*$    & 0.64$^*$ & 4.95$^*$    & 0.51$^*$  & 2.05$^*$   & 0.48$^*$  & 5.87$^*$    & 0.67$^*$  & 0.95$^*$    \\
                          & Sink                                        & 0.73$^*$  & 0.48$^*$   & 0.75$^*$  & 4.39$^*$    & 0.84$^*$  & 0.85$^*$   & 0.82$^*$  & 1.27$^*$    & 0.75$^*$ & 4.94$^*$    & 0.74$^*$  & 3.36$^*$   & 0.61$^*$  & 5.92$^*$    & 0.76$^*$  & 1.25$^*$    \\
                          & TDM                                         & 0.68$^*$  & 0.42$^*$   & 0.63$^*$  & 3.57$^*$    & 0.77$^*$  & 0.75$^*$   & 0.77$^*$  & 1.15$^*$    & 0.66$^*$ & 4.20$^*$     & 0.64$^*$  & 2.89$^*$   & 0.52$^*$  & 5.34$^*$    & 0.74$^*$  & 1.20$^*$     \\
                          & \textbf{KnewImp}                                     & \textbf{0.48}      & \uline{0.18}       & \textbf{0.25}      & \textbf{0.80}         & \textbf{0.47}      & \textbf{0.34}       & \textbf{0.42}      & \textbf{0.44}        & \textbf{0.44}     & \textbf{3.05}        & \textbf{0.32}      & \textbf{1.01}       & \textbf{0.34}      & 3.66        & \textbf{0.53}      & \uline{0.76}        \\ \midrule
\multirow{8}{*}{MNAR}     & CSDI\_T                                     & 0.83$^*$  & 2.29$^*$   & 0.82$^*$  & 15.7$^*$   & 0.85$^*$  & 2.78$^*$   & 0.83$^*$  & 3.83$^*$    & 0.74$^*$ & 15.5$^*$   & 0.84$^*$  & 12.2$^*$   & 0.62$^*$  & 19.8$^*$   & 0.78$^*$  & 4.09$^*$    \\
                          & MissDiff                                    & 0.78$^*$   & 1.43$^*$  & 0.81$^*$   & 14.9$^*$  & 0.84$^*$   & 1.27$^*$  & 0.83$^*$   & 3.53$^*$   & 0.72$^*$  & 13.3$^*$  & 0.81$^*$  & 16.0$^*$  & 0.61$^*$   & 21.6$^*$  & 0.76$^*$   & 4.70$^*$    \\
                          & GAN                                        & 0.77$^*$  & 0.57$^*$   & 0.62$^*$  & 3.94$^*$    & 0.78$^*$  & 0.79$^*$   & 0.78$^*$  & 1.15$^*$    & 0.71$^*$ & 4.85$^*$    & 0.70$^*$   & 4.20$^*$    & 0.76$^*$  & 10.5$^*$   & 0.75$^*$  & 1.23$^*$    \\
                          & MIRACLE                                     & \uline{0.63}      & \textbf{0.35}       & 0.60$^*$   & 4.26$^*$    & \uline{0.52}$^*$  & \uline{0.35}       & \uline{0.63}$^*$  & \uline{0.77}$^*$    & 3.10$^*$  & 55.6$^*$   & 3.49$^*$  & 44.8$^*$  & \uline{0.52}$^*$  & \uline{5.61}        & \uline{0.58}$^*$  & \uline{0.80}         \\
                          & MIWAE                                       & 0.66$^*$  & 0.42       & \uline{0.56}$^*$  & \uline{3.31}$^*$    & 0.74$^*$  & 0.68$^*$   & 0.85$^*$  & 1.30$^*$     & \uline{0.59}$^*$ & 4.33$^*$    & \uline{0.60}$^*$   & \uline{3.06}$^*$   & 0.53$^*$  & 7.21$^*$    & 0.67$^*$  & 0.97$^*$    \\
                          & Sink                                        & 0.79$^*$  & 0.68$^*$   & 0.83$^*$  & 5.90$^*$     & 0.83$^*$  & 0.89$^*$   & 0.84$^*$  & 1.36$^*$    & 0.75$^*$ & 4.86$^*$    & 0.84$^*$  & 5.02$^*$   & 0.64$^*$  & 7.23$^*$    & 0.77$^*$  & 1.33$^*$    \\
                          & TDM                                         & 0.76$^*$  & 0.64$^*$   & 0.74$^*$  & 5.18$^*$    & 0.76$^*$  & 0.77$^*$   & 0.79$^*$  & 1.24$^*$    & 0.64$^*$ & \uline{4.02}$^*$    & 0.76$^*$  & 4.54$^*$   & 0.57$^*$  & 6.45        & 0.74$^*$  & 1.23$^*$    \\
                          & \textbf{KnewImp}                                     & \textbf{0.60}       & \uline{0.35}       & \textbf{0.32}      & \textbf{1.46}        & \textbf{0.44}      & \textbf{0.34}       & \textbf{0.46}      & \textbf{0.52}        & \textbf{0.40}      & \textbf{2.68}        & \textbf{0.39}      & \textbf{1.56}       & \textbf{0.42}      & \textbf{5.57}        & \textbf{0.55}      & \textbf{0.81}        \\ \bottomrule
\end{tabular}
}
\end{table}

\begin{table}[htbp]
  \caption{Standard deviation of MAE and Wass metrics at 30\% missing rate. 
  }
  \label{appendix-tab:StdComparisonResults}
  \resizebox{1.0\linewidth}{!}{
\begin{tabular}{c|l|ll|ll|ll|ll|ll|ll|ll|ll}
  \toprule
  \multirow{2}{*}{Scenario} & \multicolumn{1}{c|}{\multirow{2}{*}{Model}} & \multicolumn{2}{c|}{BT}             & \multicolumn{2}{c|}{BCD}            & \multicolumn{2}{c|}{CC}             & \multicolumn{2}{c|}{CBV}            & \multicolumn{2}{c|}{IS}             & \multicolumn{2}{c|}{PK}             & \multicolumn{2}{c|}{QB}             & \multicolumn{2}{c}{WQW} \\ \cmidrule{3-18} 
  & \multicolumn{1}{c|}{}                       & MAE     & \multicolumn{1}{l|}{Wass} & MAE     & \multicolumn{1}{l|}{Wass} & MAE     & \multicolumn{1}{l|}{Wass} & MAE     & \multicolumn{1}{l|}{Wass} & MAE     & \multicolumn{1}{l|}{Wass} & MAE     & \multicolumn{1}{l|}{Wass} & MAE     & \multicolumn{1}{l|}{Wass} & MAE        & Wass       \\ \midrule
\multirow{8}{*}{MAR}      & CSDI\_T                                     & 4.8E-2 & 2.9E-1 & 5.9E-2 & 2.6E+0 & 2.7E-2 & 1.4E-1 & 2.0E-2 & 9.2E-2 & 2.1E-2 & 2.0E+0 & 3.9E-2 & 3.1E+0 & 2.0E-2 & 8.3E-1 & 1.8E-2 & 7.7E-2  \\
  & MissDiff                                    & 4.0E-2 & 4.9E-1 & 3.1E-2 & 2.6E+0 & 3.4E-2 & 2.6E-1 & 1.9E-2 & 1.2E+0 & 5.4E-2 & 5.6E-1 & 3.0E-2 & 2.4E+0 & 1.8E-2 & 4.9E+0 & 1.6E-2 & 1.3E+0  \\
  & GAIN                                        & 1.0E-1 & 1.6E-1                   & 4.3E-2 & 2.3E-1                   & 3.4E-2 & 8.9E-2                   & 1.9E-2 & 4.0E-2                   & 5.8E-2 & 3.4E-1                   & 5.4E-2 & 3.7E-1                   & 6.9E-2 & 8.4E-1                   & 3.9E-2    & 5.6E-2    \\
  & MIRACLE                                     & 2.0E-2 & 6.3E-2                   & 5.1E-2 & 4.0E-1                   & 1.1E-2 & 2.0E-2                   & 1.8E-2 & 2.1E-2                   & 6.7E-1 & 1.2E+1                   & 4.2E-1 & 5.6E+0                   & 1.6E-2 & 1.9E-1                   & 1.3E-2    & 2.8E-2    \\
  & MIWAE                                       & 6.5E-2 & 1.5E-1                   & 5.2E-2 & 2.5E-1                   & 6.1E-2 & 1.2E-1                   & 2.4E-2 & 4.6E-2                   & 5.3E-2 & 1.8E-1                   & 2.7E-2 & 1.8E-1                   & 3.8E-2 & 2.8E-1                   & 1.9E-2    & 2.7E-2    \\
  & Sink                                        & 4.6E-2 & 1.2E-1                   & 3.2E-2 & 1.6E-1                   & 2.6E-2 & 7.4E-2                   & 2.4E-2 & 5.5E-2                   & 5.0E-2 & 2.0E-1                   & 1.7E-2 & 9.3E-2                   & 1.7E-2 & 7.6E-2                   & 2.2E-2    & 4.4E-2    \\
  & TDM                                         & 4.5E-2 & 1.2E-1                   & 1.9E-2 & 8.2E-2                   & 3.4E-2 & 8.6E-2                   & 2.6E-2 & 5.2E-2                   & 6.2E-2 & 2.1E-1                   & 2.7E-2 & 1.7E-1                   & 1.1E-2 & 8.1E-2                   & 2.1E-2    & 4.7E-2    \\
  & \textbf{KnewImp}                               & 2.0E-2 & 4.0E-2                   & 2.7E-2 & 1.1E-1                   & 5.6E-2 & 6.4E-2                   & 1.6E-2 & 2.2E-2                   & 1.9E-2 & 1.1E-1                   & 1.1E-2 & 8.8E-2                   & 1.9E-2 & 2.7E-1                   & 1.6E-2    & 3.3E-2    \\ \midrule
\multirow{8}{*}{MCAR}     & CSDI\_T             & 1.0E-2 & 1.5E-1 & 8.7E-3 & 5.7E-1 & 8.7E-3 & 8.2E-2 & 4.6E-3 & 4.6E-2 & 4.6E-3 & 3.6E-1 & 1.1E-2 & 8.7E-1 & 3.7E-3 & 3.1E-1 & 1.2E-2 & 5.0E-02 \\
  & MissDiff                                    &   6.4E-3 & 3.3E-1 & 8.2E-3 & 8.3E-1 & 3.5E-3 & 2.3E-1 & 5.9E-3 & 8.4E-1 & 7.1E-3 & 1.8E-1 & 4.6E-3 & 2.5E+0 & 6.2E-3 & 2.4E+0 & 4.1E-3 & 6.5E-1  \\
  & GAIN                                        & 6.1E-2 & 1.0E-1                   & 7.9E-3 & 2.6E-2                   & 2.4E-2 & 5.2E-2                   & 1.4E-2 & 3.4E-2                   & 1.8E-2 & 1.8E-1                   & 4.5E-2 & 3.8E-1                   & 3.7E-3 & 1.8E-1                   & 2.0E-2    & 5.5E-2    \\
  & MIRACLE                                     & 2.6E-2 & 8.4E-3                   & 1.6E-2 & 1.9E-1                   & 1.7E-2 & 1.5E-2                   & 5.2E-3 & 1.4E-2                   & 4.3E-2 & 1.2E+0                   & 4.6E-2 & 1.1E+0                   & 1.0E-2 & 1.7E-1                   & 1.1E-3    & 5.5E-3    \\
  & MIWAE                                       & 3.1E-2 & 3.9E-2                   & 4.8E-3 & 4.9E-2                   & 7.6E-3 & 1.3E-2                   & 1.6E-2 & 4.3E-2                   & 9.4E-3 & 1.3E-1                   & 1.0E-2 & 7.8E-2                   & 9.1E-3 & 2.7E-1                   & 4.1E-3    & 9.5E-3    \\
  & Sink                                        & 7.3E-3 & 3.4E-2                   & 4.6E-3 & 2.5E-2                   & 7.0E-3 & 6.6E-3                   & 4.5E-3 & 4.5E-3                   & 4.2E-3 & 1.4E-1                   & 1.0E-2 & 5.9E-2                   & 3.3E-3 & 1.9E-1                   & 3.9E-3    & 1.2E-2    \\
  & TDM                                         & 4.9E-3 & 2.8E-2                   & 8.7E-3 & 3.1E-2                   & 9.8E-3 & 6.6E-3                   & 6.9E-3 & 7.9E-3                   & 1.0E-3 & 1.9E-3                   & 3.3E-3 & 3.6E-2                   & 9.3E-3 & 1.5E-1                   & 5.1E-3    & 1.3E-2    \\
  & \textbf{KnewImp}                                     & 3.3E-3 & 3.7E-3                   & 1.9E-3 & 4.6E-2                   & 1.1E-2 & 1.8E-2                   & 4.1E-3 & 1.8E-2                   & 5.7E-3 & 1.1E-1                   & 6.4E-3 & 3.7E-2                   & 4.8E-3 & 1.7E-1                   & 2.2E-3    & 1.1E-2    \\ \midrule
\multirow{8}{*}{MNAR}     & CSDI\_T             & 2.9E-2 & 2.2E-1 & 8.7E-3 & 7.8E-1 & 2.2E-2 & 1.3E-1 & 7.4E-3 & 7.4E-2 & 1.0E-2 & 5.9E-1 & 2.2E-02 & 1.8E+0 & 2.6E-3 & 4.6E-1 & 2.6E-3 & 4.6E-1 \\
  & MissDiff                                    & 3.7E-2 & 3.7E-1 & 2.4E-3 & 9.7E-1 & 5.9E-3 & 2.4E-1 & 5.5E-3 & 8.2E-1 & 1.4E-2 & 3.3E-1 & 1.0E-2 & 2.1E+0 & 8.7E-3 & 3.3E+0 & 4.0E-3 & 5.2E-1  \\
  & GAIN                                        & 4.9E-2 & 1.2E-1                   & 6.2E-2 & 6.9E-1                   & 5.3E-2 & 8.6E-2                   & 4.1E-2 & 9.3E-2                   & 5.5E-3 & 4.8E-2                   & 2.5E-2 & 4.7E-1                   & 5.0E-2 & 1.2E+0                   & 4.0E-2    & 1.0E-1    \\
  & MIRACLE                                     & 6.6E-2 & 9.5E-2                   & 1.9E-2 & 4.7E-1                   & 1.3E-2 & 1.3E-2                   & 4.0E-3 & 1.7E-2                   & 9.9E-2 & 3.5E+0                   & 6.9E-2 & 1.6E+0                   & 1.7E-2 & 1.7E-1                   & 7.5E-3    & 1.2E-2    \\
  & MIWAE                                       & 3.3E-2 & 6.4E-2                   & 8.3E-3 & 3.7E-2                   & 2.4E-2 & 3.5E-2                   & 3.0E-2 & 8.7E-2                   & 6.6E-3 & 7.2E-2                   & 2.3E-2 & 3.2E-1                   & 1.2E-2 & 1.5E-1                   & 9.1E-3    & 2.0E-2    \\
  & Sink                                        & 1.9E-2 & 6.2E-2                   & 1.4E-2 & 1.4E-1                   & 1.0E-2 & 6.3E-3                   & 1.3E-2 & 3.9E-2                   & 7.2E-3 & 5.1E-2                   & 1.8E-2 & 3.5E-1                   & 6.9E-3 & 1.6E-1                   & 4.6E-3    & 2.9E-2    \\
  & TDM                                         & 2.2E-2 & 6.8E-2                   & 1.4E-2 & 1.2E-1                   & 9.4E-3 & 8.3E-3                   & 1.5E-2 & 3.8E-2                   & 2.0E-2 & 7.7E-2                   & 1.8E-2 & 3.7E-1                   & 3.9E-3 & 1.8E-1                   & 7.1E-3    & 2.1E-2    \\
  & \textbf{KnewImp}                                    & 2.5E-2 & 1.0E-1                   & 3.9E-3 & 1.3E-1                   & 1.9E-2 & 2.9E-2                   & 8.4E-3 & 1.2E-2                   & 9.0E-3 & 1.3E-1                   & 8.5E-3 & 5.0E-2                   & 7.1E-3 & 6.8E-1                   & 5.8E-3    & 1.6E-2    \\ \bottomrule
\end{tabular}
  }
\end{table}

\begin{table}[htbp]
\caption{Ablation Study Results with missing rate at 30\%, and `$^*$' marks that KnewImp outperforms significantly at $p$-value < 0.05 over paired samples $t$-test. Best results are \textbf{bolded}.}
  \label{appendix-tab:AblationStudy}
    \resizebox{1.0\linewidth}{!}{
  \begin{tabular}{c|l|l|ll|ll|ll|ll|ll|ll|ll|ll}
  \toprule
  \multirow{2}{*}{Missing} & \multicolumn{1}{c|}{\multirow{2}{*}{NER}} & \multicolumn{1}{c|}{\multirow{2}{*}{Joint}} & \multicolumn{2}{c|}{BT}              & \multicolumn{2}{c|}{BCD}             & \multicolumn{2}{c|}{CC}              & \multicolumn{2}{c|}{CBV}             & \multicolumn{2}{c|}{IS}              & \multicolumn{2}{c|}{PK}              & \multicolumn{2}{c|}{QB} & \multicolumn{2}{c}{WQW} \\ \cmidrule{4-19} 
                           & \multicolumn{1}{c|}{}                     & \multicolumn{1}{c|}{}                       & MAE      & \multicolumn{1}{l|}{Wass} & MAE      & \multicolumn{1}{l|}{Wass} & MAE      & \multicolumn{1}{l|}{Wass} & MAE      & \multicolumn{1}{l|}{Wass} & MAE      & \multicolumn{1}{l|}{Wass} & MAE      & \multicolumn{1}{l|}{Wass} & MAE        & Wass       & MAE        & Wass       \\ \midrule
  \multirow{4}{*}{MAR}     & \XSolidBrush                                     & \XSolidBrush                                        & 0.96$^*$ & 3.82$^*$                  & 1.05$^*$ & 20.2$^*$                 & 1.04$^*$ & 5.47$^*$                  & 0.86$^*$ & 5.81$^*$                  & 0.67$^*$ &20.2$^*$                 & 1.06$^*$ & 15.6$^*$                 & 0.72$^*$   & 22.5$^*$  & 0.79$^*$   & 6.49$^*$   \\
                           & \XSolidBrush                                      & \Checkmark                                       & 0.54     & 0.42                      & 0.34     & 0.82                      & 0.61$^*$ & 0.40$^*$                   & 0.58$^*$ & 0.47$^*$                  & 0.43$^*$ & 1.34                      & 0.46$^*$ & 1.25$^*$                  & 0.47$^*$   & 3.56$^*$   & 0.55$^*$   & 0.64$^*$   \\
                           & \Checkmark                                      & \XSolidBrush                                       & 0.96$^*$ & 3.83$^*$                  & 1.05$^*$ & 20.3$^*$                 & 1.04$^*$ & 5.49$^*$                  & 0.86$^*$ & 5.83$^*$                  & 0.67$^*$ & 20.2$^*$                 & 1.06$^*$ & 15.7$^*$                 & 0.72$^*$   & 22.5$^*$  & 0.79$^*$   & 6.51$^*$   \\
                           & \Checkmark                                      & \Checkmark                                        & \textbf{0.52}     & \textbf{0.38}                      & \textbf{0.34}     & \textbf{0.82}                      & \textbf{0.35}     & \textbf{0.25}                      & \textbf{0.31}     & \textbf{0.20}                      & \textbf{0.39}     & \textbf{1.31}                      & \textbf{0.44}     & \textbf{1.21}                      & \textbf{0.45}       & \textbf{3.50}       & \textbf{0.46}       & \textbf{0.55}       \\ \midrule
  \multirow{4}{*}{MCAR}    & \XSolidBrush                                     & \XSolidBrush                                        & 0.72$^*$ & 2.11$^*$                  & 0.74$^*$ & 16.7$^*$                 & 0.85$^*$ & 3.72$^*$                  & 0.83$^*$ & 5.22$^*$                  & 0.74$^*$ & 18.4$^*$                 & 0.71$^*$ & 12.7$^*$                 & 0.58$^*$   & 20.1$^*$  & 0.76$^*$   & 5.57$^*$   \\
                           & \XSolidBrush                                      & \Checkmark                                       & 0.52$^*$ & 0.17$^*$                  & 0.25     & 0.79                      & 0.62$^*$ & 0.46$^*$                  & 0.61$^*$ & 0.71$^*$                  & 0.46     & 3.05                      & 0.34     & 1.09                      & 0.36$^*$   & 3.74$^*$   & 0.58$^*$   & 0.82$^*$   \\
                           & \Checkmark                                      & \XSolidBrush                                       & 0.72$^*$ & 2.12$^*$                  & 0.73$^*$ & 16.8$^*$                 & 0.86$^*$ & 3.73$^*$                  & 0.83$^*$ & 5.24$^*$                  & 0.74$^*$ & 18.4$^*$                 & 0.71$^*$ & 12.8$^*$                 & 0.58$^*$   & 20.1$^*$  & 0.76$^*$   & 5.60$^*$    \\
                           & \Checkmark                                      & \Checkmark                                        & \textbf{0.48}     & \textbf{0.18}                      & \textbf{0.25}     & \textbf{0.80}                      & \textbf{0.47}     & \textbf{0.34}                      & \textbf{0.42}     & \textbf{0.44}                      & \textbf{0.44}     & \textbf{3.05}                      & \textbf{0.32}     & \textbf{1.01}                      & \textbf{0.34}       & \textbf{3.66}       & \textbf{0.53}       & \textbf{0.76}       \\ \midrule
  \multirow{4}{*}{MNAR}    & \XSolidBrush                                     & \XSolidBrush                                       & 0.81$^*$   & 2.47$^*$  & 0.89$^*$   & 18.2$^*$  & 0.87$^*$   & 3.85$^*$  & 0.85$^*$   & 5.26$^*$   & 0.69$^*$  & 17.6$^*$  & 0.87$^*$  & 13.0$^*$  & 0.64$^*$  & 20.6$^*$  & 0.77$^*$   & 5.71$^*$   \\
                           & \XSolidBrush                                      & \Checkmark                                      & 0.62       & 0.37      & 0.32       & 1.47       & 0.61$^*$   & 0.47$^*$  & 0.64$^*$   & 0.79$^*$   & 0.44      & 2.79       & 0.43$^*$  & 1.88$^*$   & 0.44$^*$  & 5.65       & 0.60$^*$    & 0.87       \\
                           & \Checkmark                                      & \XSolidBrush                                      & 0.82$^*$   & 2.57$^*$  & 0.89$^*$   & 18.3$^*$  & 0.87$^*$   & 3.86$^*$  & 0.85$^*$   & 5.28$^*$   & 0.69$^*$  & 17.7$^*$  & 0.88$^*$  & 13.5$^*$  & 0.64$^*$  & 20.7$^*$  & 0.77$^*$   & 5.73$^*$   \\
                           & \Checkmark                                      & \Checkmark                                       & \textbf{0.60}       & \textbf{0.35}      & \textbf{0.32}       & \textbf{1.46}       & \textbf{0.44}       & \textbf{0.34}      & \textbf{0.46}       & \textbf{0.52}       & \textbf{0.40}       & \textbf{2.68}       & \textbf{0.39}      & \textbf{1.56}       & \textbf{0.42}      & \textbf{5.57}       & \textbf{0.55}       & \textbf{0.81}       \\ \bottomrule
  \end{tabular}
    }
  \end{table}

\begin{table}[htbp]
  \caption{Standard deviation of Ablation Study Results with missing rate at 30\%.}
\label{appendix-tab:stdAblationStudy}
  \resizebox{1.0\linewidth}{!}{
\begin{tabular}{l|l|l|ll|ll|ll|ll|ll|ll|ll|ll}
\toprule
\multicolumn{1}{c|}{\multirow{2}{*}{Missing}} & \multicolumn{1}{c|}{\multirow{2}{*}{NER}} & \multicolumn{1}{c|}{\multirow{2}{*}{Joint}} & \multicolumn{2}{c|}{BT}             & \multicolumn{2}{c|}{BCD}            & \multicolumn{2}{c|}{CC}             & \multicolumn{2}{c|}{CBV}            & \multicolumn{2}{c|}{IS}             & \multicolumn{2}{c|}{PK}             & \multicolumn{2}{c|}{QB}             & \multicolumn{2}{c}{WQW} \\ \cmidrule{4-19} 
\multicolumn{1}{c|}{}                         & \multicolumn{1}{c|}{}                     & \multicolumn{1}{c|}{}                       & MAE     & \multicolumn{1}{l|}{Wass} & MAE     & \multicolumn{1}{l|}{Wass} & MAE     & \multicolumn{1}{l|}{Wass} & MAE     & \multicolumn{1}{l|}{Wass} & MAE     & \multicolumn{1}{l|}{Wass} & MAE     & \multicolumn{1}{l|}{Wass} & MAE     & \multicolumn{1}{l|}{Wass} & MAE        & Wass       \\ \midrule
MAR                                           & \XSolidBrush                                     & \XSolidBrush                                       & 6.1E-2 & 4.1E-1                   & 4.8E-2 & 6.8E-1                   & 1.1E-1 & 5.6E-1                   & 4.6E-2 & 4.5E-1                   & 4.4E-2 & 3.8E+0                   & 1.4E-1 & 3.9E+0                   & 7.2E-2 & 2.4E+0                   & 5.0E-2    & 4.1E-1    \\
MAR                                           & \XSolidBrush                                     & \Checkmark                                        & 6.9E-2 & 1.2E-1                   & 2.7E-2 & 1.1E-1                   & 2.8E-2 & 4.7E-2                   & 2.5E-2 & 3.8E-2                   & 2.7E-2 & 1.3E-1                   & 1.0E-2 & 9.0E-2                   & 1.8E-2 & 2.6E-1                   & 2.8E-2    & 5.6E-2    \\
MAR                                           & \Checkmark                                      & \XSolidBrush                                       & 6.1E-2 & 4.1E-1                   & 4.8E-2 & 6.8E-1                   & 1.1E-1 & 5.6E-1                   & 4.6E-2 & 4.5E-1                   & 4.4E-2 & 3.8E+0                   & 1.4E-1 & 3.9E+0                   & 7.2E-2 & 2.4E+0                   & 5.0E-2    & 4.1E-1    \\
MAR                                           & \Checkmark                                      & \Checkmark                                        & 2.0E-2 & 4.0E-2                   & 2.7E-2 & 1.1E-1                   & 5.6E-2 & 6.4E-2                   & 1.6E-2 & 2.2E-2                   & 1.9E-2 & 1.1E-1                   & 1.1E-2 & 8.8E-2                   & 1.9E-2 & 2.7E-1                   & 1.6E-2    & 3.3E-2    \\ \midrule
MCAR                                          & \XSolidBrush                                     & \XSolidBrush                                       & 9.9E-3 & 5.8E-2                   & 1.0E-2 & 2.3E-1                   & 3.6E-3 & 6.0E-2                   & 3.3E-3 & 2.0E-2                   & 9.5E-3 & 4.2E-1                   & 1.5E-2 & 2.3E-1                   & 1.1E-2 & 1.0E+0                   & 4.0E-3    & 1.4E-2    \\
MCAR                                          & \XSolidBrush                                     & \Checkmark                                        & 5.0E-3 & 5.9E-3                   & 2.8E-3 & 4.6E-2                   & 1.2E-2 & 1.5E-2                   & 9.4E-3 & 1.7E-2                   & 6.4E-3 & 1.2E-1                   & 1.0E-2 & 1.2E-1                   & 7.3E-3 & 1.9E-1                   & 9.4E-04    & 3.9E-3    \\
MCAR                                          & \Checkmark                                      & \XSolidBrush                                       & 9.9E-3 & 5.7E-2                   & 1.0E-2 & 2.3E-1                   & 3.6E-3 & 6.0E-2                   & 3.2E-3 & 2.1E-2                   & 9.5E-3 & 4.2E-1                   & 1.5E-2 & 2.3E-1                   & 1.0E-2 & 1.0E+0                   & 4.0E-3    & 1.4E-2    \\
MCAR                                          & \Checkmark                                      & \Checkmark                                        & 3.3E-3 & 3.7E-3                   & 1.9E-3 & 4.6E-2                   & 1.1E-2 & 1.8E-2                   & 4.1E-3 & 1.8E-2                   & 5.7E-3 & 1.1E-1                   & 6.4E-3 & 3.7E-2                   & 4.8E-3 & 1.7E-1                   & 2.2E-3    & 1.1E-2    \\ \midrule
MNAR                                          & \XSolidBrush                                     & \XSolidBrush                                       & 4.2E-2 & 1.5E-1                   & 2.3E-2 & 8.5E-1                   & 3.2E-2 & 1.8E-1                   & 1.2E-2 & 5.3E-2                   & 6.9E-3 & 1.2E-1                   & 3.2E-2 & 9.6E-1                   & 1.6E-2 & 8.3E-1                   & 1.4E-2    & 1.0E-1    \\
MNAR                                          & \XSolidBrush                                     & \Checkmark                                        & 4.0E-2 & 1.4E-1                   & 3.4E-3 & 1.3E-1                   & 1.8E-2 & 2.1E-2                   & 4.8E-3 & 1.7E-2                   & 1.0E-2 & 1.3E-1                   & 1.1E-2 & 1.8E-1                   & 8.0E-3 & 7.0E-1                   & 7.7E-3    & 1.4E-2    \\
MNAR                                          & \Checkmark                                      & \XSolidBrush                                       & 4.8E-2 & 1.4E-1                   & 2.4E-2 & 8.5E-1                   & 3.3E-2 & 1.8E-1                   & 1.2E-2 & 5.3E-2                   & 6.9E-3 & 1.2E-1                   & 1.9E-2 & 2.5E-1                   & 1.6E-2 & 8.3E-1                   & 1.4E-2    & 1.0E-1    \\
MNAR                                          & \Checkmark                                      & \Checkmark                                        & 2.5E-2 & 1.0E-1                   & 3.9E-3 & 1.3E-1                   & 1.9E-2 & 2.9E-2                   & 8.4E-3 & 1.2E-2                   & 9.0E-3 & 1.3E-1                   & 8.5E-3 & 5.0E-2                   & 7.1E-3 & 6.8E-1                   & 5.8E-3    & 1.6E-2    \\ 
\bottomrule
\end{tabular}
}
\end{table}

\begin{figure}[htbp]
  \centering
  \subfigure[MAR with 30\% missing rate at CC dataset.]{\includegraphics[width=0.48\linewidth]{picture/sens_fig/sens_data_concrete_compression_p_0.3_scenario_MAR.pdf}}
  \subfigure[MCAR with 30\% missing rate at CC dataset.]{\includegraphics[width=0.48\linewidth]{picture/sens_fig/sens_data_concrete_compression_p_0.3_scenario_MCAR.pdf}}
  \subfigure[MNAR with 30\% missing rate at CC dataset.]{\includegraphics[width=0.48\linewidth]{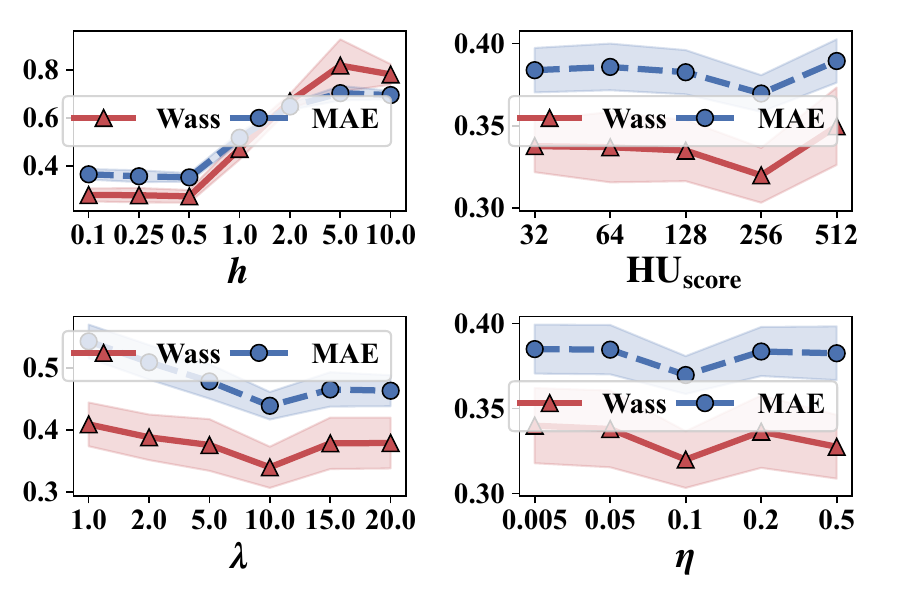}}
  \caption{Parameter sensitivity of KnewImp on bandwidth for kernel function ($h$), hidden unit of score network $\text{HU}_{\text{score}}$, NER weight $\lambda$, and discretization step $\eta$ for~\Cref{eq:jointVelocityField} on CC dataset. Mean values and one standard deviations from mean are represented by scatters and shaded area, respectively.
  }
  \label{appendix-fig:sensResult}
\end{figure}

\newpage
\subsection{Time Complexity Analysis}\label{appendix-subsec:timeComplexity}
\begin{figure}[htbp]
  \centering
  \subfigure[Running time vary $\mathrm{N}$, MAR.]{\includegraphics[width=0.30\linewidth]{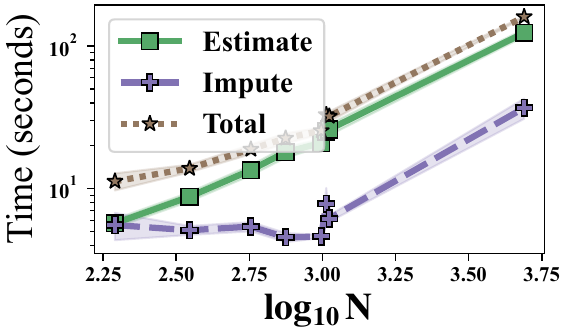}}
  \subfigure[Running time vary $\mathrm{N}$, MCAR.]{\includegraphics[width=0.30\linewidth]{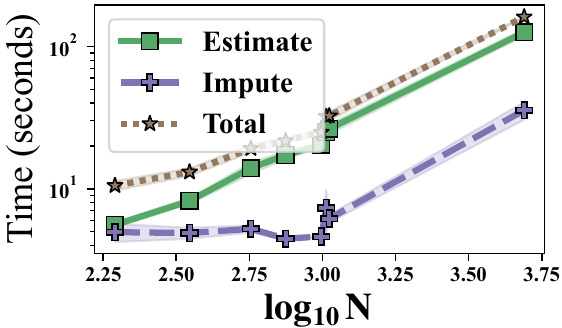}}
  \subfigure[Running time vary $\mathrm{N}$, MNAR.]{\includegraphics[width=0.30\linewidth]{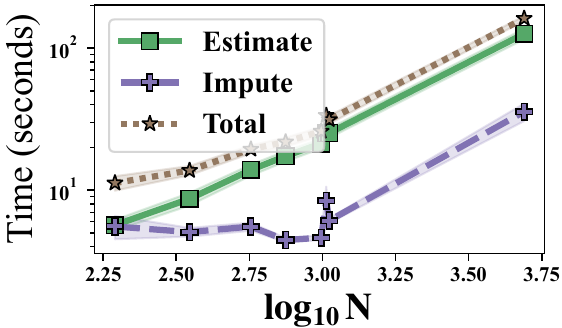}}
  \subfigure[Running time vary $\mathrm{D}$, MAR.]{\includegraphics[width=0.30\linewidth]{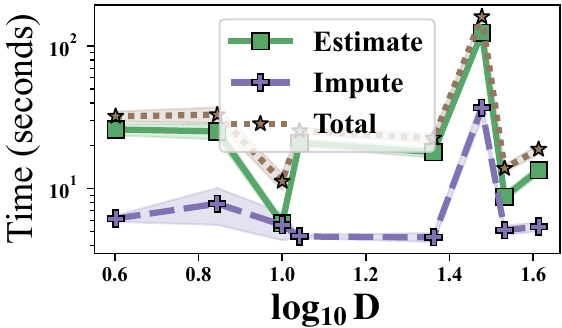}}
  \subfigure[Running time vary $\mathrm{D}$, MCAR.]{\includegraphics[width=0.30\linewidth]{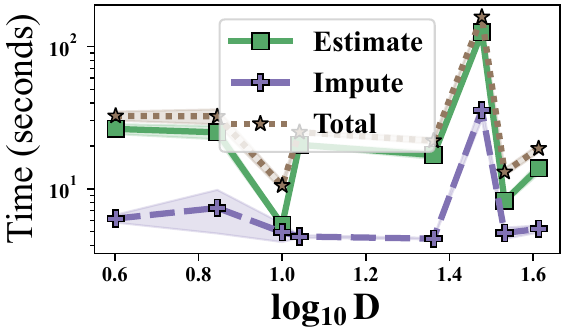}}
  \subfigure[Running time vary $\mathrm{D}$, MNAR.]{\includegraphics[width=0.30\linewidth]{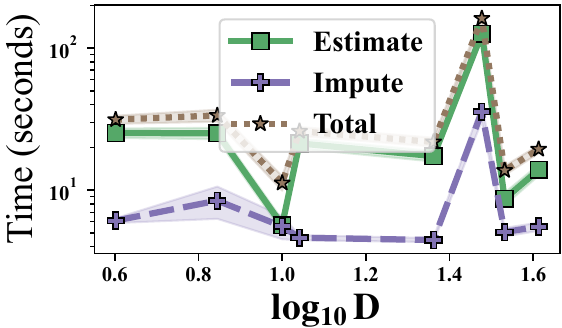}}
  \caption{Average computation time. The scatters and shaded areas indicate the mean and one standard deviation from the mean, respectively.
  }
  \label{appendix-fig:timeResult}
\end{figure}
In this subsection, we present an analysis of the complexity of time for our KnewImp approach. The complexity analysis is based on the algorithms described in~\Cref{algo:impWGFAlgo,algo:impScoreEstimate,algo:knewImpAlgorithm}. We begin by estimating the time complexity of the score function $\nabla_{\boldsymbol{X}^{\text{(joint)}}}\log{\hat{p}}(\boldsymbol{X}^{\text{(joint)}})$. Assuming the number of layers in the score network is $\mathrm{L}$ and each layer has an equal number of hidden units denoted as $\text{HU}_{\text{score}}$, the time complexity for the imputation algorithm defined in~\Cref{algo:impWGFAlgo,algo:knewImpAlgorithm} is detailed as follows:

\begin{enumerate}[leftmargin=*]
  \item{\textbf{Impute Part:}
    \begin{itemize}[leftmargin=*]
      \item{\textbf{Score function computation:}
        The time complexity for computing the score function is expressed as:
        \begin{equation}
          \mathcal{O}\left[2 \times \mathrm{N} \times \left( \mathrm{D} \times \text{HU}_{\text{score}} + (\mathrm{L}-1) \times \text{HU}_{\text{score}}^2 \right)\right],
        \end{equation}
        where the factor $2$ accounts for the backward propagation needed during the score function computation.
      }
      \item{\textbf{Kernel function and its gradient:}
        Employing the RBF kernel $\mathcal{K}(\boldsymbol{X}, \tilde{\boldsymbol{X}})\coloneqq \exp\left(-\frac{\|\boldsymbol{X} - \tilde{\boldsymbol{X}}\|^2}{2h^2}\right)$, the gradient with respect to $\tilde{\boldsymbol{X}}$ is analytically determined as:
        \begin{equation}
          [\nabla_{\tilde{\boldsymbol{X}}} \mathcal{K}(\boldsymbol{X}  , \tilde{\boldsymbol{X}})]{[:,j]} = - \frac{1}{h^2} \left\{[\mathcal{K}(\boldsymbol{X}  , \tilde{\boldsymbol{X}}) \times \tilde{\boldsymbol{X}} ]{[:, j]} + \tilde{\boldsymbol{X}}{[:,j]} \odot \sum_{j=1}^{\mathrm{D}}{\mathcal{K}(\boldsymbol{X}  , \tilde{\boldsymbol{X}}){[:,j]}} \right\}.
         \end{equation}
        The time complexities for calculating the kernel function and its gradient are specified in~\Cref{eq:appendix-timeComplexKernel1,eq:appendix-timeComplexGradKernel2}:
        \begin{equation}\label{eq:appendix-timeComplexKernel1}
          \mathcal{O}\left[\mathrm{N}^2 \times \mathrm{D} + \mathrm{N}^2\right],
        \end{equation}
        \begin{equation}\label{eq:appendix-timeComplexGradKernel2}
          \mathcal{O}\left[\mathrm{N}^2 \times \mathrm{D} + \mathrm{N}^2 + \mathrm{N} \times \mathrm{D}\right].
        \end{equation}
      }
    \end{itemize}
  }
  \item{\textbf{Estimate Part:}
    Building on the previous item, the time complexity for the estimation algorithm defined in~\Cref{algo:impScoreEstimate} is given as:
  \begin{equation}\label{eq:appendix-timeComplEstimate}
    \mathcal{O}\left[4\times \mathrm{N} \times \left( \mathrm{D} \times \text{HU}_{\text{score}} + (\mathrm{L}-1) \times \text{HU}_{\text{score}}^2  \right)\right],
  \end{equation}
   where the factor of $4$ comprises three distinct components: backward propagation ($1$), forward propagation ($1$), and the acquisition of the sample-wise score function ($2$). Note that the network parameter size is substantially smaller than the number of data points, thereby making the forward computation of the score function the primary factor in time complexity.
  }
\end{enumerate}

Based on the analysis outlined above, we explore how computational complexity varies with different dataset sizes $\mathrm{N}$ and the number of features $\mathrm{D}$, as shown in Figs.~\ref{appendix-fig:timeResult} (a) and (b), respectively. From these figures, it is evident that computational time increases with the dataset size $\mathrm{N}$. However, changes in the number of features $\mathrm{D}$ do not significantly affect the computation time. This observation underscores that the primary determinant of computational complexity in our context is the dataset size, aligning with our theoretical analysis, which indicates a quadratic relationship between time complexity and the size of the dataset $\mathrm{N}$ for the `Impute' part, and $\mathrm{N}\gg\mathrm{D}$ for the `Estimate' part. 

Moreover, the data reveals that the total computational time is predominantly governed by `Estimate' part of our KnewImp approach. This suggests that the training of the score function represents a critical bottleneck in the efficiency of the KnewImp algorithm. Therefore, accelerating the KnewImp algorithm crucially hinges on reducing the computational demands of the `Estimate' part.



\subsection{Convergence Analysis}
\begin{figure}[htbp]
  \centering
\subfigure[MAE, BT.]{\includegraphics[width=0.24\linewidth]{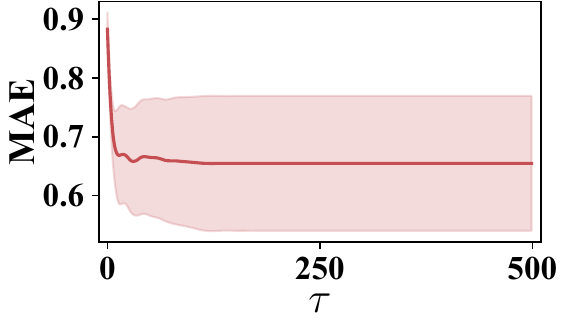}}
\subfigure[Wass, BT.]{\includegraphics[width=0.24\linewidth]{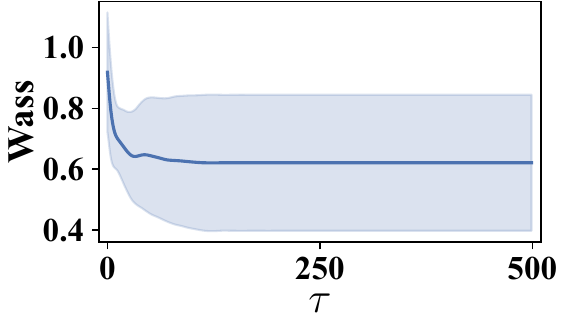}}
\subfigure[MAE, BCD.]{\includegraphics[width=0.24\linewidth]{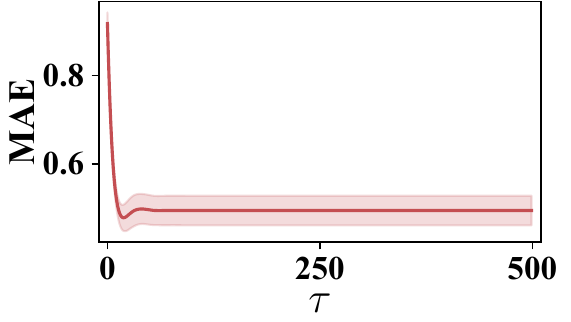}}
\subfigure[MAE, BCD.]{\includegraphics[width=0.24\linewidth]{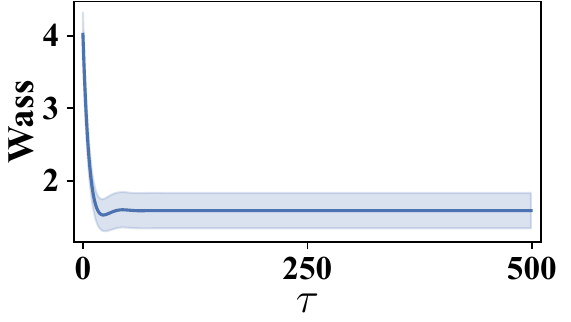}}
\subfigure[MAE, CC.]{\includegraphics[width=0.24\linewidth]{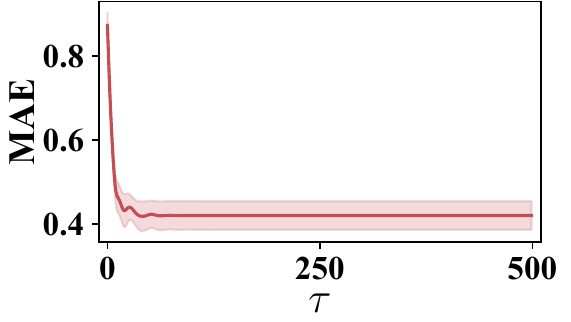}}
\subfigure[Wass, CC.]{\includegraphics[width=0.24\linewidth]{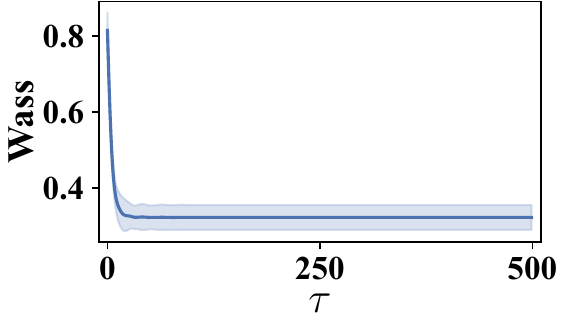}}
\subfigure[MAE, CBV.]{\includegraphics[width=0.24\linewidth]{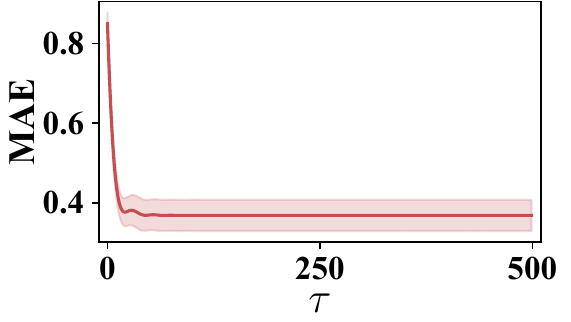}}
\subfigure[MAE, CBV.]{\includegraphics[width=0.24\linewidth]{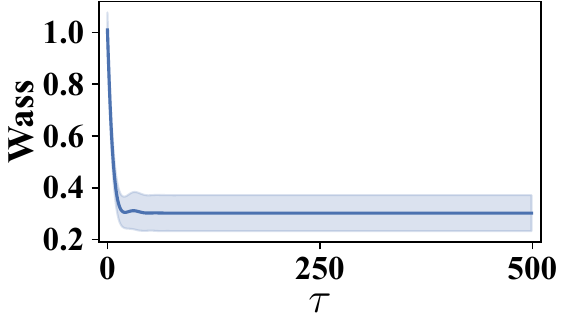}}
\subfigure[MAE, IS.]{\includegraphics[width=0.24\linewidth]{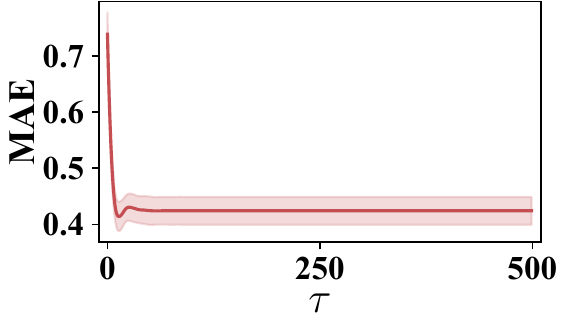}}
\subfigure[Wass, IS.]{\includegraphics[width=0.24\linewidth]{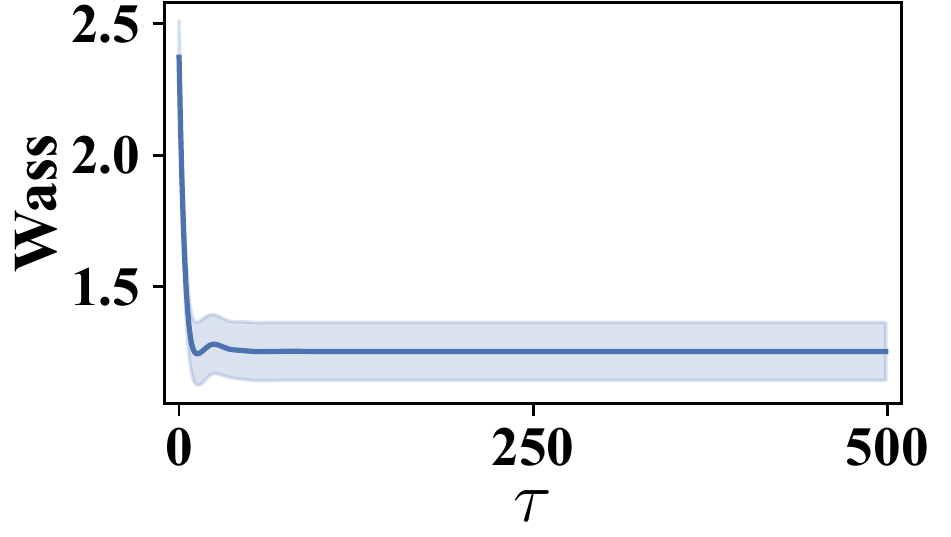}}
\subfigure[MAE, PK.]{\includegraphics[width=0.24\linewidth]{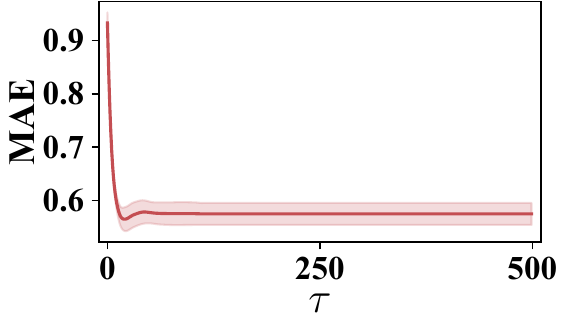}}
\subfigure[MAE, PK.]{\includegraphics[width=0.24\linewidth]{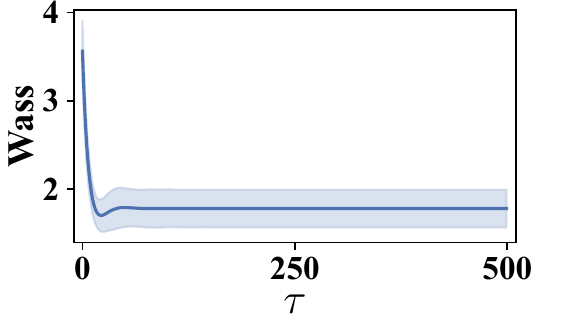}}
\subfigure[MAE, QB.]{\includegraphics[width=0.24\linewidth]{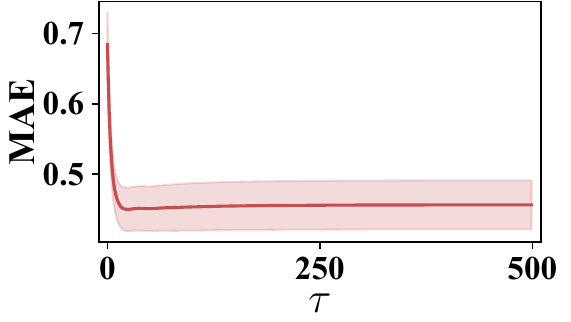}}
\subfigure[Wass, QB.]{\includegraphics[width=0.24\linewidth]{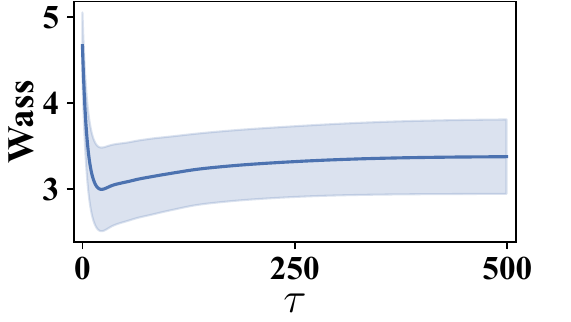}}
\subfigure[MAE, WQW.]{\includegraphics[width=0.24\linewidth]{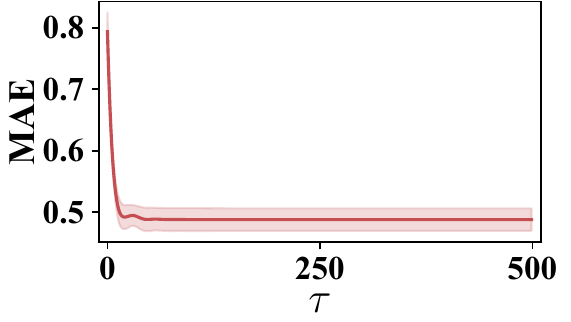}}
\subfigure[MAE, WQW.]{\includegraphics[width=0.24\linewidth]{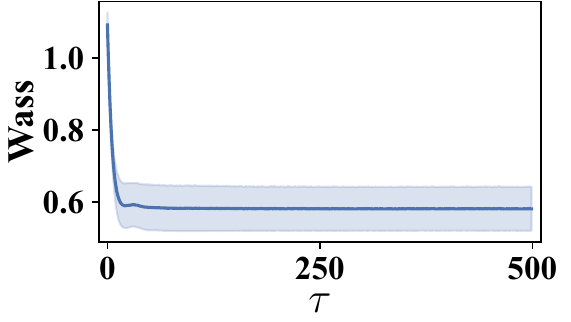}}
  \caption{Evolution of evaluation metrics along iteration time $\tau$ under MAR scenario at 30\% missing rate. The shaded area indicates the $\pm$ 1.0 standard deviation uncertainty interval.}
  \label{appendix-fig:convExperMAR}
\end{figure}

In this subsection, we explore the convergence of the Impute part as defined in~\Cref{algo:impWGFAlgo} within our KnewImp approach. Prior to delving into this discussion, it is essential to establish a clear definition of convergence:
\begin{definition}\label{appendix-def:Convergence}
A sequence $\{\mathcal{F}_1, \mathcal{F}_2, ..., \mathcal{F}_\mathrm{T}\}$ is said to be convergent if there exists a real number $\mathcal{G}$ such that for any given positive number $\varepsilon$ ($\varepsilon > 0$), there exists a positive integer $N$, such that for all indices $n$ greater than $N$, the corresponding terms $ \mathcal{F}_n, n\ge N$ satisfy the  inequality $\left|  \mathcal{F}_n - \mathcal{G} \right| < \varepsilon$.
\end{definition}
Based on Definition~\ref{appendix-def:Convergence}, if a sequence is either monotonically increasing or monotonically decreasing and bounded (either bounded above or bounded below), then it is guaranteed to converge according to the celebrated monotone convergence theorem (Section 3.14 in reference~\cite{rudin1964principles}). Based on this, we first prove the following proposition for the convergence in `Impute' part of our approach:
\begin{proposition}
  The convergence of the `Imputer' part can be guaranteed, given that the step size $\eta$ is small enough. 
\end{proposition}
\begin{proof}
First, let's reformulate the velocity field as follows:
\begin{equation}\label{appendix-eq:jointVelocityField}
\begin{aligned}
 & {u}( \boldsymbol{X}^{\text{(joint)}}, \tau) \\
 = &  \mathbb{E}_{r({\tilde{\boldsymbol{X}}}^{\text{(joint)}},\tau)}\left\{ 
    \begin{aligned}  
    & { - \lambda \nabla_{  {\tilde{\boldsymbol{X}}^{\text{(miss)}}}} \mathcal{K}(  \boldsymbol{X}^{\text{(joint)}}  ,    {\tilde{\boldsymbol{X}}^{\text{(joint)}}})  }\\
    &\quad \quad+ [\nabla_{  {\tilde{\boldsymbol{X}}^{\text{(miss)}}}}\log{\hat{p}( \tilde{\boldsymbol{X}}^{\text{(joint)}})}]^\top\mathcal{K}(  \boldsymbol{X}^{\text{(joint)}}  ,   {\tilde{\boldsymbol{X}}}^{\text{(joint)}}) 
    \end{aligned} 
    \right\}\\
    \overset{\text{(i)}}{=} &   \mathbb{E}_{r(\tilde{\boldsymbol{X}}^{\text{(joint)}},\tau)}\left\{ 
      \begin{aligned}  
      & {  \lambda [\nabla_{  {\tilde{\boldsymbol{X}}^{\text{(miss)}}}}\log{r({\tilde{\boldsymbol{X}}}^{\text{(joint)}})}]^\top \mathcal{K}(  \boldsymbol{X}^{\text{(joint)}}  ,    {\tilde{\boldsymbol{X}}^{\text{(joint)}}})   }\\
      &\quad \quad+ [\nabla_{  {\tilde{\boldsymbol{X}}^{\text{(miss)}}}}\log{\hat{p}(  \tilde{\boldsymbol{X}}^{\text{(joint)}})}]^\top\mathcal{K}(  \boldsymbol{X}^{\text{(joint)}}  ,   \tilde{\boldsymbol{X}}^{\text{(joint)}}) 
      \end{aligned} \right\}  \\
      = &   \int r(\tilde{\boldsymbol{X}}^{\text{(joint)}})  \left\{ 
        \begin{aligned}  
        & {  \lambda\nabla_{  {\tilde{\boldsymbol{X}}^{\text{(miss)}}}}\log{r(\tilde{\boldsymbol{X}}^{\text{(joint)}})}  }\\
        &\quad \quad+ \nabla_{  {\tilde{\boldsymbol{X}}^{\text{(miss)}}}}\log{\hat{p}(  \tilde{\boldsymbol{X}}^{\text{(joint)}})}
        \end{aligned} \right\}^\top \mathcal{K}(  \boldsymbol{X}^{\text{(joint)}}  ,    {\tilde{\boldsymbol{X}}^{\text{(joint)}}}) \mathrm{d}\tilde{\boldsymbol{X}}^{\text{(joint)}} \\
        =&  \int   \left\{ 
          \begin{aligned}  
          & {  \lambda \nabla_{  {\tilde{\boldsymbol{X}}^{\text{(miss)}}}}\log{r({\tilde{\boldsymbol{X}}}^{\text{(joint)}})}  }\\
          &\quad \quad+ \nabla_{  {\tilde{\boldsymbol{X}}^{\text{(miss)}}}}\log{\hat{p}(  \tilde{\boldsymbol{X}}^{\text{(joint)}})}
          \end{aligned} \right\}^\top \mathcal{K}(  \boldsymbol{X}^{\text{(joint)}}  ,    {\tilde{\boldsymbol{X}}^{\text{(joint)}}}) \mathrm{d}r(\tilde{\boldsymbol{X}}^{\text{(joint)}}) 
        ,
  \end{aligned}
\end{equation}
where (i) is based on integration by parts. 

Based on this reformulation, the inner product can be given as follows:
\begin{equation}\label{appendix-eq:monotonicIncreasing}
  \begin{aligned}
    &\frac{\mathrm{d} \mathcal{F}_{\text{joint-NER}}}{\mathrm{d}\tau}\\
   = & \int{\left<\nabla_{\boldsymbol{X}^{\text{(miss)}}}\frac{\delta \mathcal{F}_{\text{joint-NER}}}{\delta r(\boldsymbol{X}^{\text{(joint)}})},{u}( \boldsymbol{X}^{\text{(joint)}}, \tau)  \right>}\mathrm{d}r(\boldsymbol{X}^{\text{(miss)}}) \\
  =&  \iint{   
\begin{aligned}
   \left\{ 
          \begin{aligned}  
          & {  \lambda \nabla_{  {{\tilde{\boldsymbol{X}}}^{\text{(miss)}}}}\log{r({\tilde{\boldsymbol{X}}}^{\text{(joint)}})}  }\\
          &+ \nabla_{  {\tilde{\boldsymbol{X}}^{\text{(miss)}}}}\log{\hat{p}(  \tilde{\boldsymbol{X}}^{\text{(joint)}})}
          \end{aligned} \right\}^\top &  \mathcal{K}(  \boldsymbol{X}^{\text{(joint)}}  ,    {\tilde{\boldsymbol{X}}^{\text{(joint)}}}) \times \\      
& \indent \left\{ 
            \begin{aligned}  
            & {  \lambda \nabla_{  {{\boldsymbol{X}}^{\text{(miss)}}}}\log{r({\boldsymbol{X}}^{\text{(joint)}})}  }\\
            &+ \nabla_{  {\boldsymbol{X}^{\text{(miss)}}}}\log{\hat{p}(  \boldsymbol{X}^{\text{(joint)}})}
            \end{aligned} \right\} \mathrm{d}r(\tilde{\boldsymbol{X}}^{\text{(joint)}})  \mathrm{d}r(\boldsymbol{X}^{\text{(joint)}})
          \end{aligned}
            }\\
      \overset{\text{(i)}}{\ge} &  0,
  \end{aligned}
\end{equation}
where the (i) is predicated on the requirement that the kernel function, $\mathcal{K}(\cdot,\cdot)$, is semi-positive definite; according to the above-mentioned derivation, we can conclude that the evolution of $\mathcal{F}_{\text{joint-NER}}$ is monotonic increasing along $\tau$. 
Furthermore, $ \mathcal{F}_{\text{joint-NER}}$ satisfies the following inequality:
\begin{equation}\label{appendix-eq:ODEConverge}
  \begin{aligned}
  & \mathcal{F}_{\text{joint-NER}}\\
  \le & \mathcal{F}_{\text{joint-NER}} - (\lambda + 1)\mathbb{E}_{r(\boldsymbol{X}^{\text{(joint)}})}[\log{r(\boldsymbol{X}^{\text{(joint)}})}]\\
  = & -\mathbb{D}_{\text{KL}}\left[r(\boldsymbol{X}^{\text{(joint)}})\Vert \hat{p}(\boldsymbol{X}^{\text{(joint)}})\right]\\
  \le & 0,
\end{aligned}
\end{equation}
which indicates that $\mathcal{F}_{\text{joint-NER}}$ is upper-bounded by $0$.

According to~\Cref{appendix-eq:monotonicIncreasing,appendix-eq:ODEConverge}, the cost functional $\mathcal{F}_{\text{joint-NER}}$, driven by the velocity field $u(\boldsymbol{X}^{\text{(joint)}}, \tau)$ along $\tau$, converges. 
Building on this, employing a smaller step size $\eta$ results in the iteration curve of $\mathcal{F}_{\text{joint-NER}}$ more closely approximating the ODE defined in~\Cref{appendix-eq:monotonicIncreasing}. Consequently, a smaller $\eta$ leads to a sequence where $\mathcal{F}_{\text{joint-NER}}$ monotonically increases, aligning with the theoretical expectations of the ODE behavior.
\end{proof}


Unfortunately, directly obtaining $\mathcal{F}_{\text{joint-NER}}$ is intractable. Nevertheless, we can still observe the changes in Wass and MAE across iteration time $\tau$ to demonstrate the convergence of the 'Impute' part. To this end, we present the convergence trends along $\tau$ in~\Cref{appendix-fig:convExperMAR,appendix-fig:convExperMCAR,appendix-fig:convExperMNAR}. These figures illustrate that both MAE and Wass generally decrease as the iteration epochs increase and eventually stabilize after $\tau = 250$. This observed behavior supports our theoretical findings regarding the convergence of the `Impute' part.

\begin{figure}[htbp]
  \vspace{-0.5cm}
  \centering
\subfigure[MAE, BT.]{\includegraphics[width=0.24\linewidth]{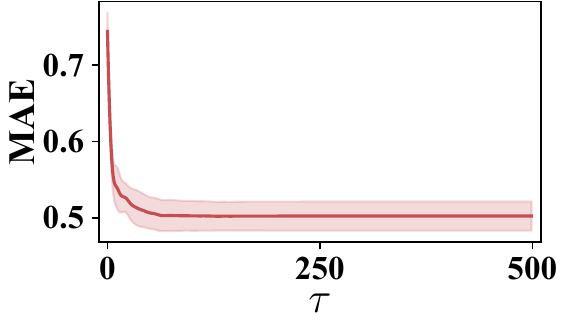}}
\subfigure[Wass, BT.]{\includegraphics[width=0.24\linewidth]{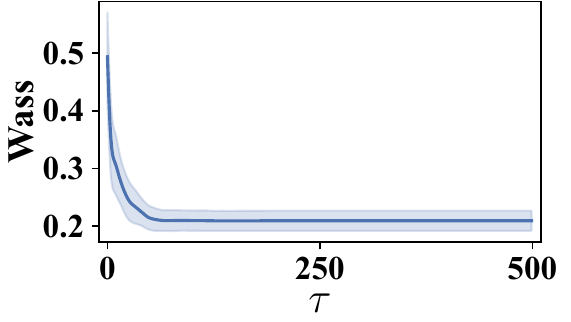}}
\subfigure[MAE, BCD.]{\includegraphics[width=0.24\linewidth]{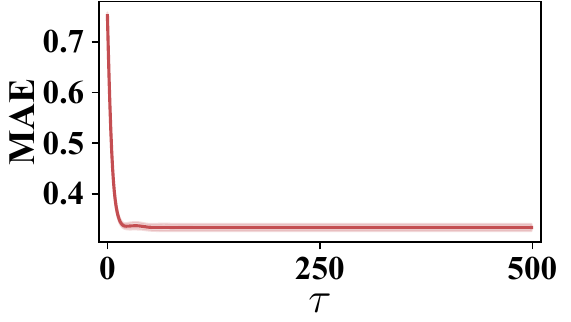}}
\subfigure[MAE, BCD.]{\includegraphics[width=0.24\linewidth]{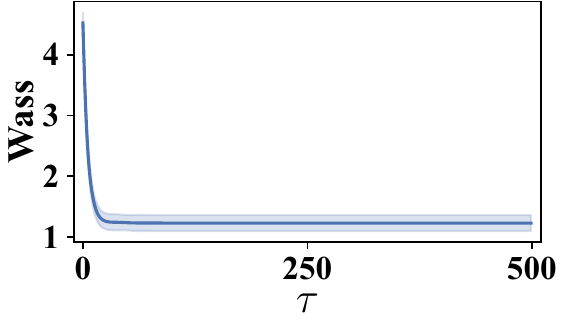}}
\subfigure[MAE, CC.]{\includegraphics[width=0.24\linewidth]{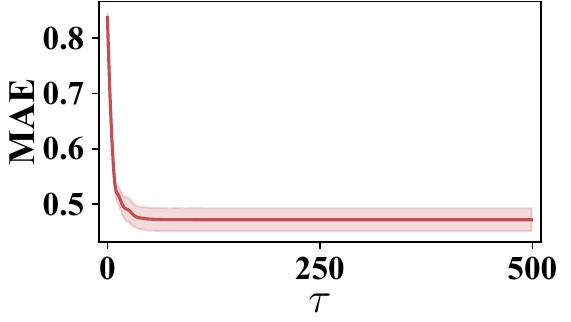}}
\subfigure[Wass, CC.]{\includegraphics[width=0.24\linewidth]{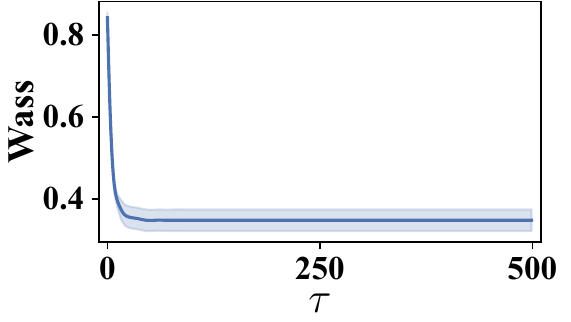}}
\subfigure[MAE, CBV.]{\includegraphics[width=0.24\linewidth]{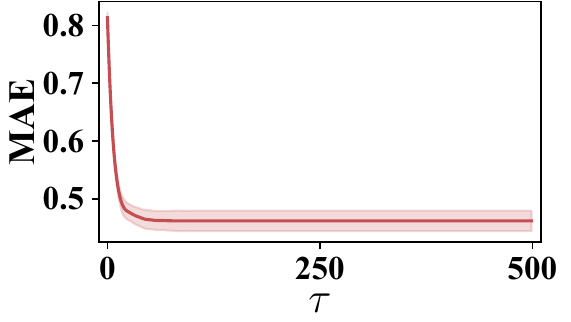}}
\subfigure[MAE, CBV.]{\includegraphics[width=0.24\linewidth]{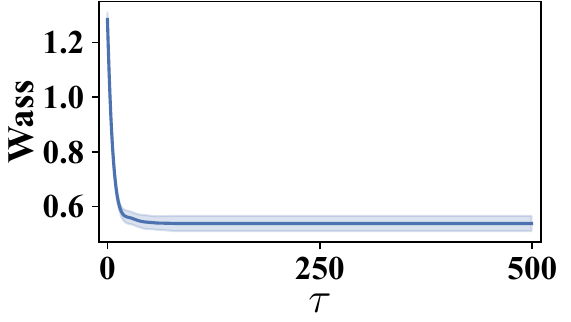}}
\subfigure[MAE, IS.]{\includegraphics[width=0.24\linewidth]{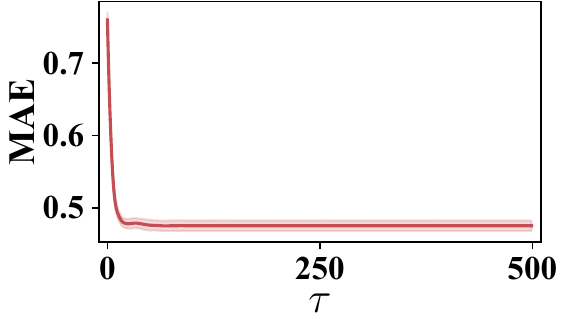}}
\subfigure[Wass, IS.]{\includegraphics[width=0.24\linewidth]{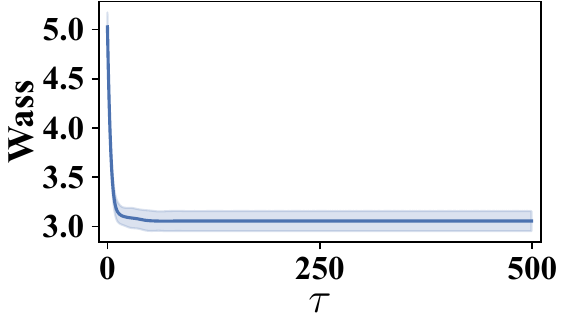}}
\subfigure[MAE, PK.]{\includegraphics[width=0.24\linewidth]{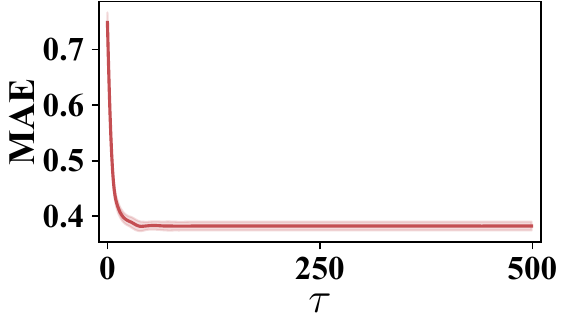}}
\subfigure[MAE, PK.]{\includegraphics[width=0.24\linewidth]{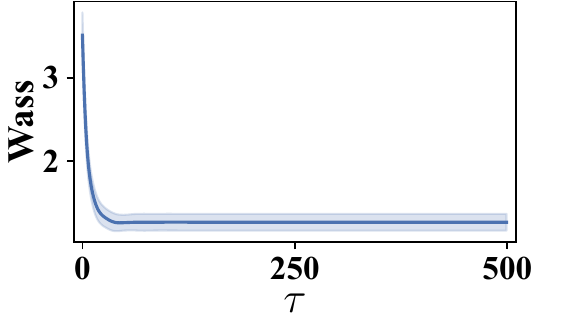}}
\subfigure[MAE, QB.]{\includegraphics[width=0.24\linewidth]{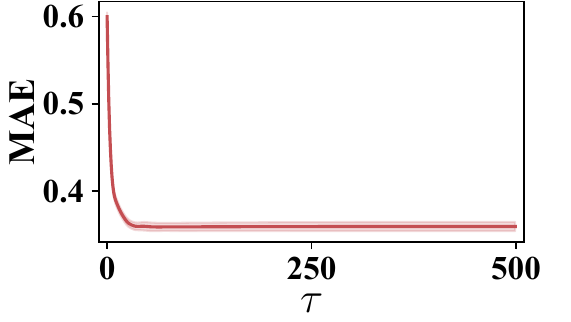}}
\subfigure[Wass, QB.]{\includegraphics[width=0.24\linewidth]{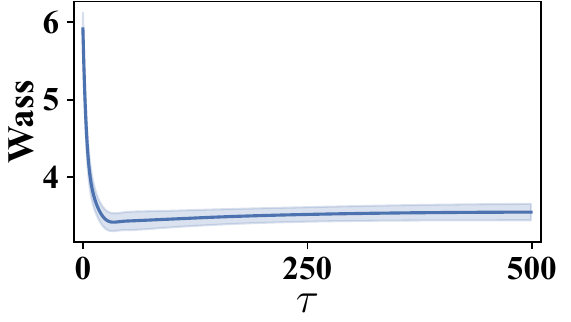}}
\subfigure[MAE, WQW.]{\includegraphics[width=0.24\linewidth]{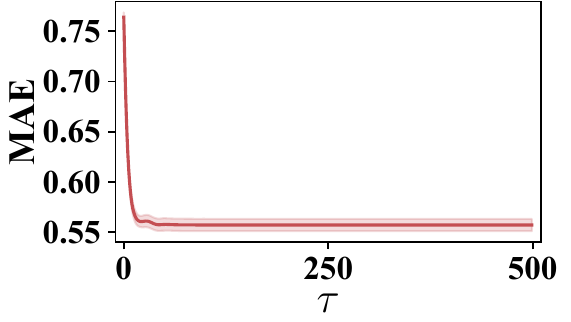}}
\subfigure[MAE, WQW.]{\includegraphics[width=0.24\linewidth]{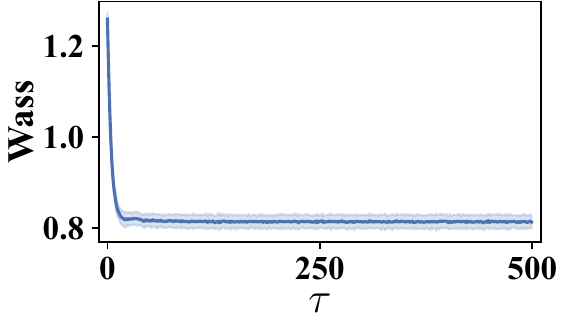}}
  \vspace{-0.3cm}
  \caption{Evolution of evaluation metrics along iteration time $\tau$ under MCAR scenario at 30\% missing rate. The shaded area indicates the $\pm$ 1.0 standard deviation uncertainty interval.}
  \label{appendix-fig:convExperMCAR}
\end{figure}
\begin{figure}[htbp]
  \vspace{-0.5cm}
  \centering
\subfigure[MAE, BT.]{\includegraphics[width=0.24\linewidth]{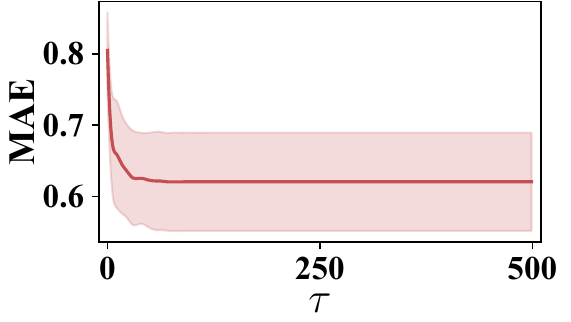}}
\subfigure[Wass, BT.]{\includegraphics[width=0.24\linewidth]{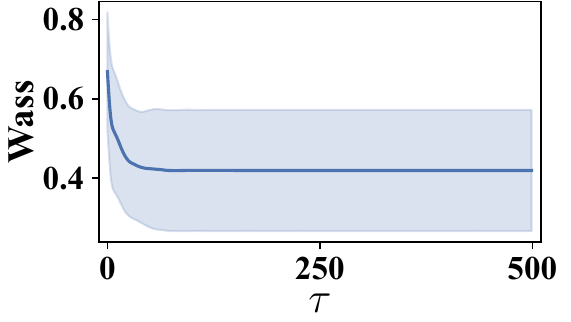}}
\subfigure[MAE, BCD.]{\includegraphics[width=0.24\linewidth]{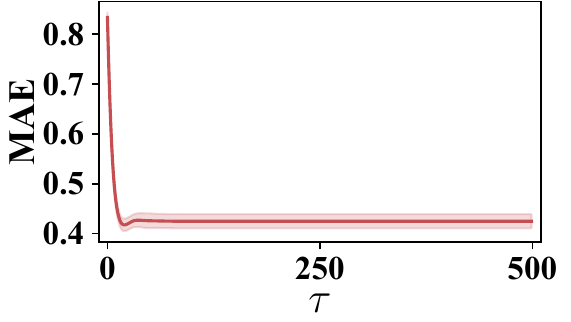}}
\subfigure[MAE, BCD.]{\includegraphics[width=0.24\linewidth]{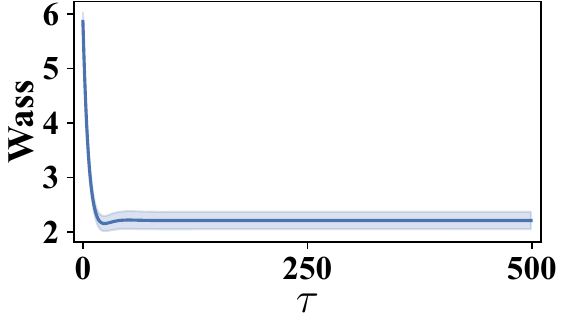}}
\subfigure[MAE, CC.]{\includegraphics[width=0.24\linewidth]{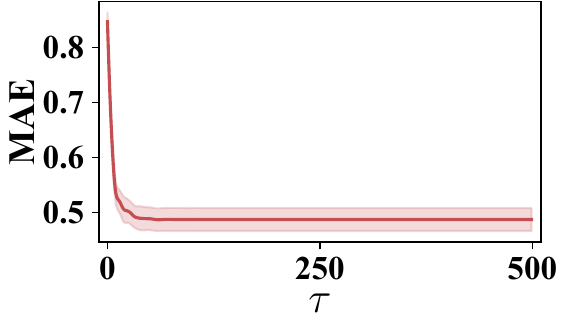}}
\subfigure[Wass, CC.]{\includegraphics[width=0.24\linewidth]{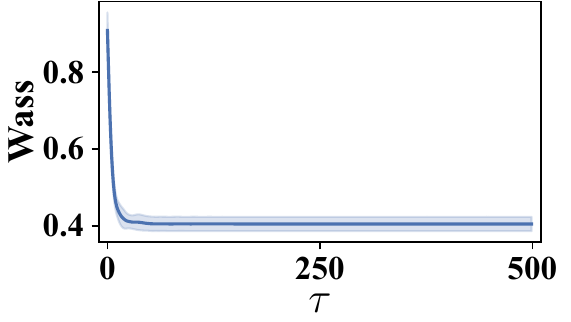}}
\subfigure[MAE, CBV.]{\includegraphics[width=0.24\linewidth]{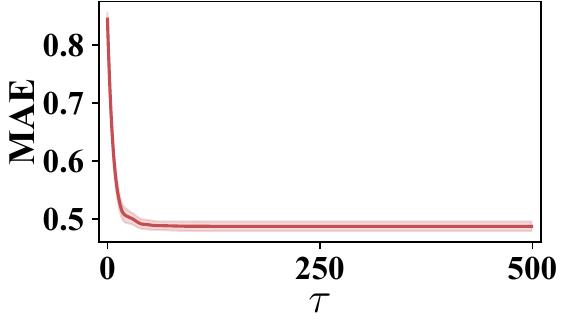}}
\subfigure[MAE, CBV.]{\includegraphics[width=0.24\linewidth]{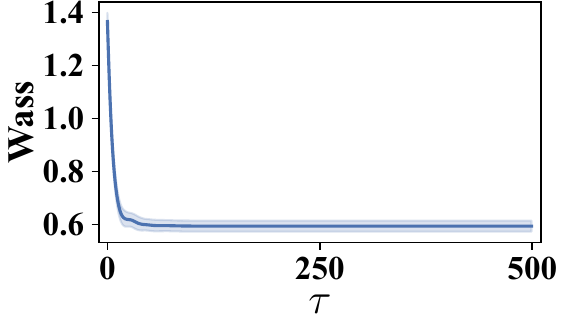}}
\subfigure[MAE, IS.]{\includegraphics[width=0.24\linewidth]{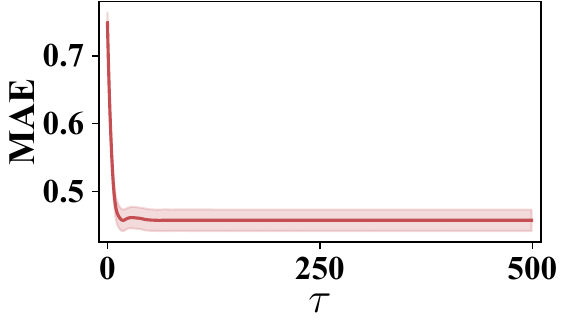}}
\subfigure[Wass, IS.]{\includegraphics[width=0.24\linewidth]{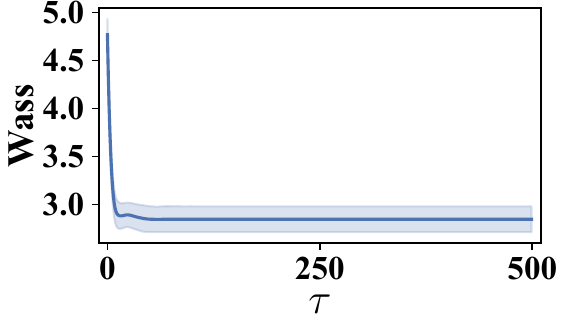}}
\subfigure[MAE, PK.]{\includegraphics[width=0.24\linewidth]{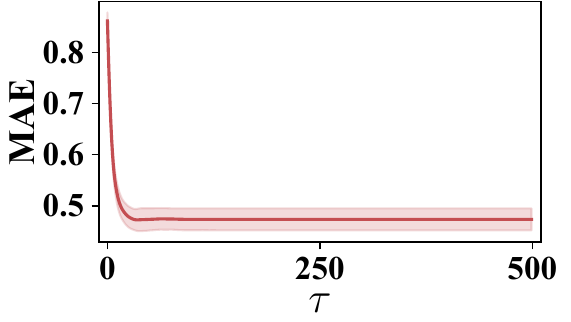}}
\subfigure[MAE, PK.]{\includegraphics[width=0.24\linewidth]{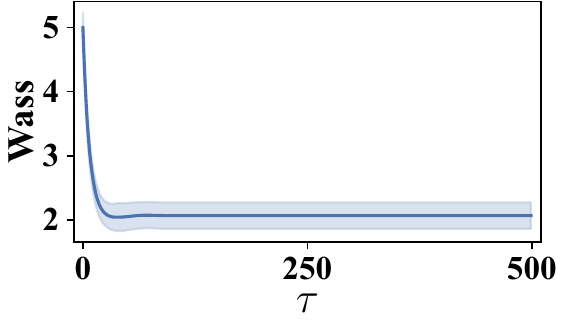}}
\subfigure[MAE, QB.]{\includegraphics[width=0.24\linewidth]{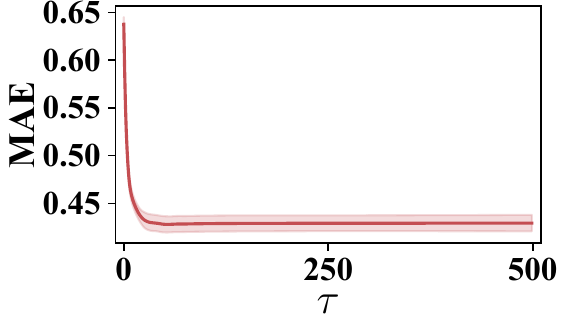}}
\subfigure[Wass, QB.]{\includegraphics[width=0.24\linewidth]{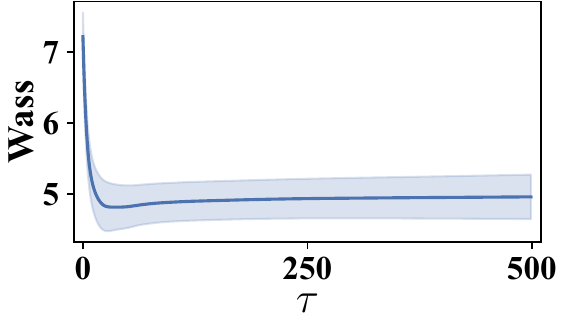}}
\subfigure[MAE, WQW.]{\includegraphics[width=0.24\linewidth]{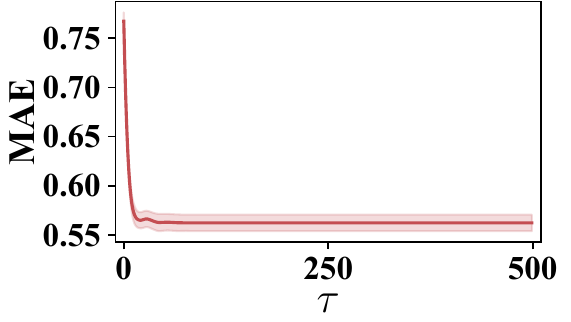}}
\subfigure[MAE, WQW.]{\includegraphics[width=0.24\linewidth]{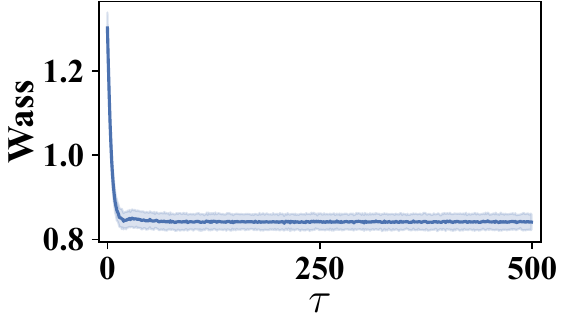}}
  \vspace{-0.3cm}
  \caption{Evolution of evaluation metrics along iteration time $\tau$ under MNAR scenario at 30\% missing rate. The shaded area indicates the $\pm$ 1.0 standard deviation uncertainty interval.}
  \label{appendix-fig:convExperMNAR}
\end{figure}

\newpage
\subsection{Baseline Comparison Vary Different Missing Rates and Scenarios}
In this subsection, we present an extended analysis of model performance across varying missing data rates, as detailed in
~\Cref{fig:extraExperMAR10,fig:extraExperMCAR10,fig:extraExperMNAR10,fig:extraExperMAR50,fig:extraExperMCAR50,fig:extraExperMNAR50}
. The results demonstrate that our KnewImp approach performs competitively across a broad spectrum of missing data regimes.


\begin{figure}[htbp]
  \vspace{-0.5cm}
  \centering
\subfigure{\includegraphics[width=0.32\linewidth]{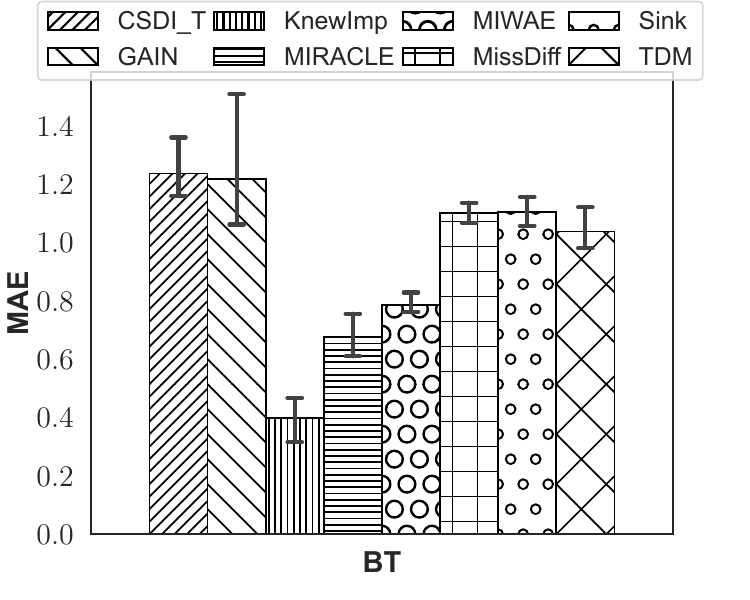}}
\subfigure{\includegraphics[width=0.32\linewidth]{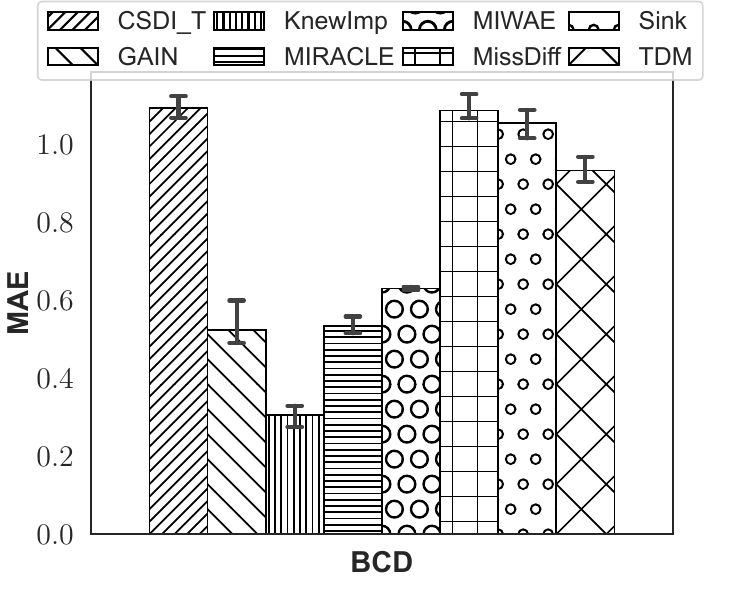}}
\subfigure{\includegraphics[width=0.32\linewidth]{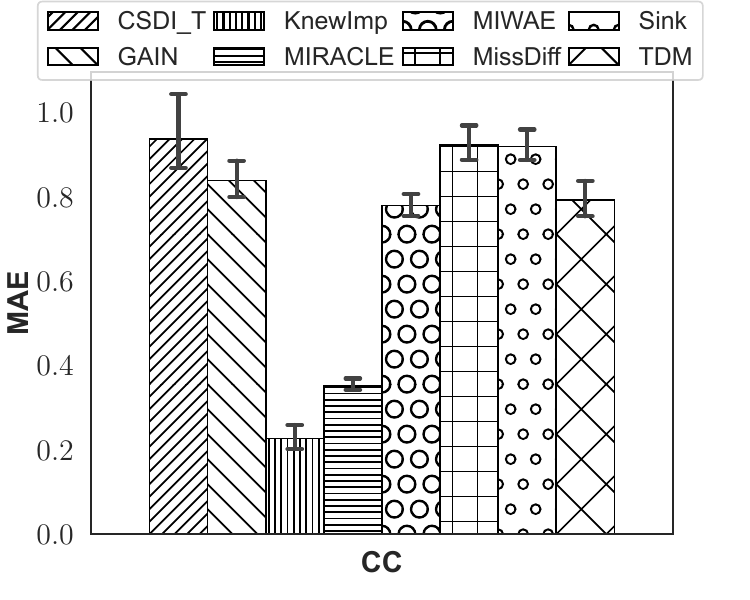}}
\subfigure{\includegraphics[width=0.32\linewidth]{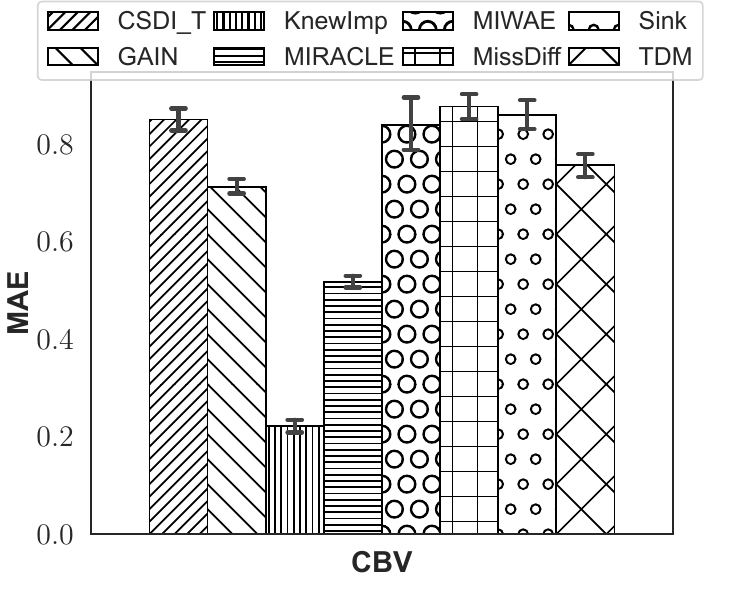}}
\subfigure{\includegraphics[width=0.32\linewidth]{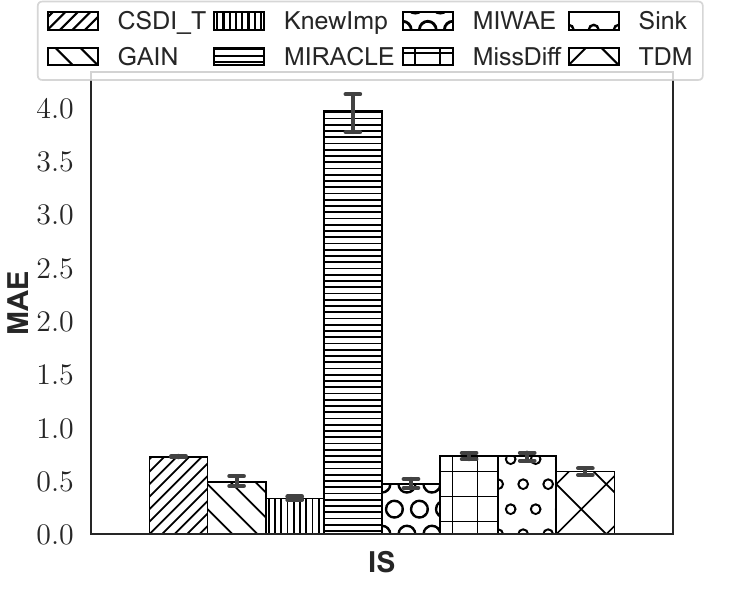}}
\subfigure{\includegraphics[width=0.32\linewidth]{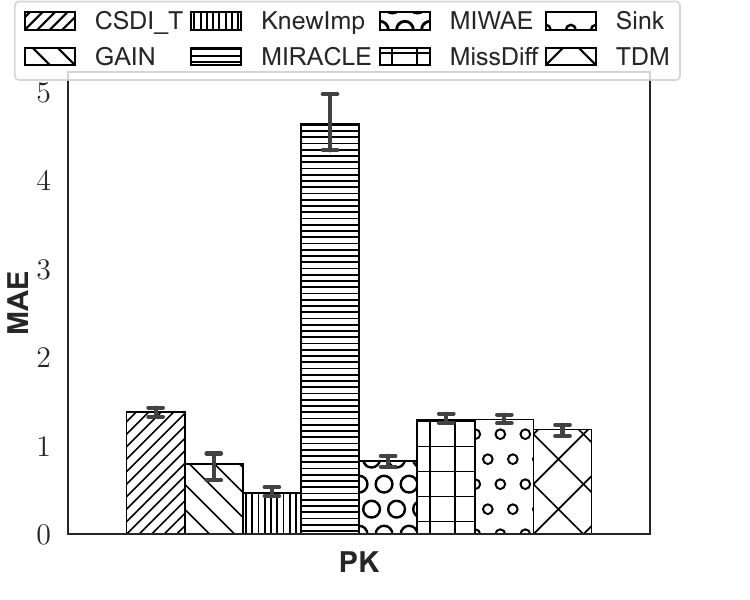}}
\subfigure{\includegraphics[width=0.32\linewidth]{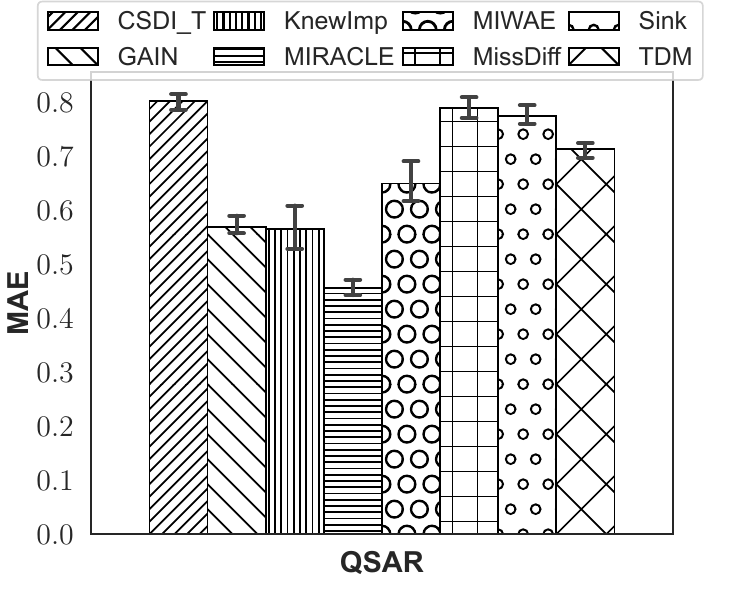}}
\subfigure{\includegraphics[width=0.32\linewidth]{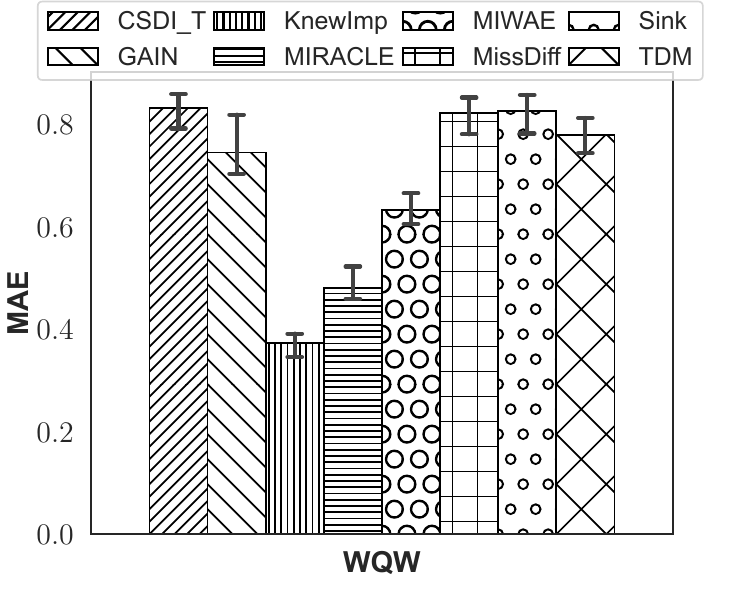}}
\subfigure{\includegraphics[width=0.32\linewidth]{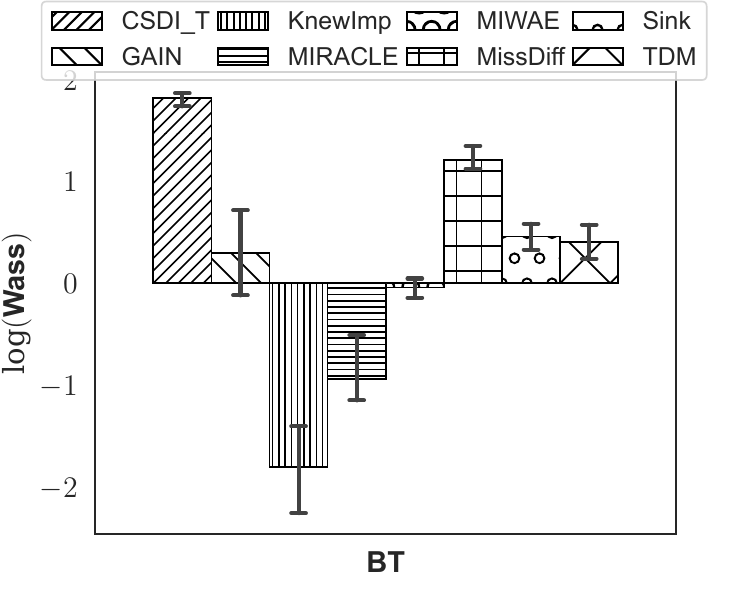}}
\subfigure{\includegraphics[width=0.32\linewidth]{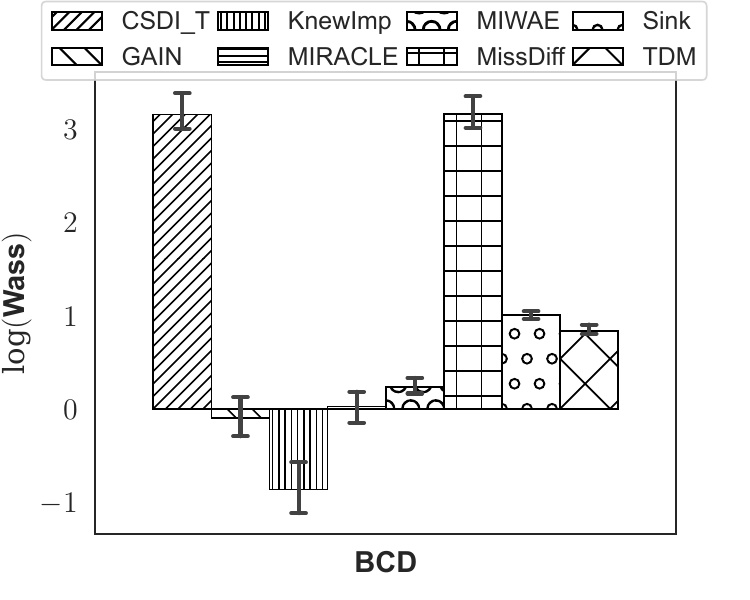}}
\subfigure{\includegraphics[width=0.32\linewidth]{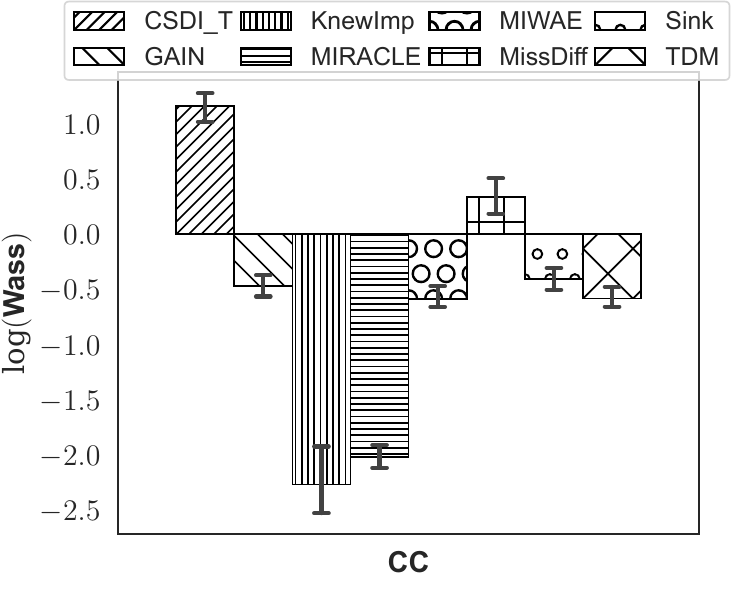}}
\subfigure{\includegraphics[width=0.32\linewidth]{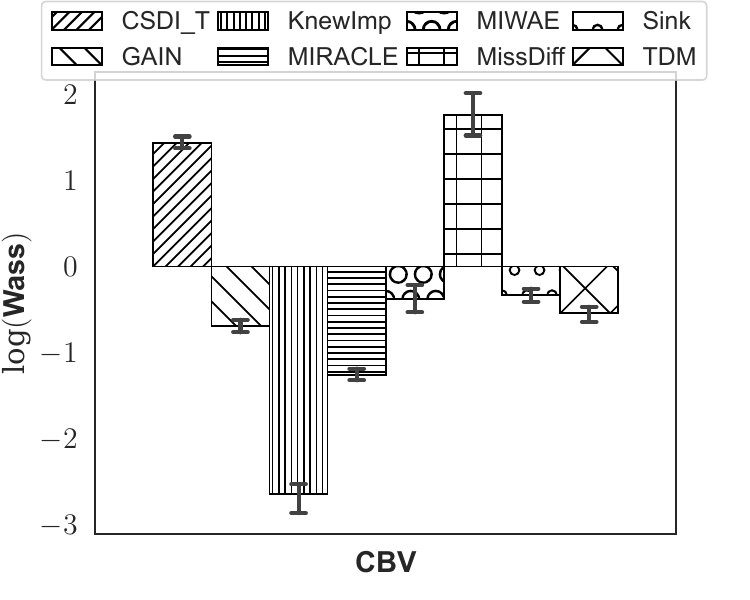}}
\subfigure{\includegraphics[width=0.32\linewidth]{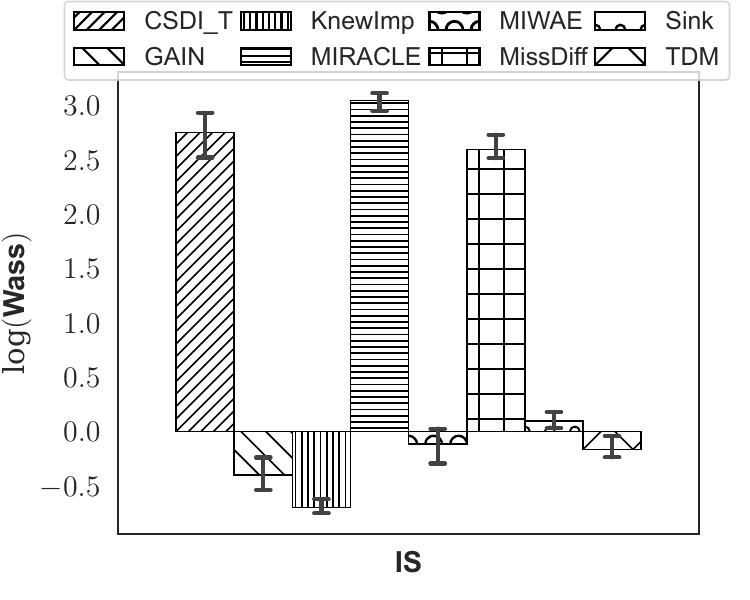}}
\subfigure{\includegraphics[width=0.32\linewidth]{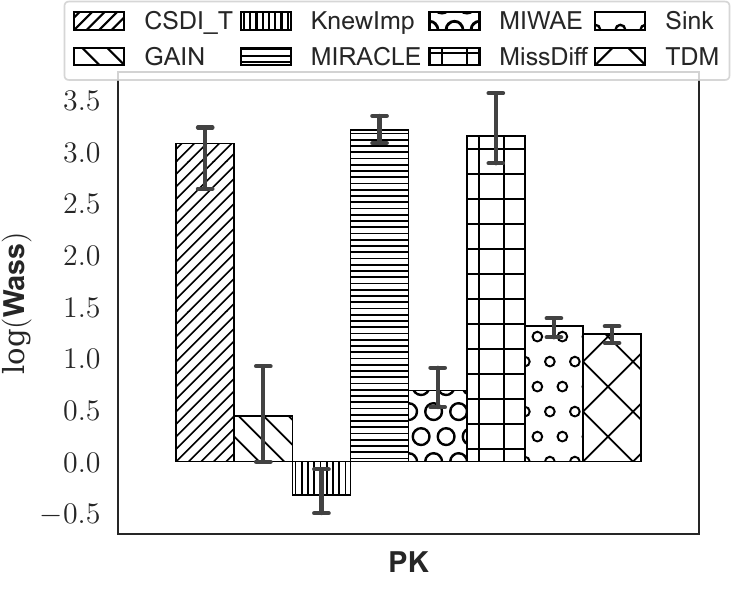}}
\subfigure{\includegraphics[width=0.32\linewidth]{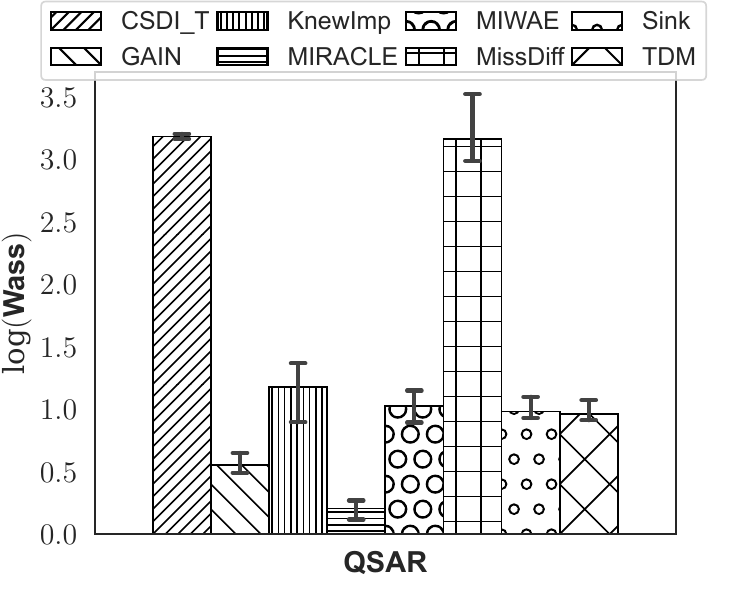}}
\subfigure{\includegraphics[width=0.32\linewidth]{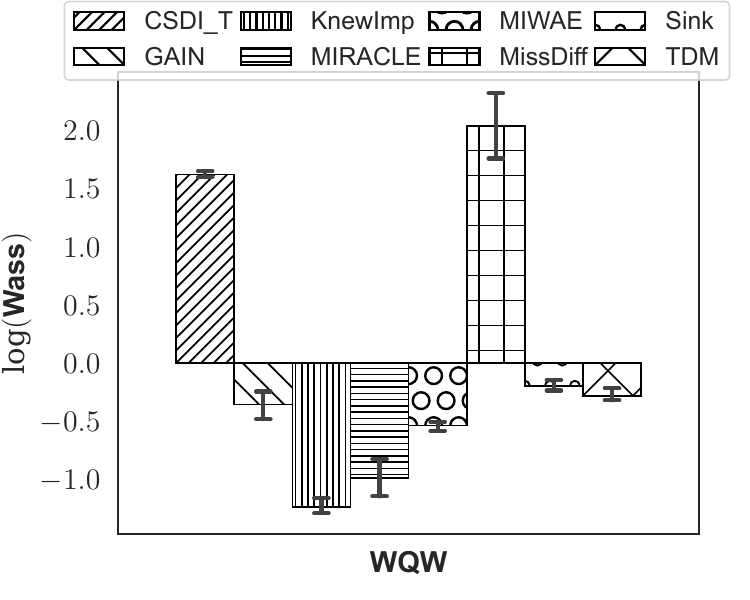}}
\vspace{-0.5cm}
  \caption{Imputation accuracy comparison for MAR scenario at 10\% missing rate. The error bars indicate the 100\% confidence intervals.}
  \label{fig:extraExperMAR10}
\end{figure}

\begin{figure}[htbp]
\vspace{-0.5cm}
\centering
\subfigure{\includegraphics[width=0.32\linewidth]{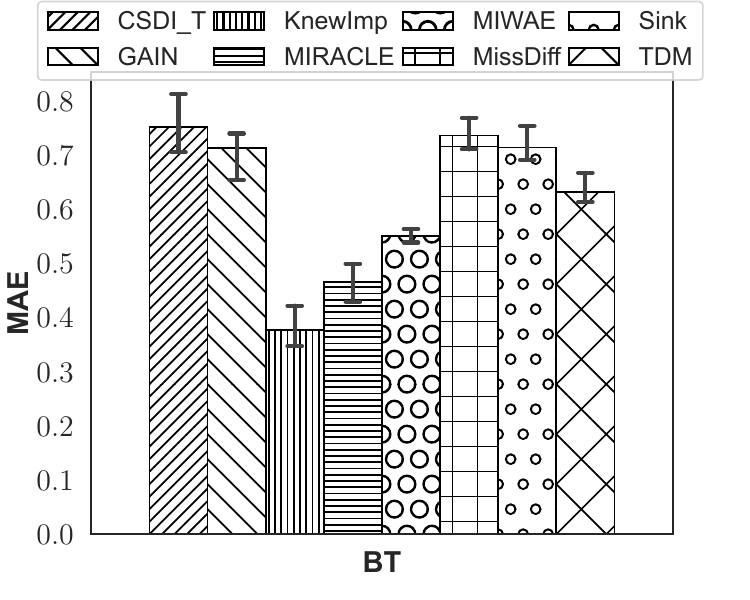}}
\subfigure{\includegraphics[width=0.32\linewidth]{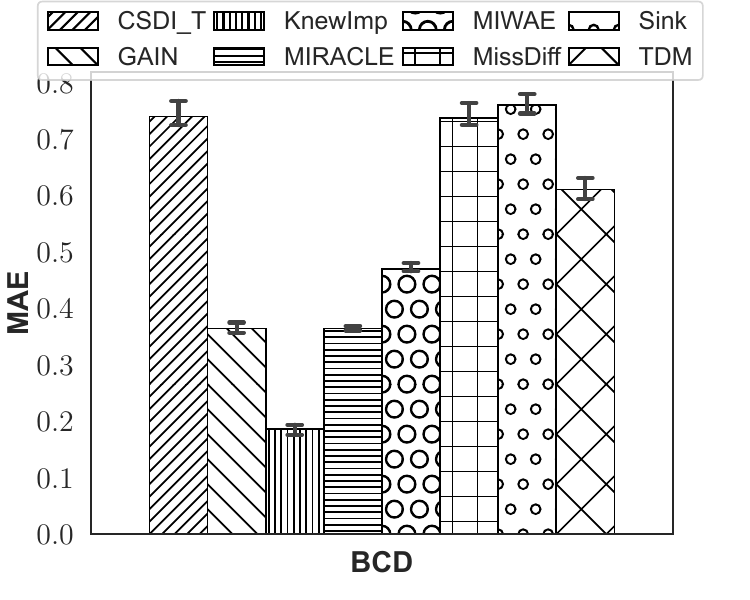}}
\subfigure{\includegraphics[width=0.32\linewidth]{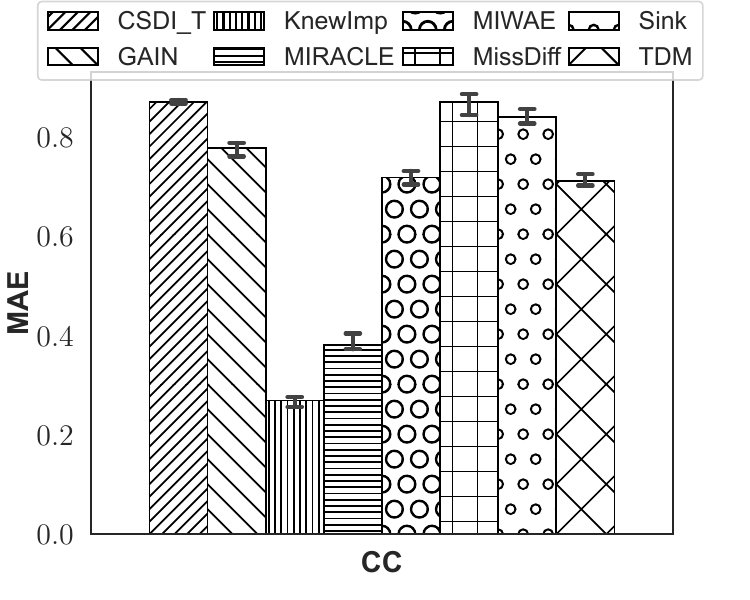}}
\subfigure{\includegraphics[width=0.32\linewidth]{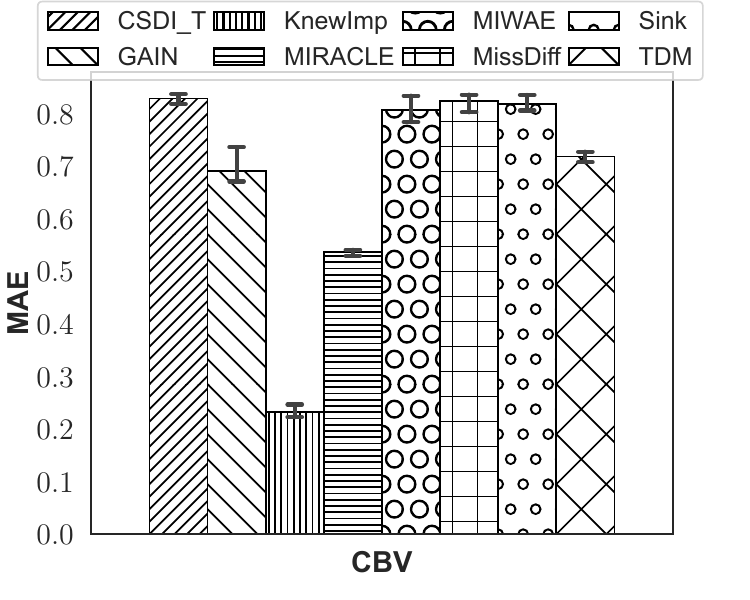}}
\subfigure{\includegraphics[width=0.32\linewidth]{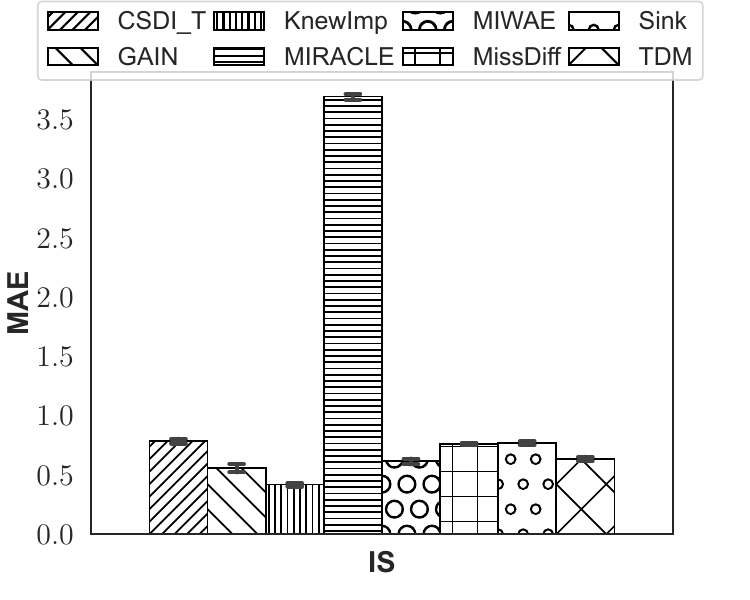}}
\subfigure{\includegraphics[width=0.32\linewidth]{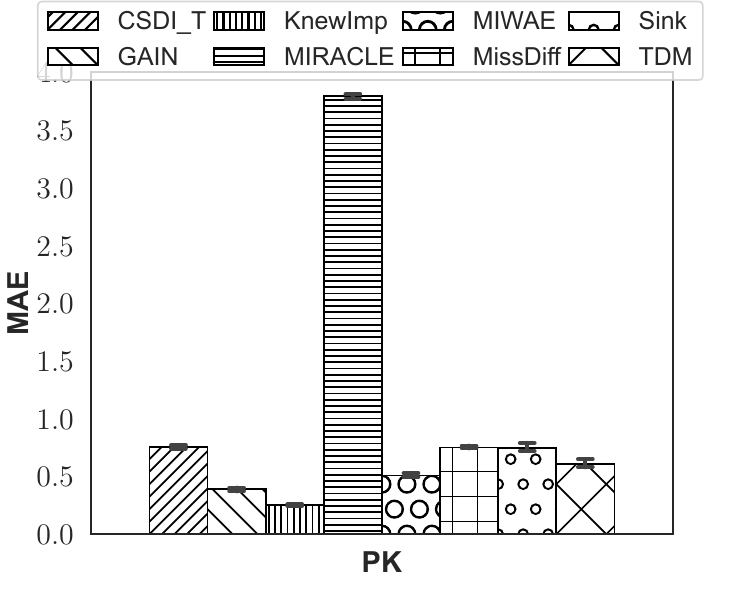}}
\subfigure{\includegraphics[width=0.32\linewidth]{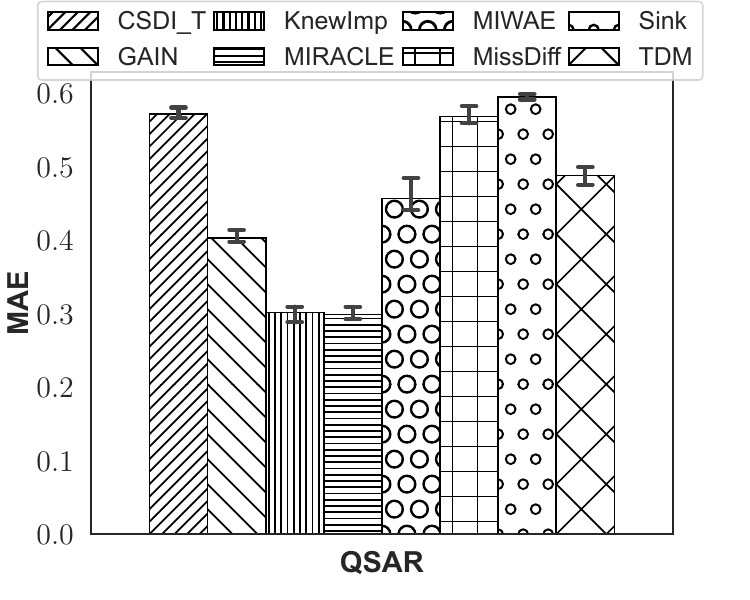}}
\subfigure{\includegraphics[width=0.32\linewidth]{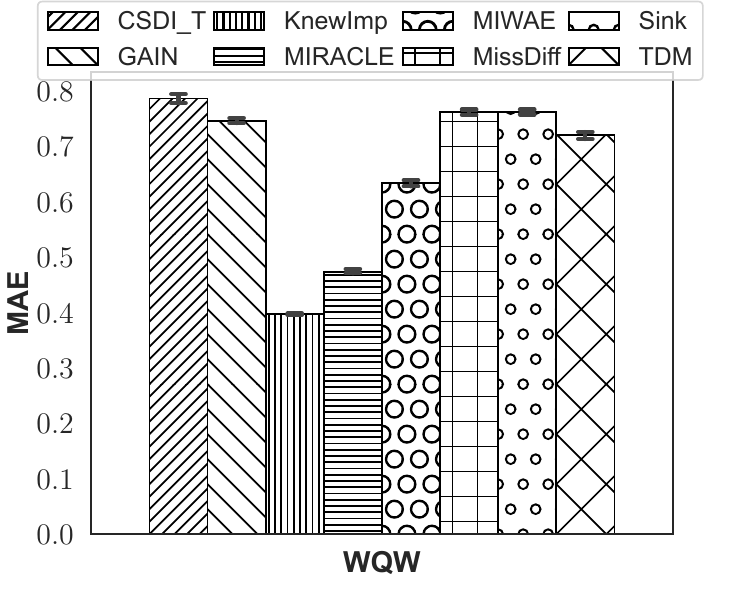}}
\subfigure{\includegraphics[width=0.32\linewidth]{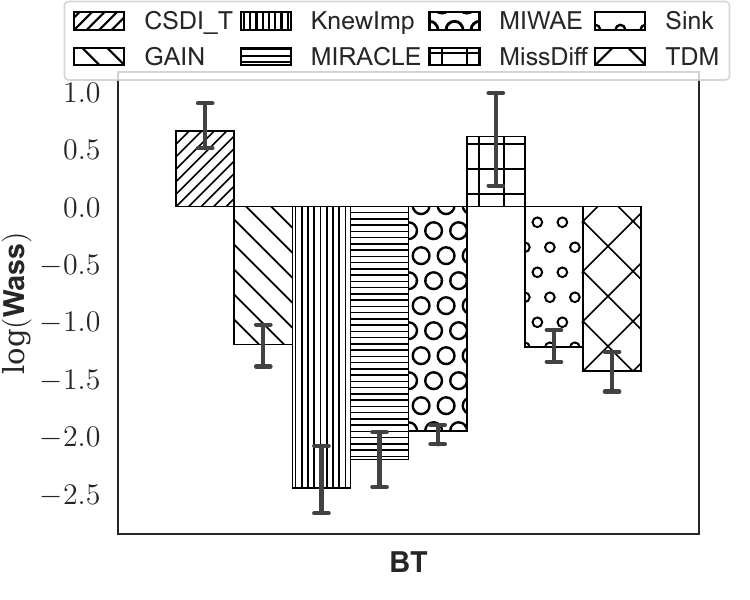}}
\subfigure{\includegraphics[width=0.32\linewidth]{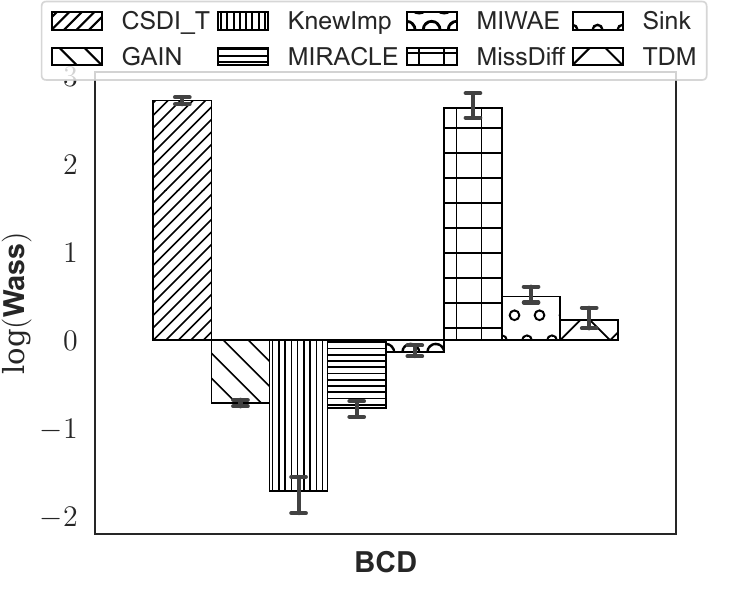}}
\subfigure{\includegraphics[width=0.32\linewidth]{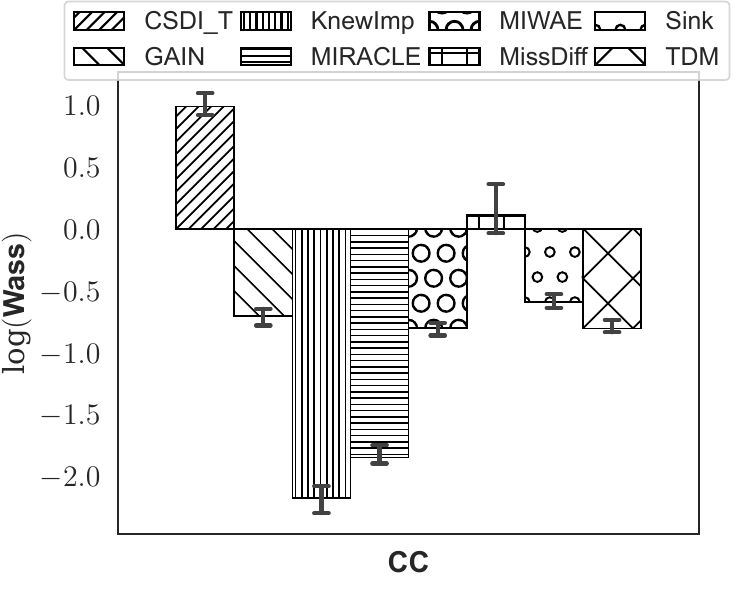}}
\subfigure{\includegraphics[width=0.32\linewidth]{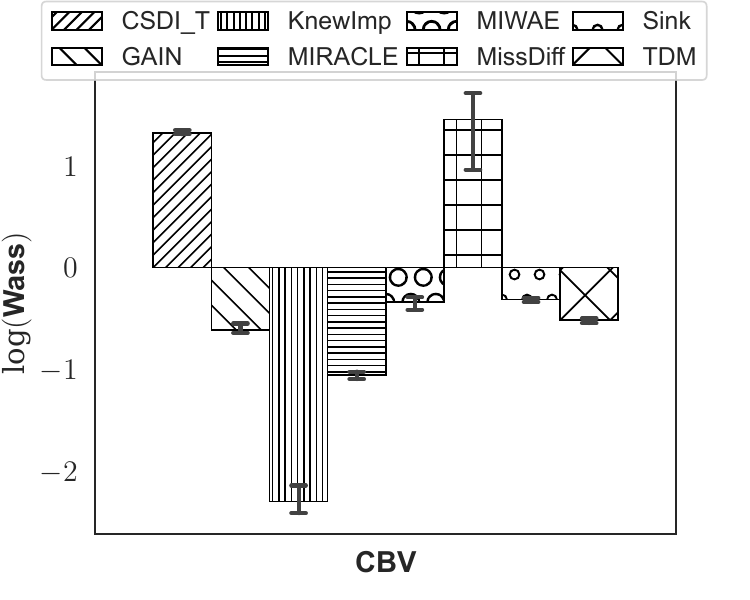}}
\subfigure{\includegraphics[width=0.32\linewidth]{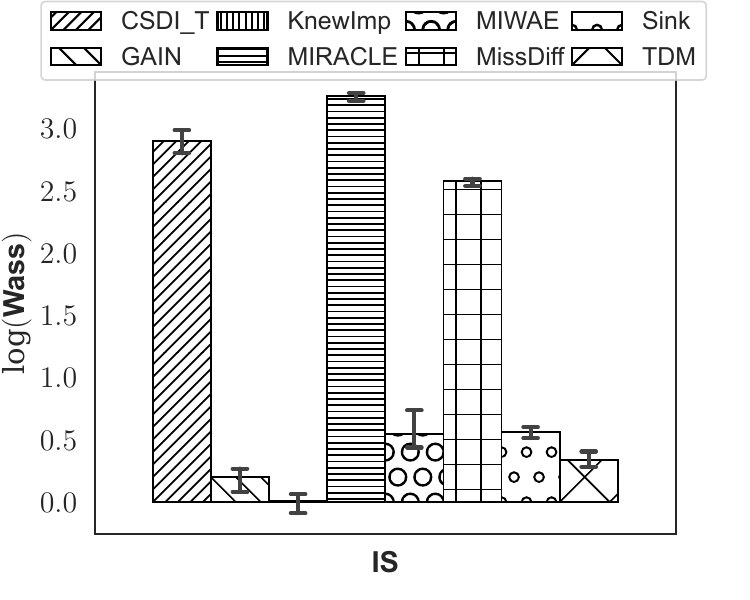}}
\subfigure{\includegraphics[width=0.32\linewidth]{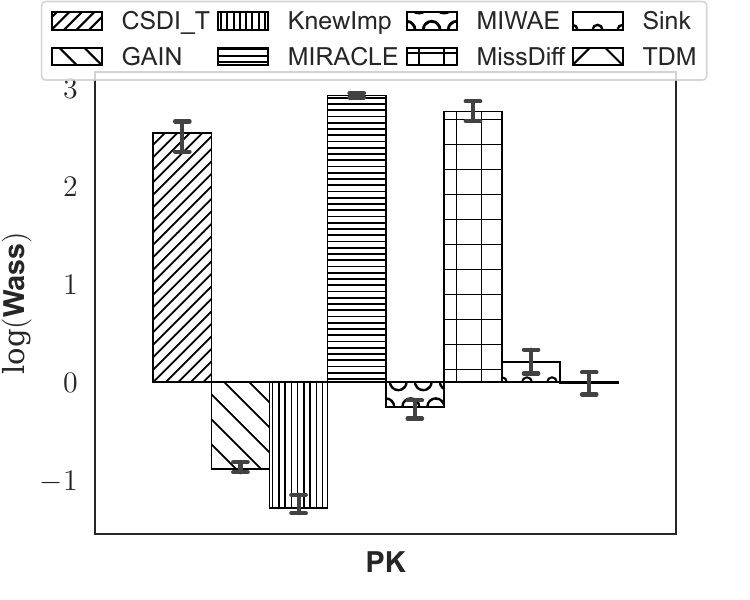}}
\subfigure{\includegraphics[width=0.32\linewidth]{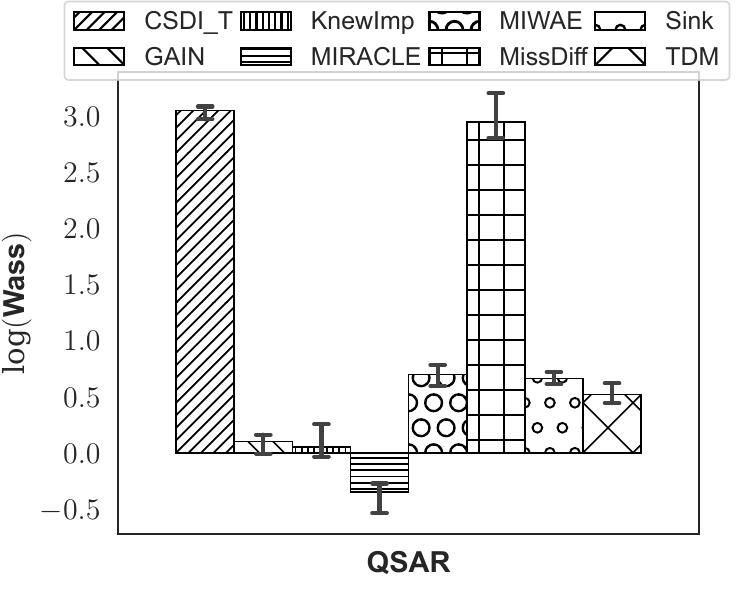}}
\subfigure{\includegraphics[width=0.32\linewidth]{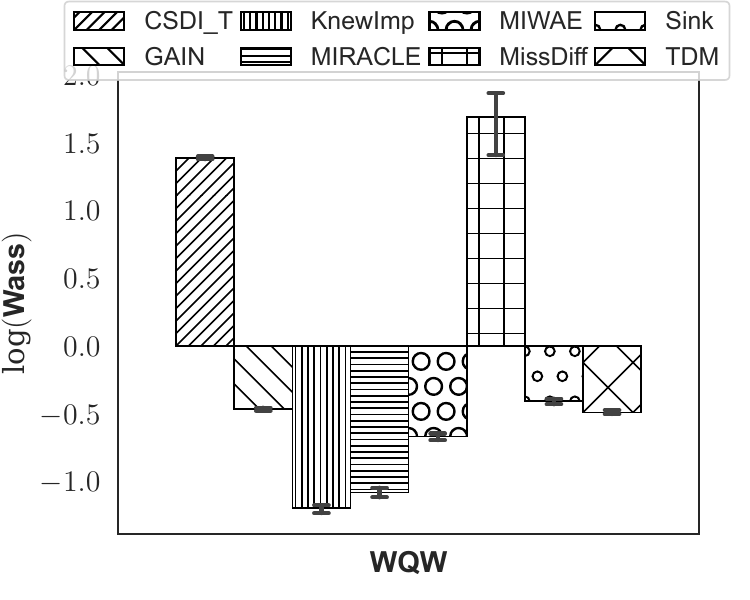}}
  \vspace{-0.5cm}
\caption{Imputation accuracy comparison for MCAR scenario at 10\% missing rate. The error bars indicate the 100\% confidence intervals.}
\label{fig:extraExperMCAR10}
\vspace{-0.5cm}
\end{figure}
\begin{figure}[htbp]
\vspace{-0.5cm}
\centering
\subfigure{\includegraphics[width=0.32\linewidth]{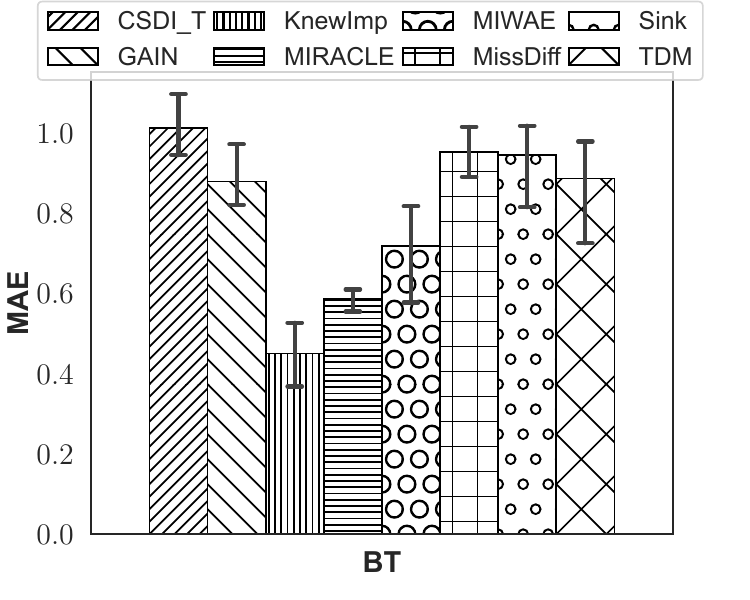}}
\subfigure{\includegraphics[width=0.32\linewidth]{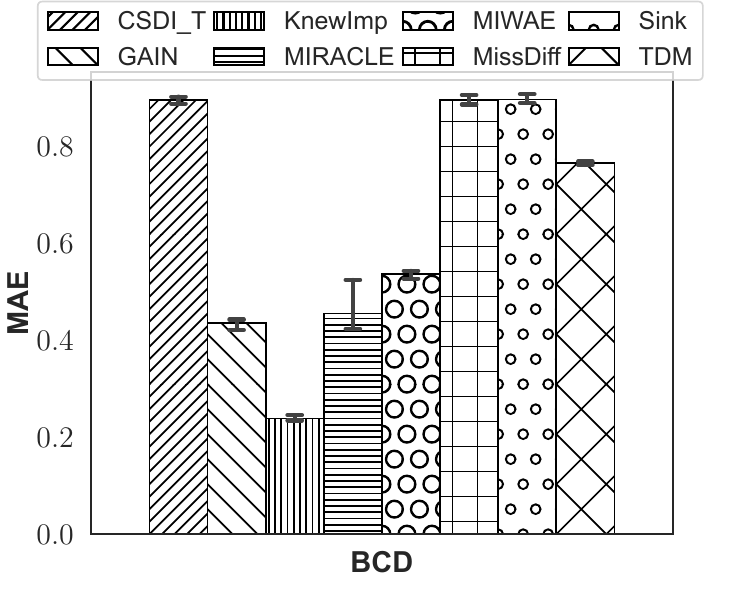}}
\subfigure{\includegraphics[width=0.32\linewidth]{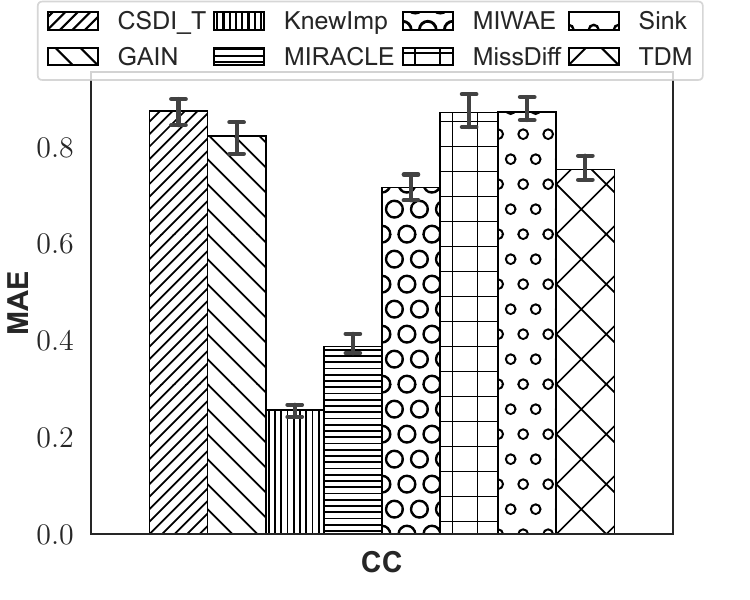}}
\subfigure{\includegraphics[width=0.32\linewidth]{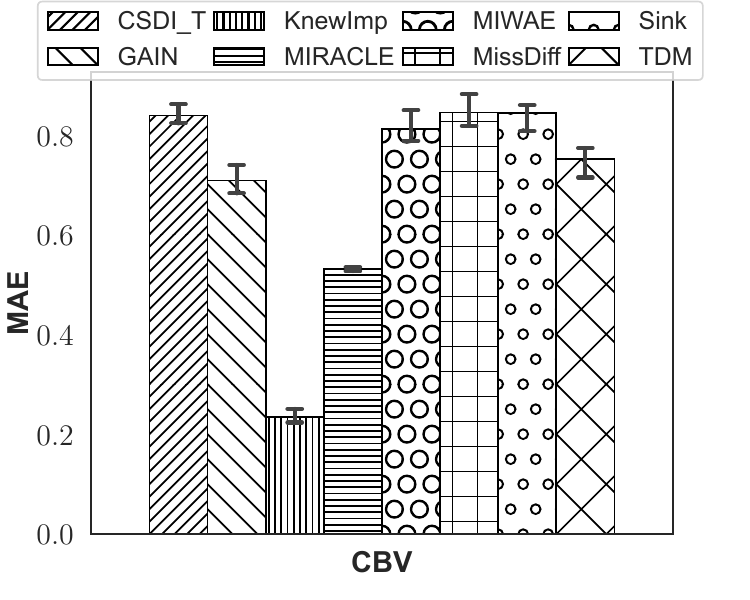}}
\subfigure{\includegraphics[width=0.32\linewidth]{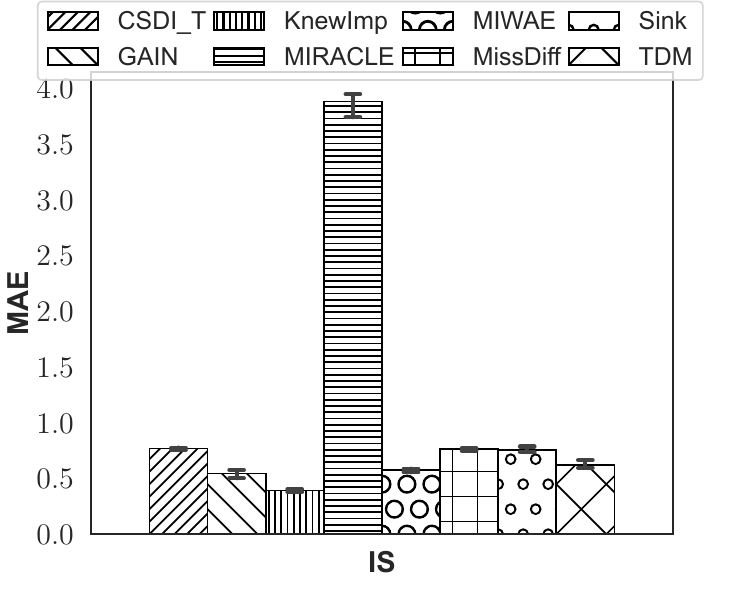}}
\subfigure{\includegraphics[width=0.32\linewidth]{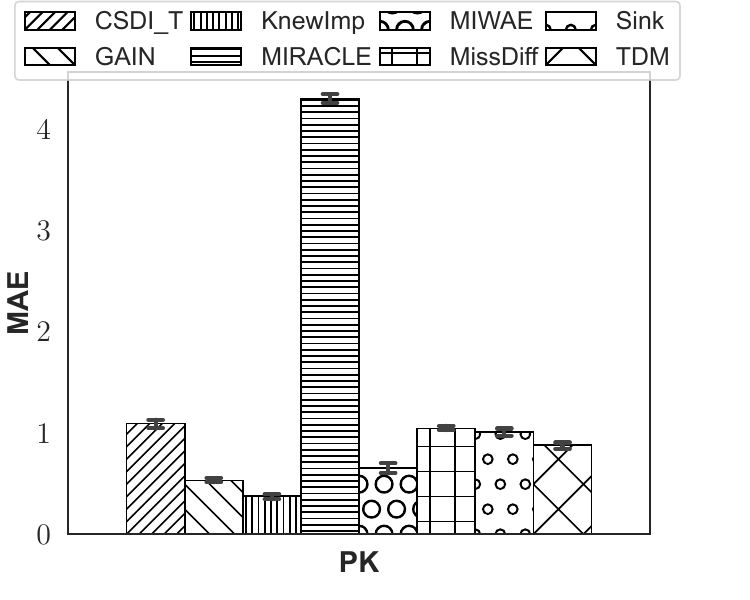}}
\subfigure{\includegraphics[width=0.32\linewidth]{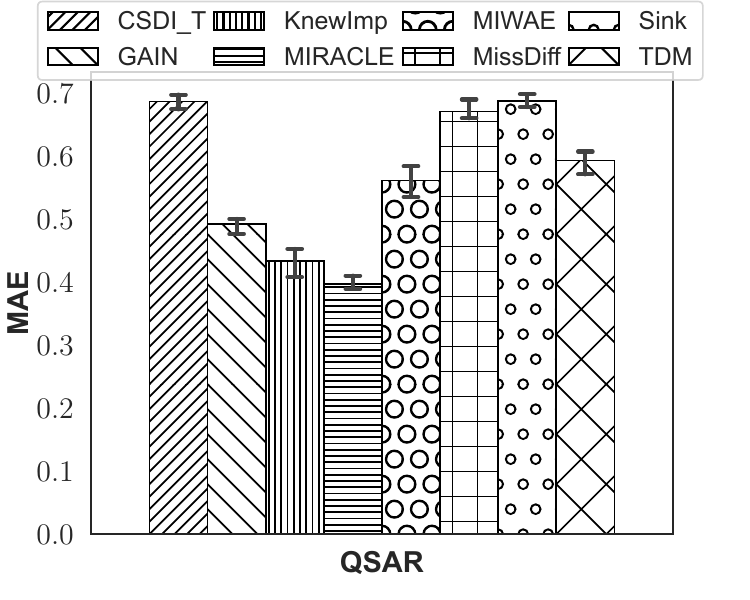}}
\subfigure{\includegraphics[width=0.32\linewidth]{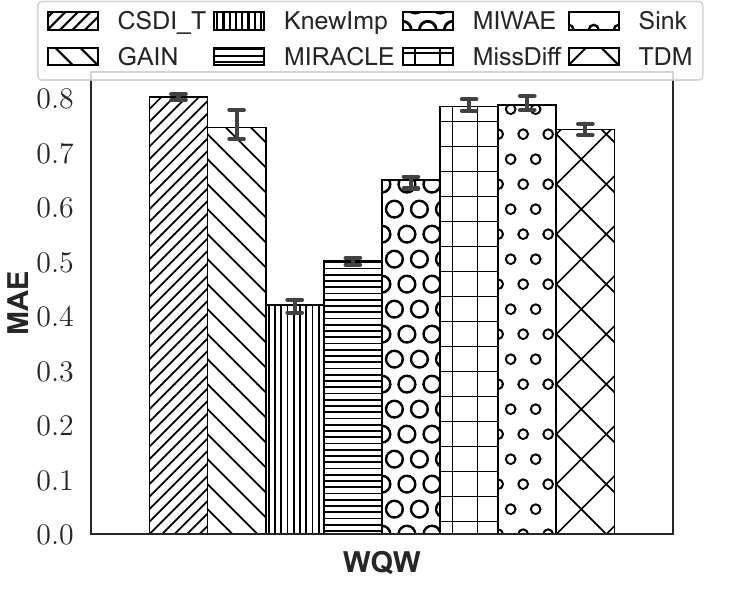}}
\subfigure{\includegraphics[width=0.32\linewidth]{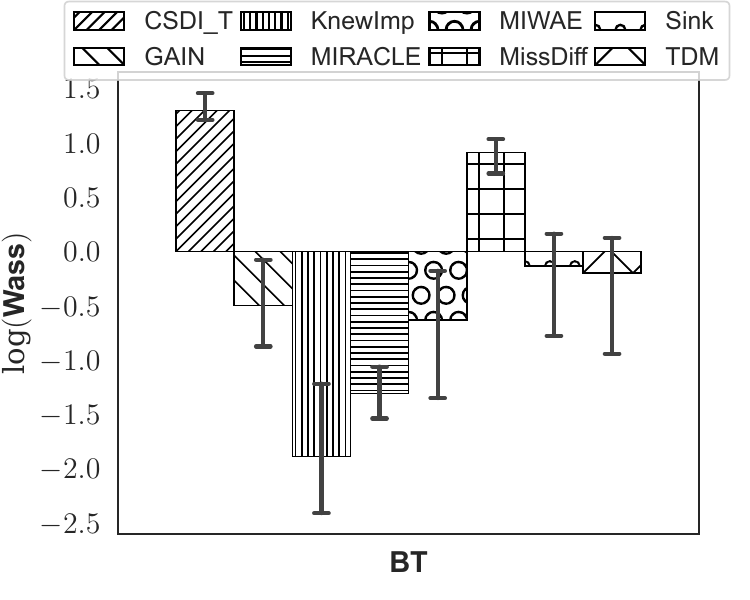}}
\subfigure{\includegraphics[width=0.32\linewidth]{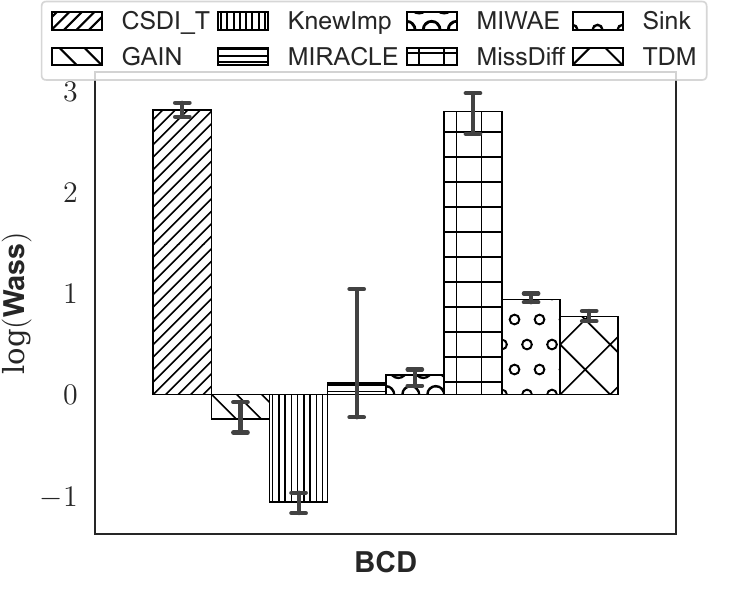}}
\subfigure{\includegraphics[width=0.32\linewidth]{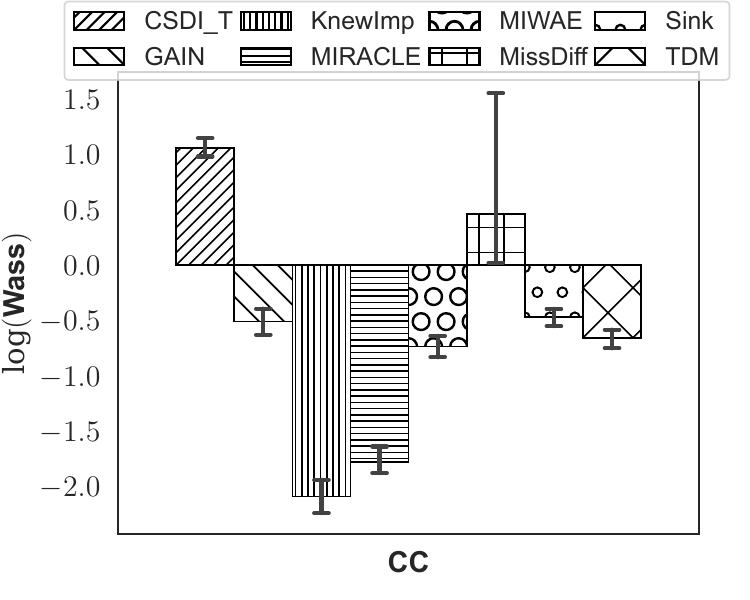}}
\subfigure{\includegraphics[width=0.32\linewidth]{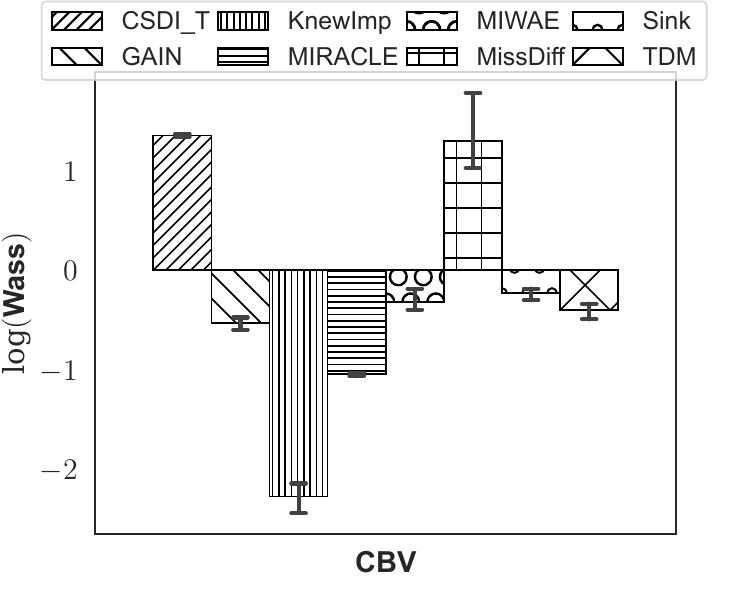}}
\subfigure{\includegraphics[width=0.32\linewidth]{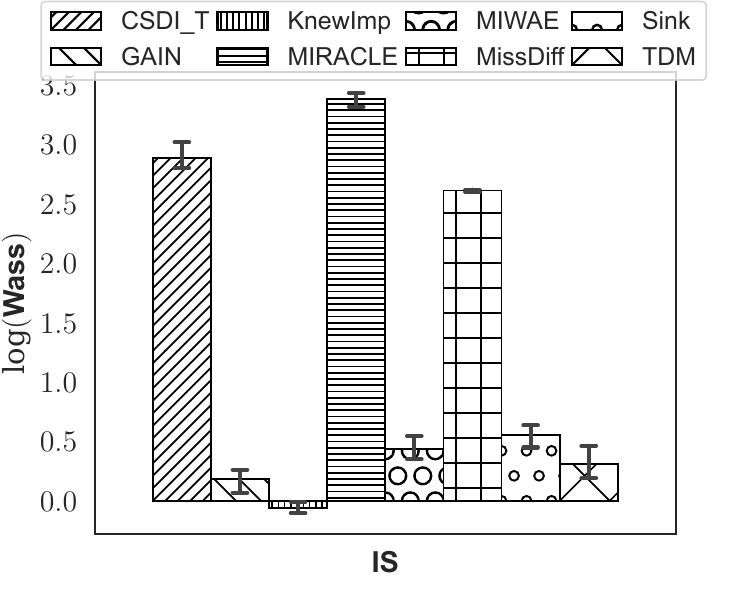}}
\subfigure{\includegraphics[width=0.32\linewidth]{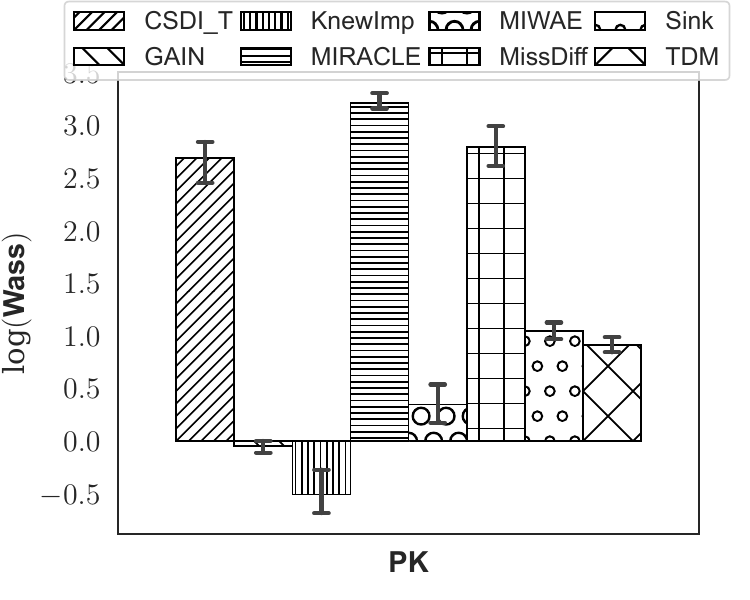}}
\subfigure{\includegraphics[width=0.32\linewidth]{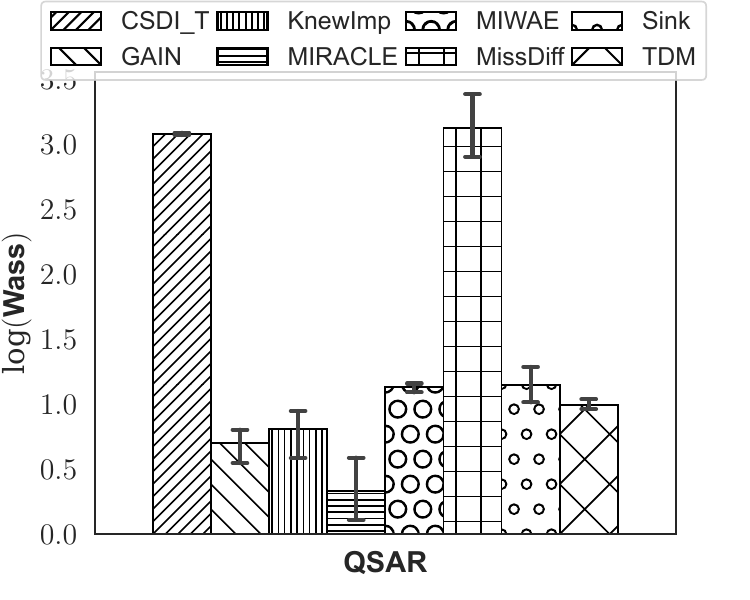}}
\subfigure{\includegraphics[width=0.32\linewidth]{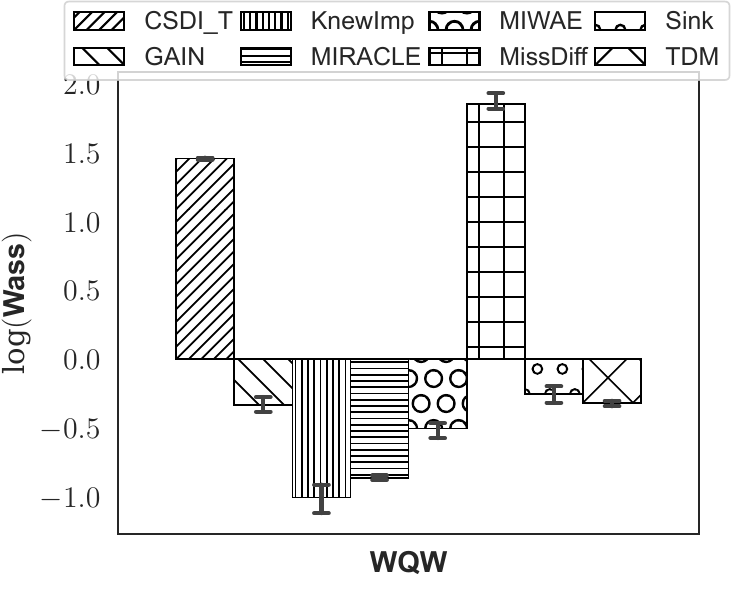}}
\vspace{-0.5cm}
\caption{Imputation accuracy comparison for MNAR scenario at 10\% missing rate. The error bars indicate the 100\% confidence intervals.}
\label{fig:extraExperMNAR10}
  \vspace{-0.5cm}
\end{figure}

\begin{figure}[htbp]
  \vspace{-0.5cm}
  \centering
\subfigure{\includegraphics[width=0.32\linewidth]{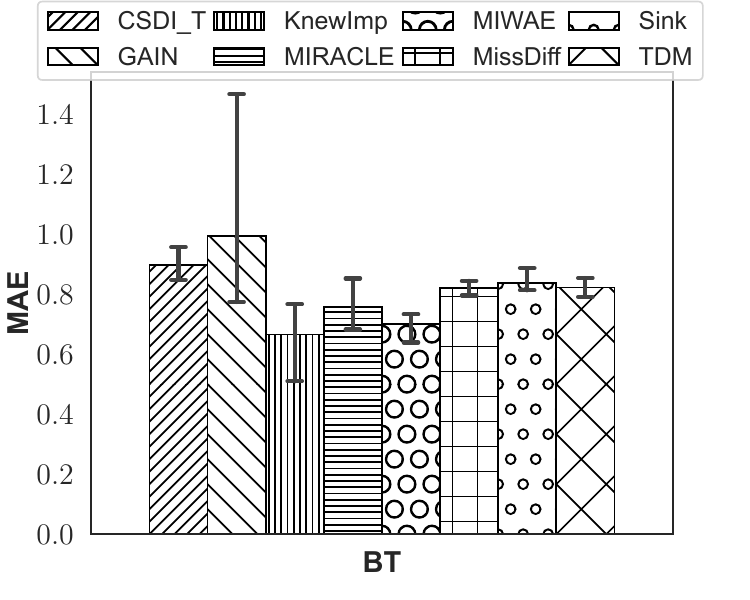}}
\subfigure{\includegraphics[width=0.32\linewidth]{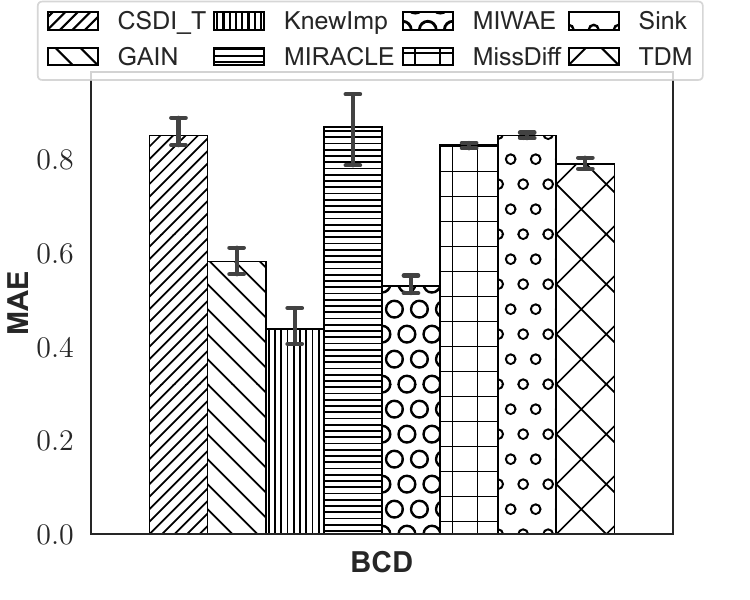}}
\subfigure{\includegraphics[width=0.32\linewidth]{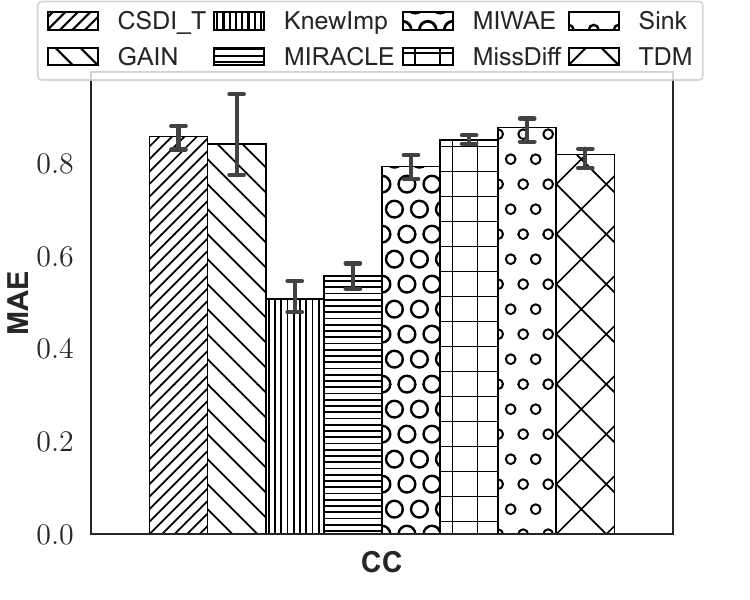}}
\subfigure{\includegraphics[width=0.32\linewidth]{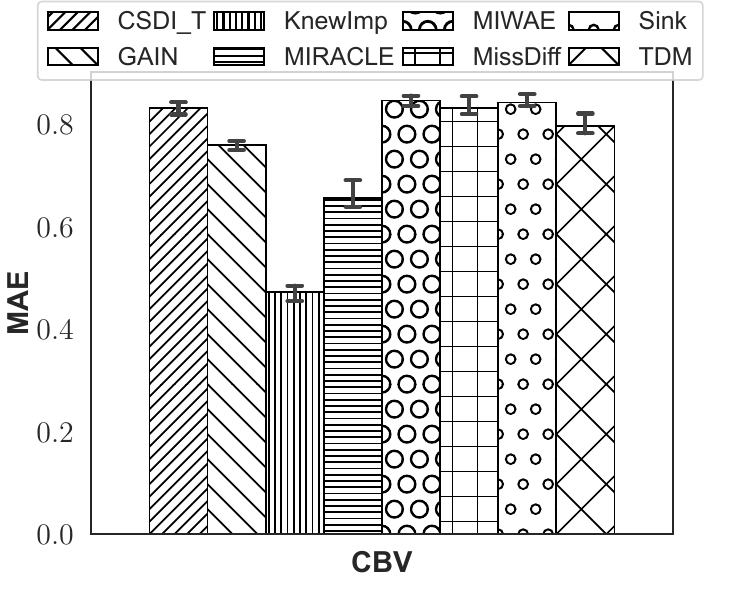}}
\subfigure{\includegraphics[width=0.32\linewidth]{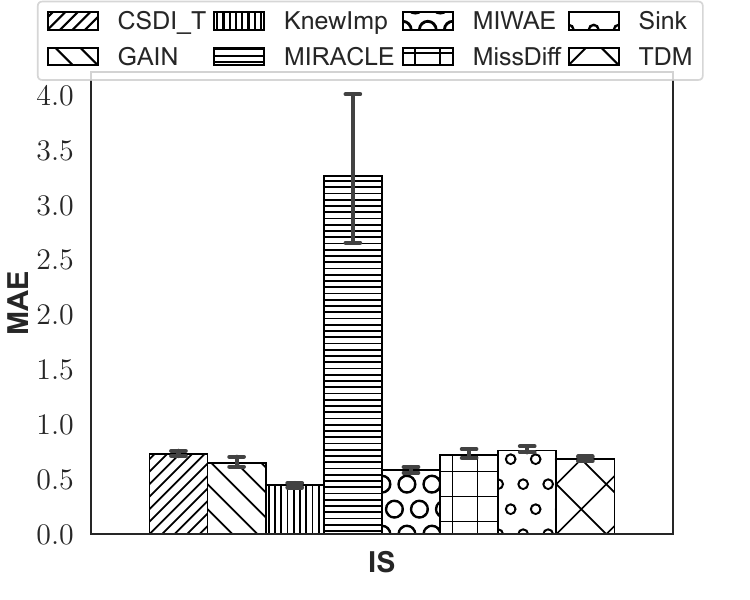}}
\subfigure{\includegraphics[width=0.32\linewidth]{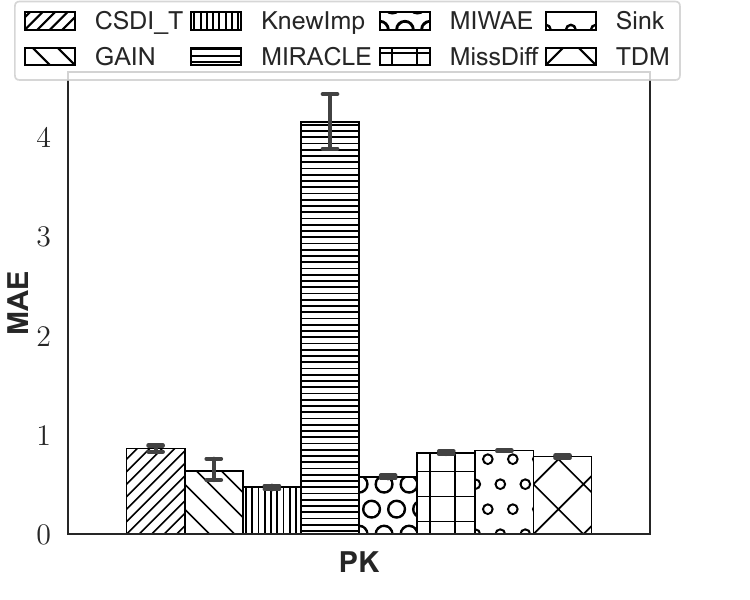}}
\subfigure{\includegraphics[width=0.32\linewidth]{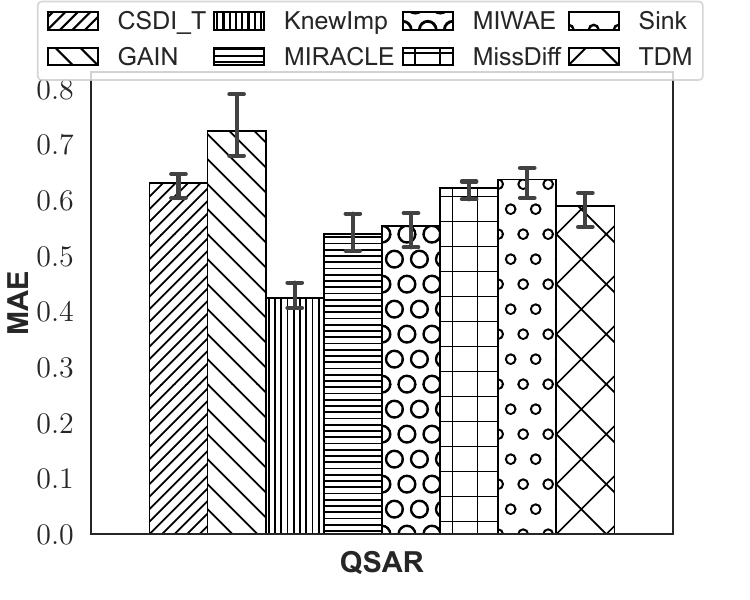}}
\subfigure{\includegraphics[width=0.32\linewidth]{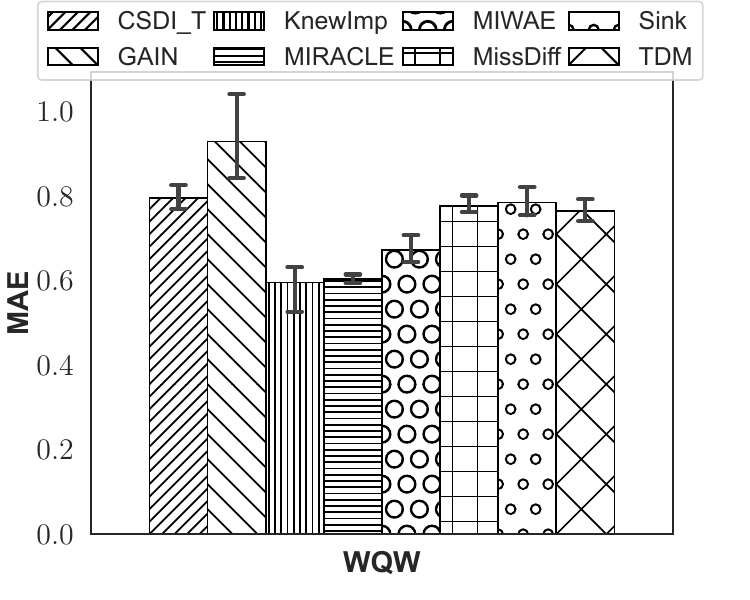}}
\subfigure{\includegraphics[width=0.32\linewidth]{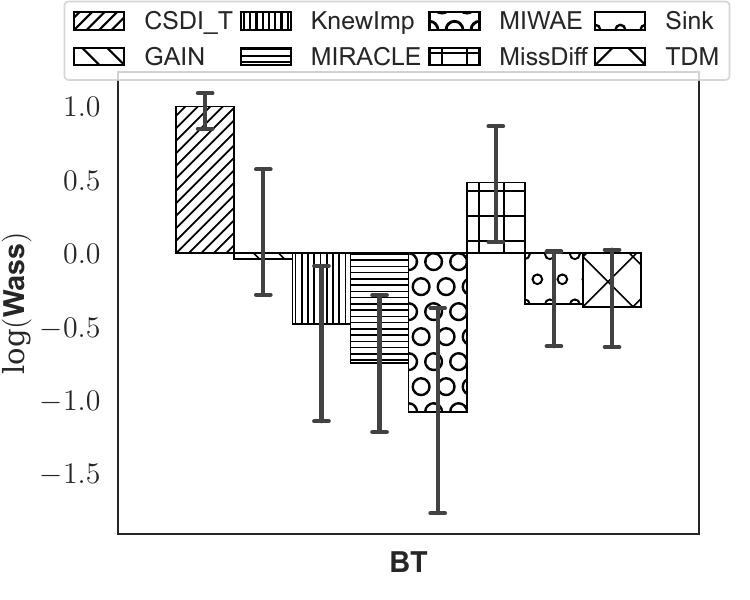}}
\subfigure{\includegraphics[width=0.32\linewidth]{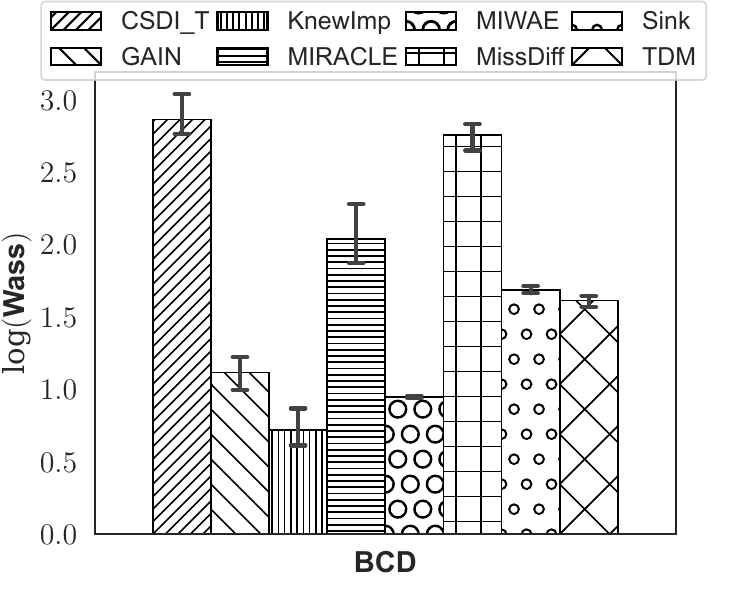}}
\subfigure{\includegraphics[width=0.32\linewidth]{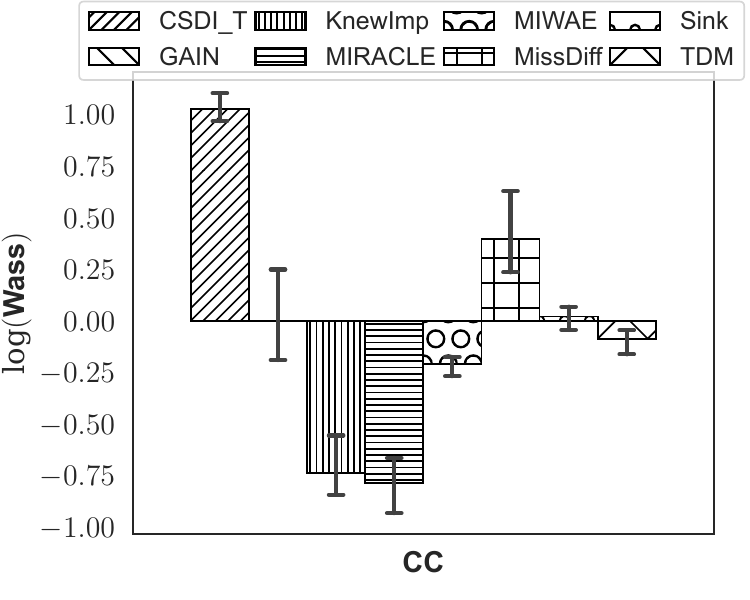}}
\subfigure{\includegraphics[width=0.32\linewidth]{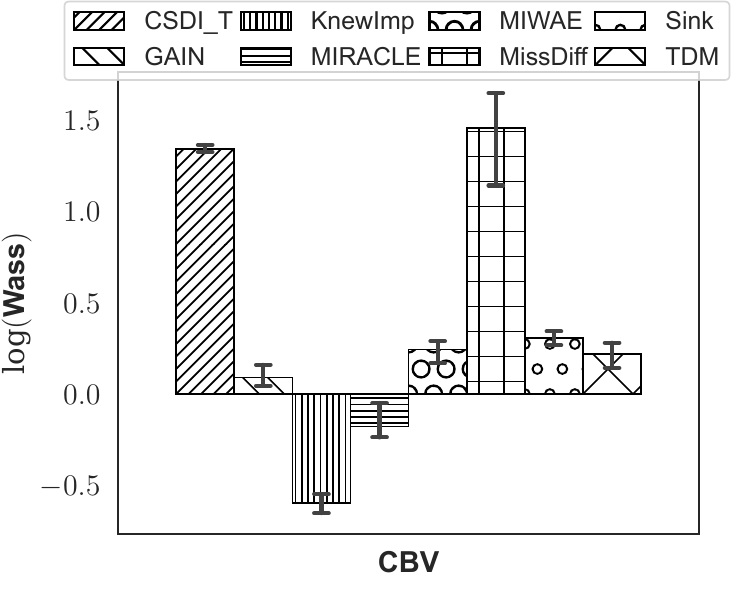}}
\subfigure{\includegraphics[width=0.32\linewidth]{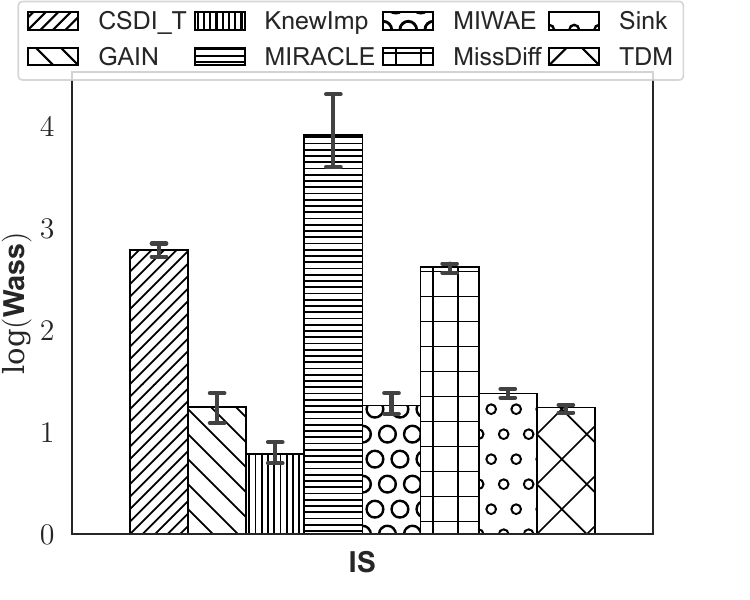}}
\subfigure{\includegraphics[width=0.32\linewidth]{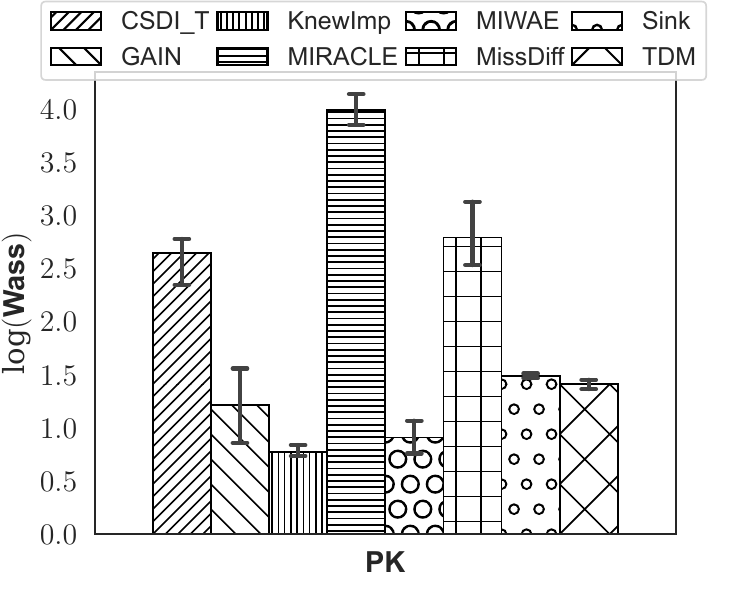}}
\subfigure{\includegraphics[width=0.32\linewidth]{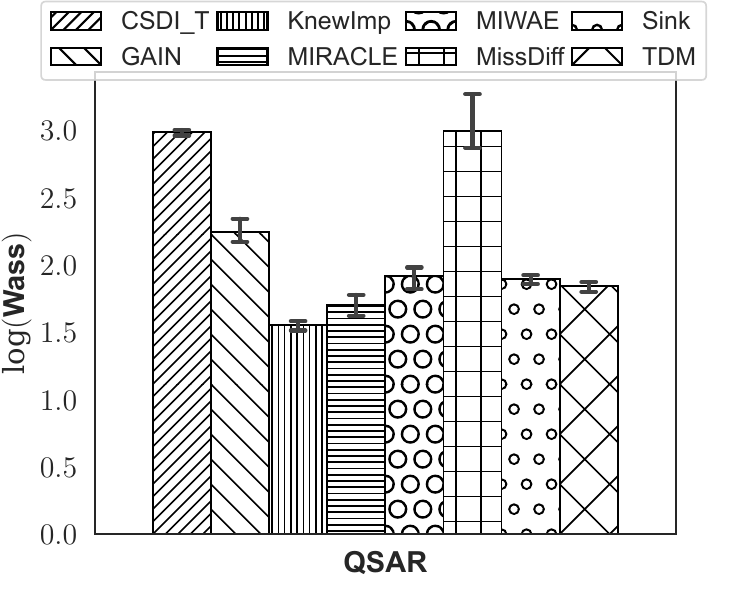}}
\subfigure{\includegraphics[width=0.32\linewidth]{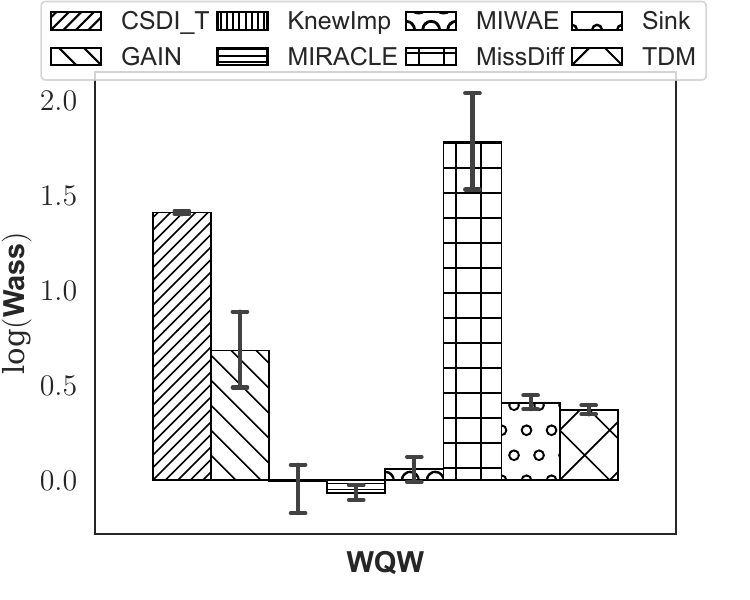}}
\vspace{-0.5cm}
  \caption{Imputation accuracy comparison for MAR scenario at 50\% missing rate. The error bars indicate the 100\% confidence intervals.}
  \label{fig:extraExperMAR50}
\end{figure}

\begin{figure}[htbp]
\vspace{-0.5cm}
\centering
\subfigure{\includegraphics[width=0.32\linewidth]{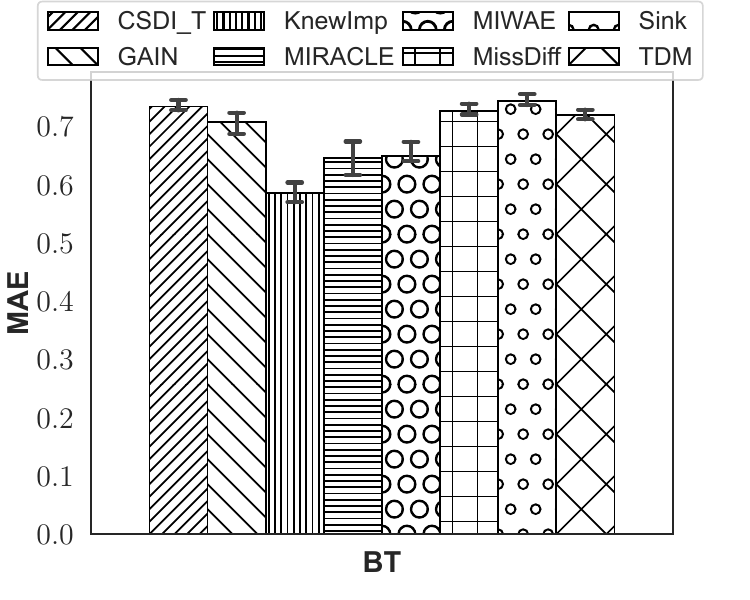}}
\subfigure{\includegraphics[width=0.32\linewidth]{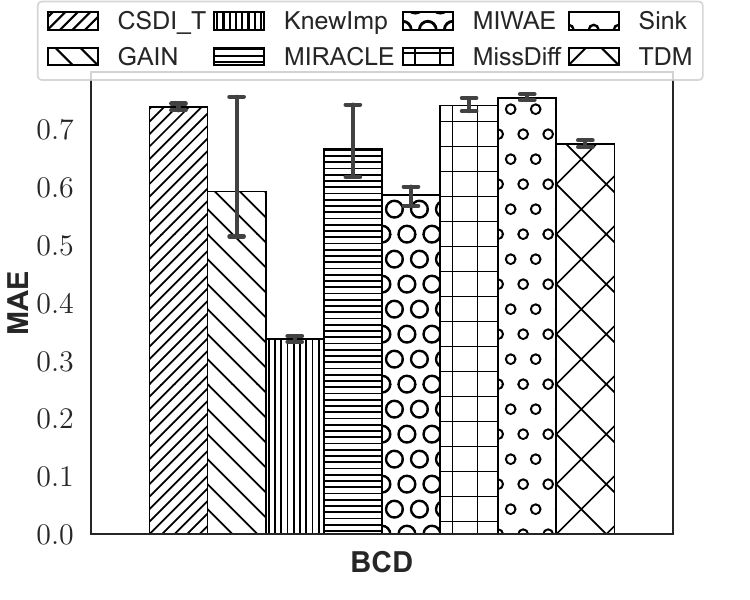}}
\subfigure{\includegraphics[width=0.32\linewidth]{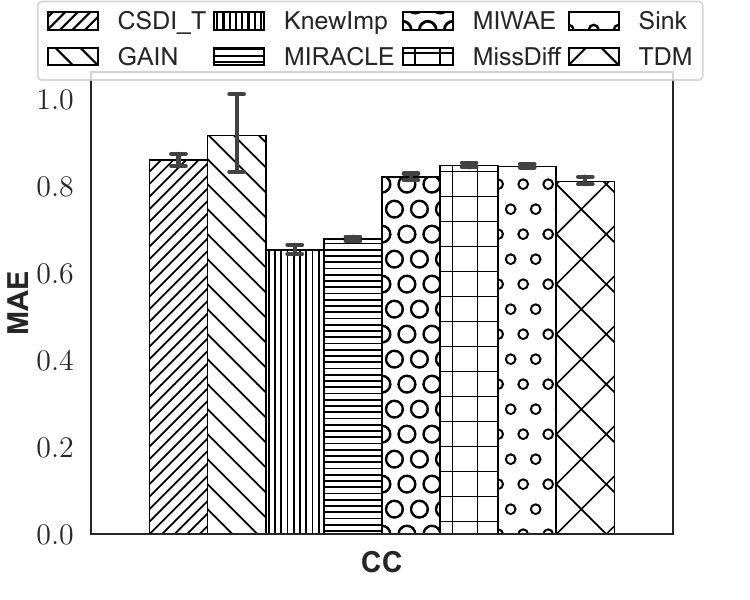}}
\subfigure{\includegraphics[width=0.32\linewidth]{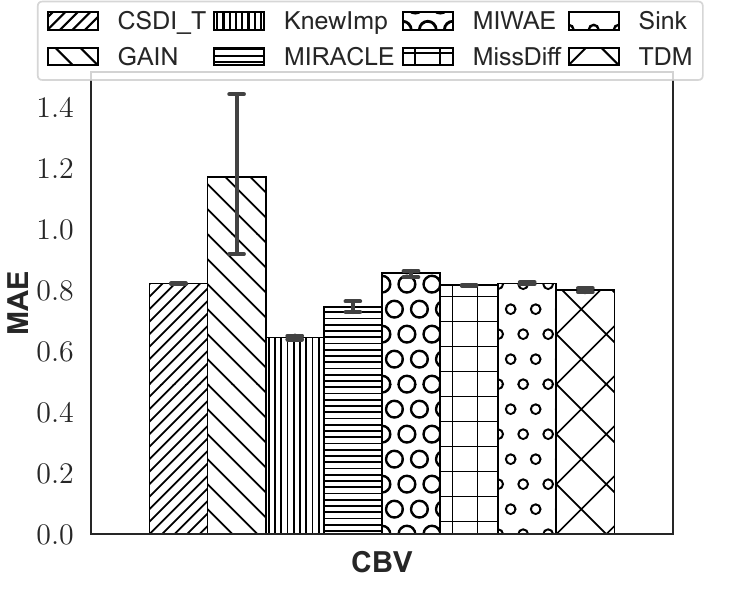}}
\subfigure{\includegraphics[width=0.32\linewidth]{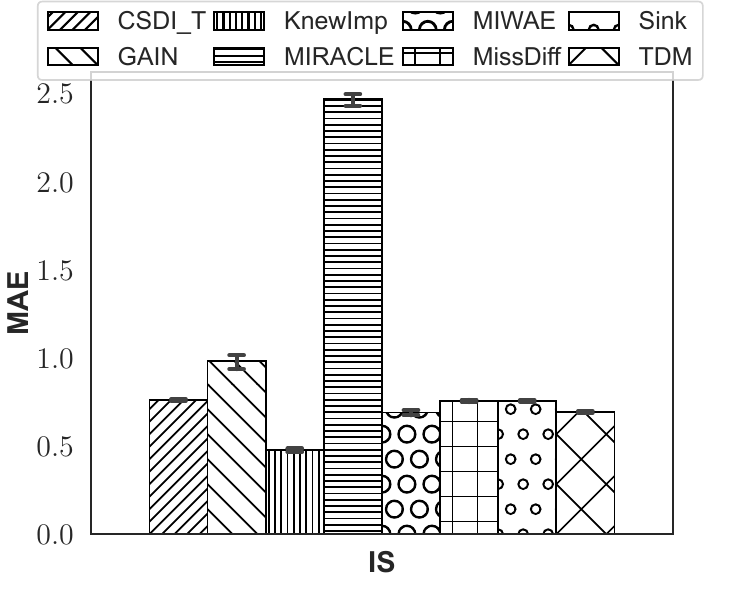}}
\subfigure{\includegraphics[width=0.32\linewidth]{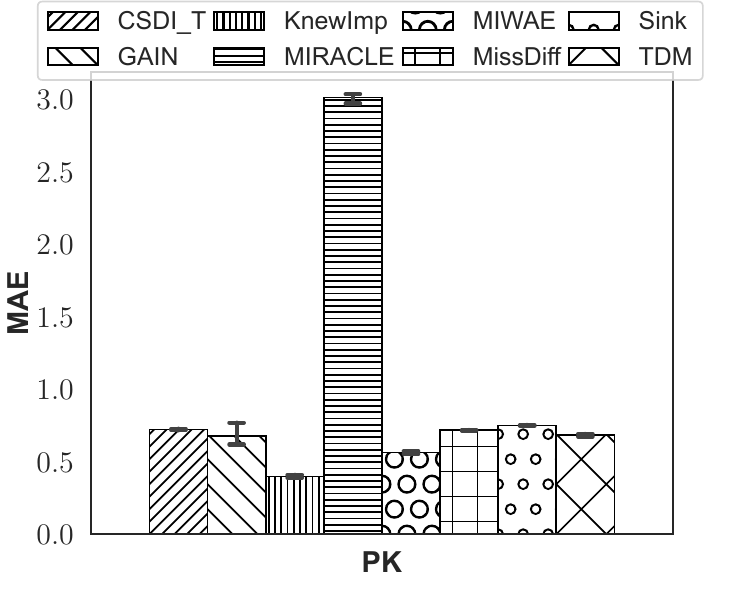}}
\subfigure{\includegraphics[width=0.32\linewidth]{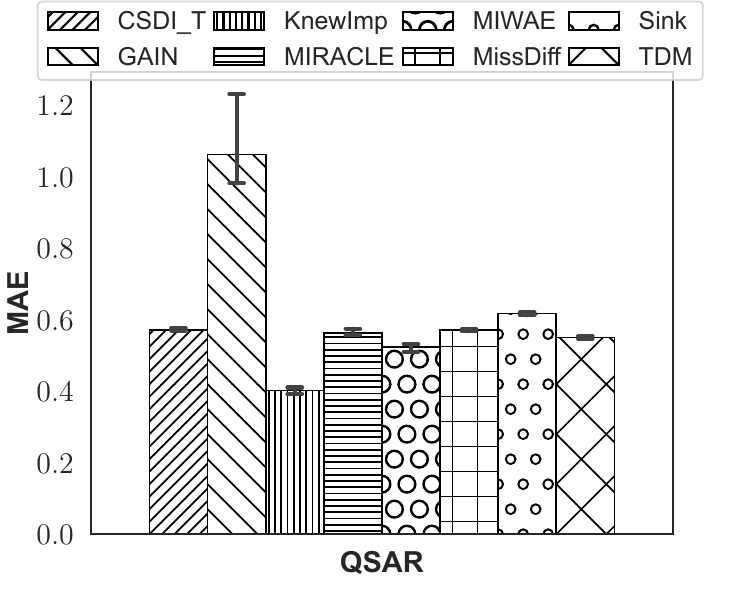}}
\subfigure{\includegraphics[width=0.32\linewidth]{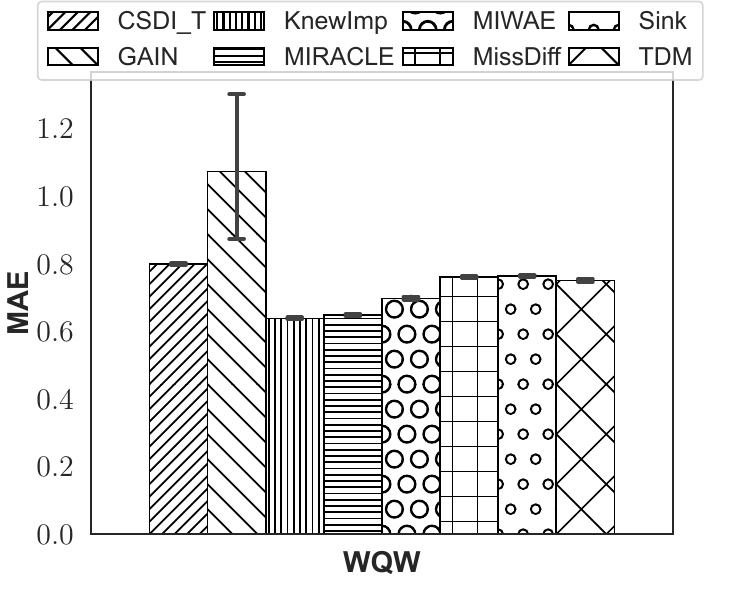}}
\subfigure{\includegraphics[width=0.32\linewidth]{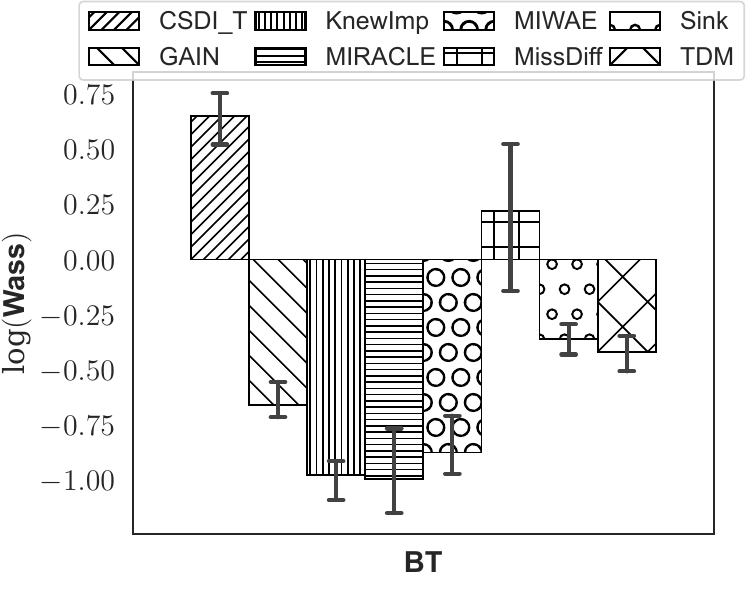}}
\subfigure{\includegraphics[width=0.32\linewidth]{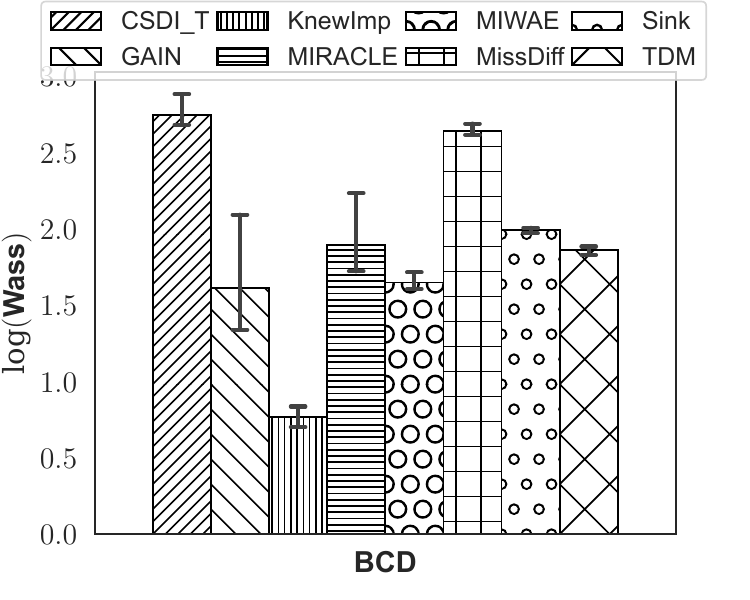}}
\subfigure{\includegraphics[width=0.32\linewidth]{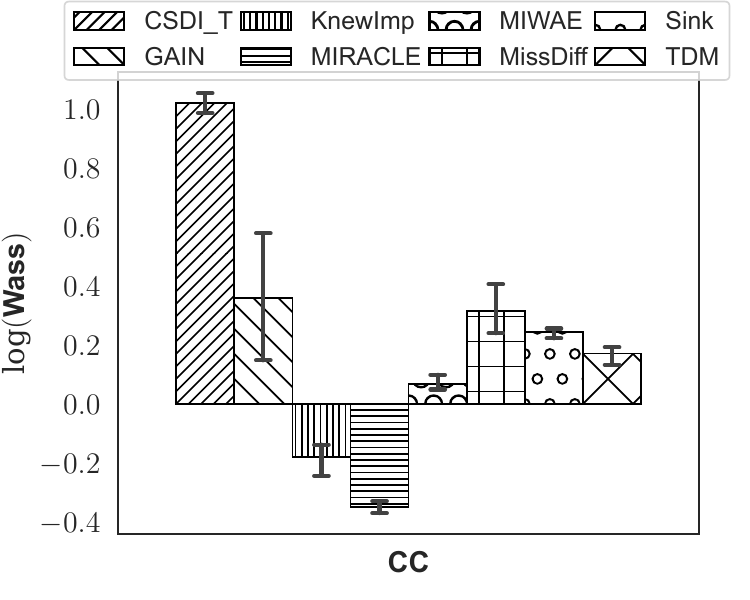}}
\subfigure{\includegraphics[width=0.32\linewidth]{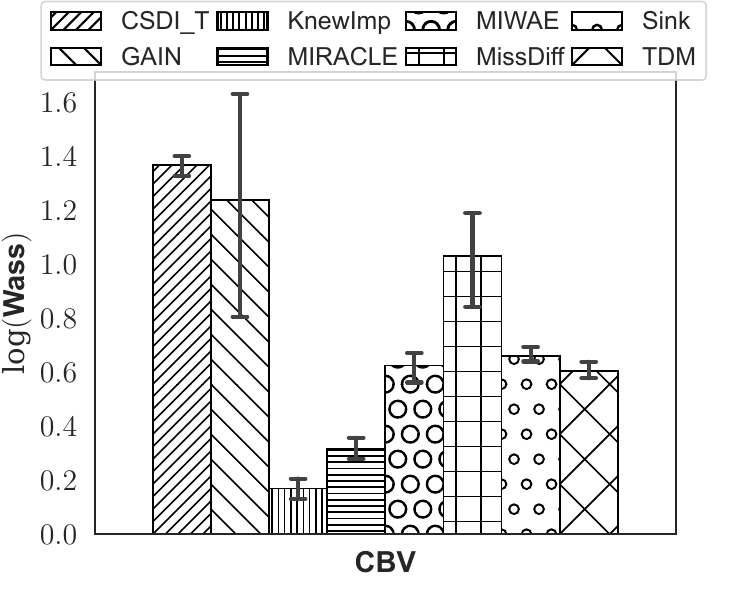}}
\subfigure{\includegraphics[width=0.32\linewidth]{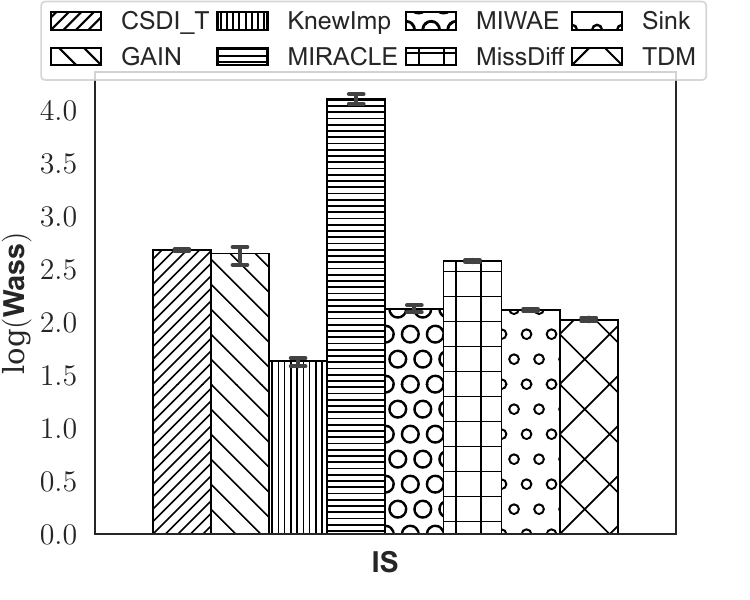}}
\subfigure{\includegraphics[width=0.32\linewidth]{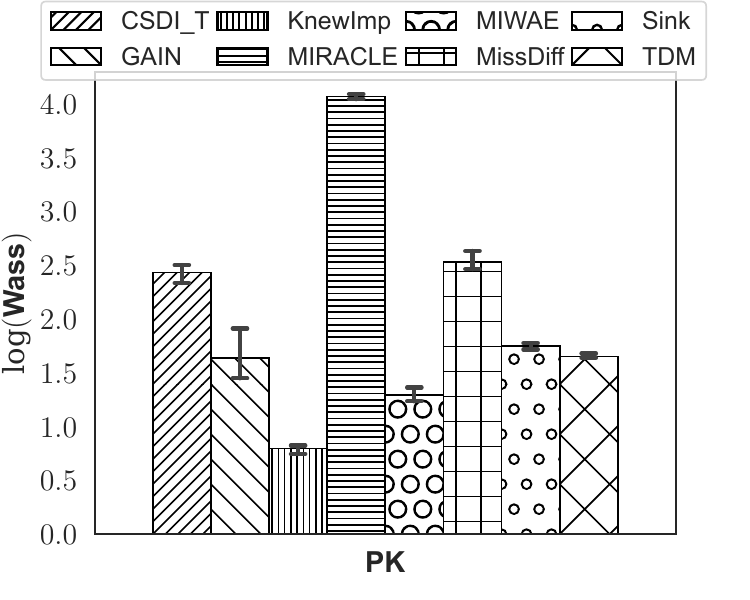}}
\subfigure{\includegraphics[width=0.32\linewidth]{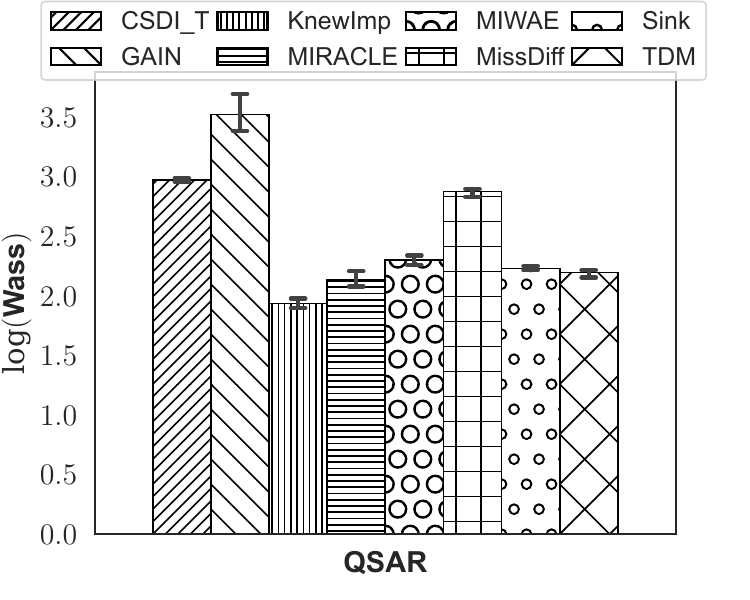}}
\subfigure{\includegraphics[width=0.32\linewidth]{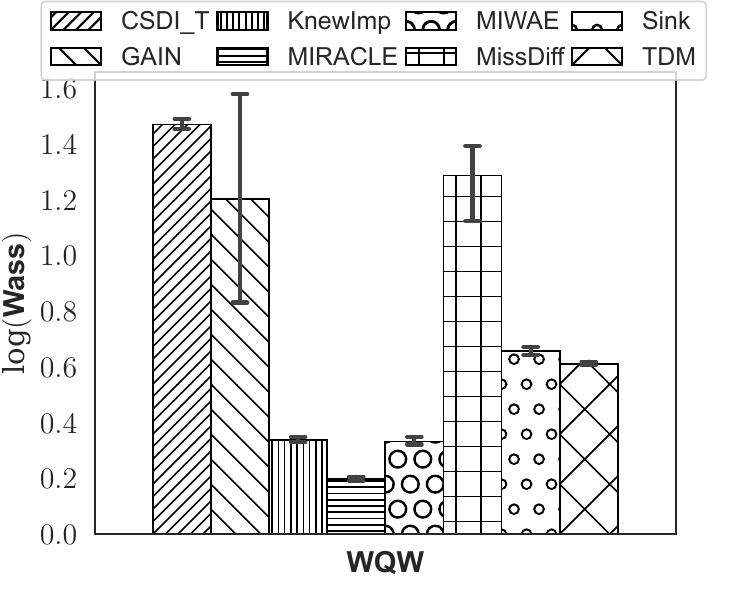}}
  \vspace{-0.5cm}
\caption{Imputation accuracy comparison for MCAR scenario at 50\% missing rate. The error bars indicate the 100\% confidence intervals.}
\label{fig:extraExperMCAR50}
\vspace{-0.5cm}
\end{figure}
\begin{figure}[htbp]
\vspace{-0.5cm}
\centering
\subfigure{\includegraphics[width=0.32\linewidth]{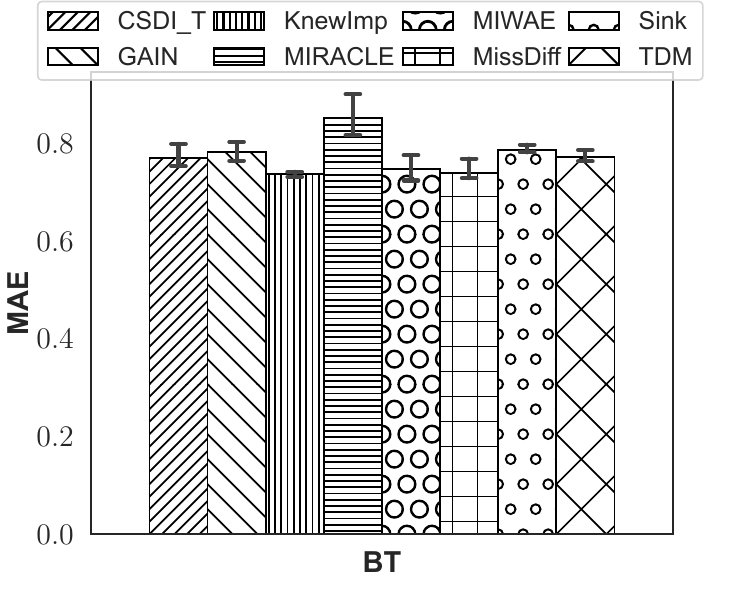}}
\subfigure{\includegraphics[width=0.32\linewidth]{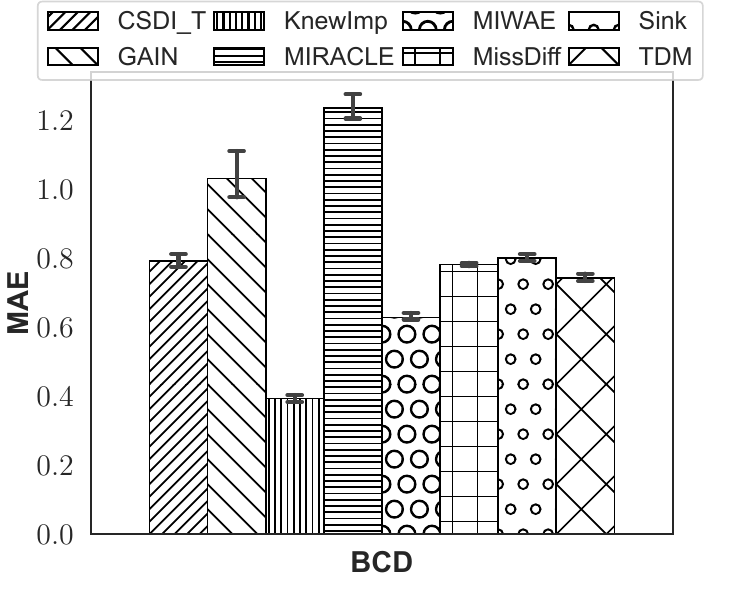}}
\subfigure{\includegraphics[width=0.32\linewidth]{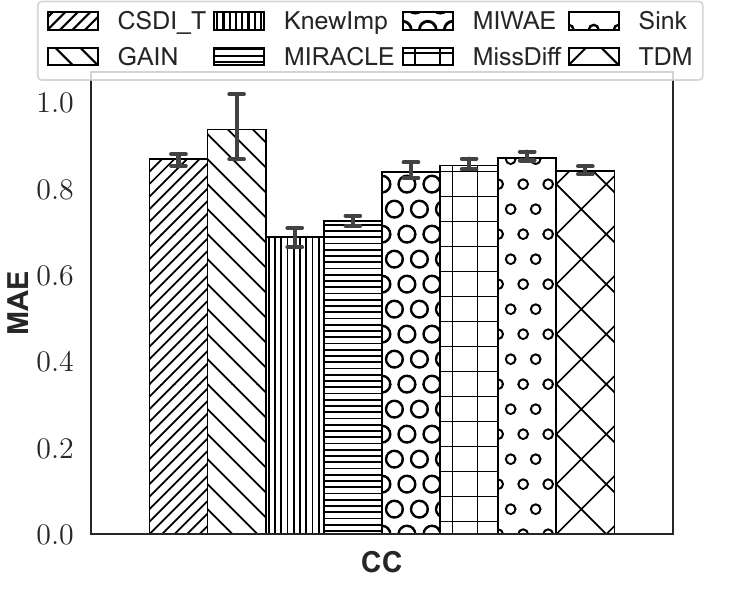}}
\subfigure{\includegraphics[width=0.32\linewidth]{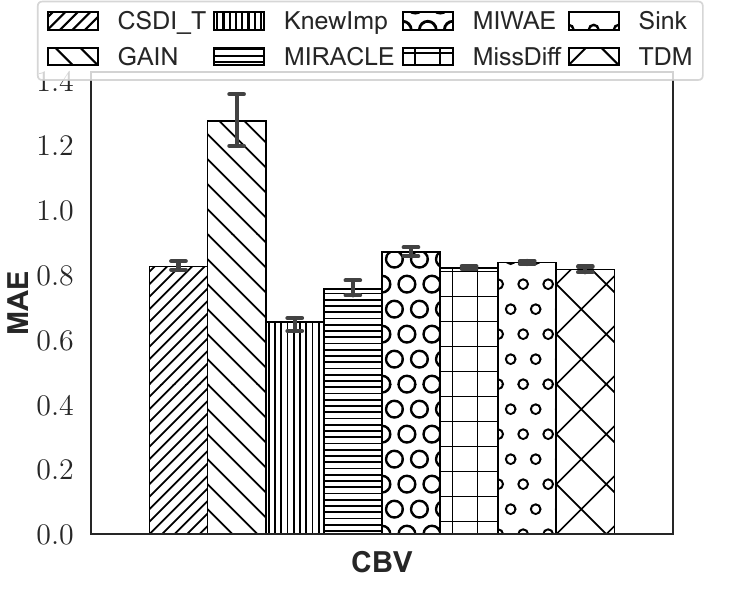}}
\subfigure{\includegraphics[width=0.32\linewidth]{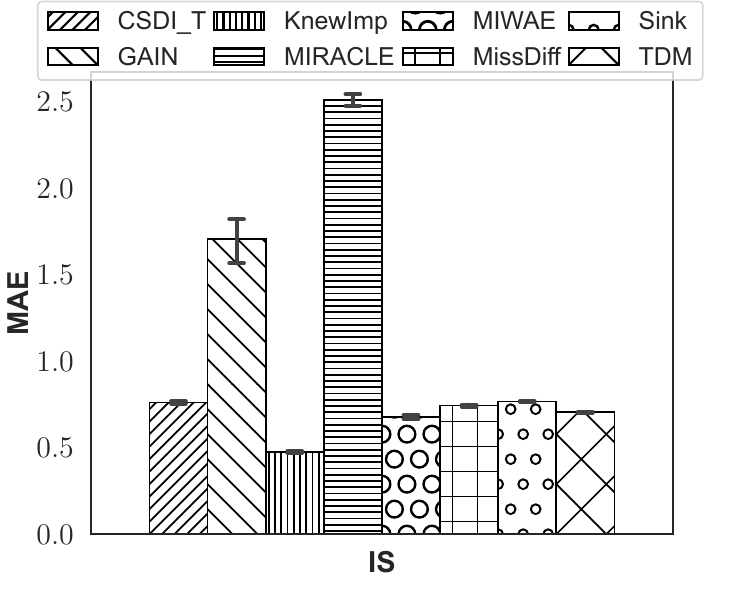}}
\subfigure{\includegraphics[width=0.32\linewidth]{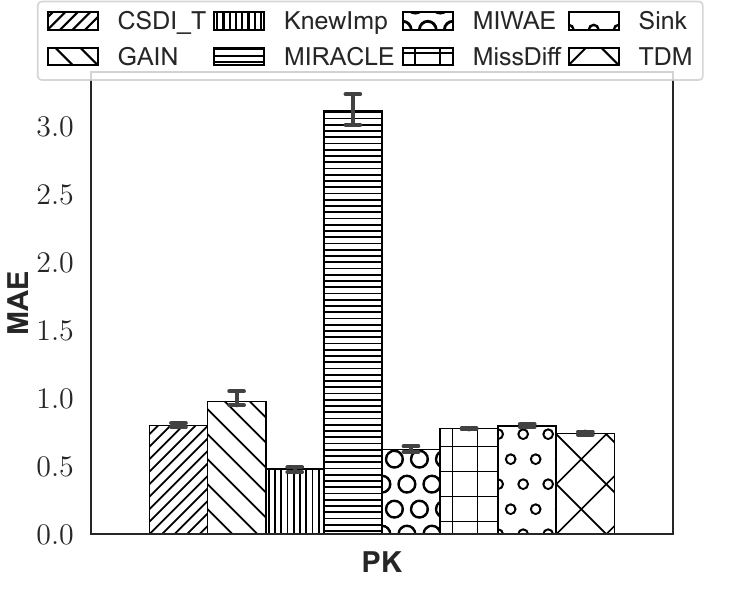}}
\subfigure{\includegraphics[width=0.32\linewidth]{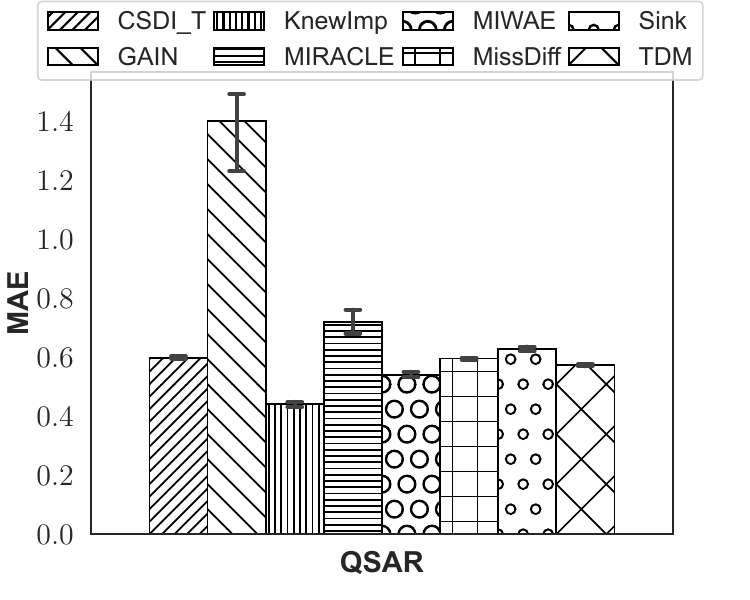}}
\subfigure{\includegraphics[width=0.32\linewidth]{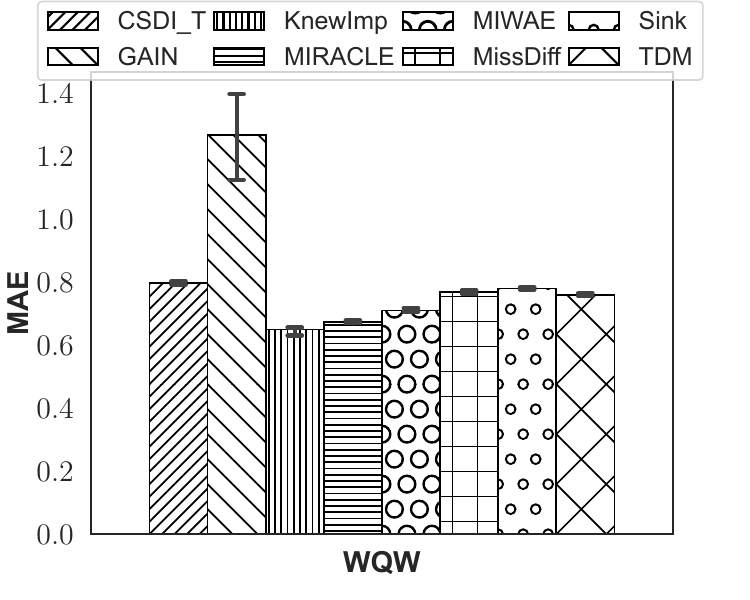}}
\subfigure{\includegraphics[width=0.32\linewidth]{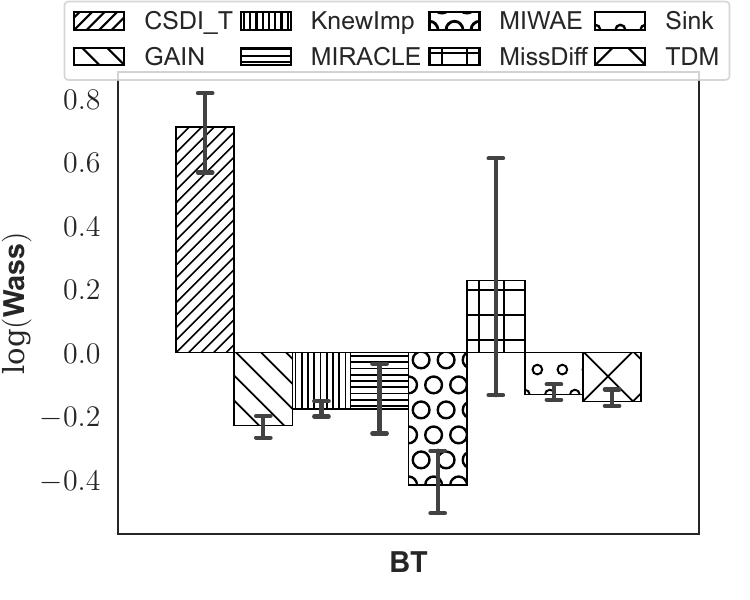}}
\subfigure{\includegraphics[width=0.32\linewidth]{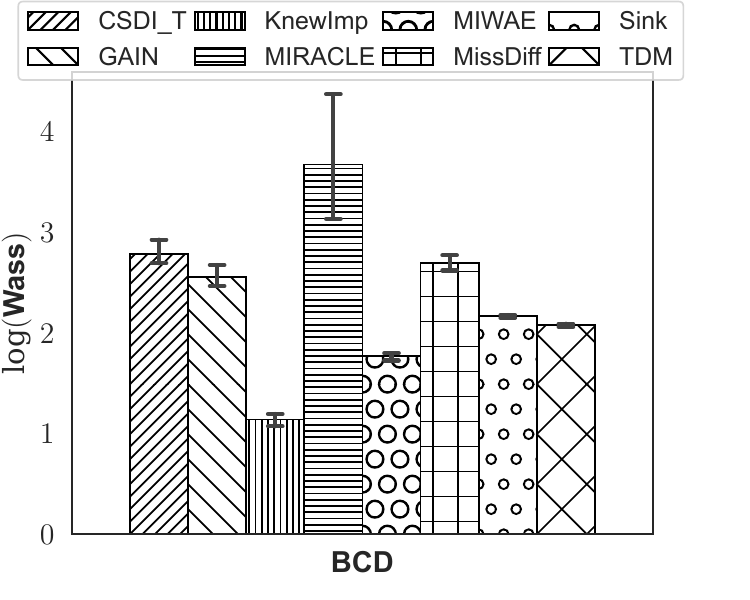}}
\subfigure{\includegraphics[width=0.32\linewidth]{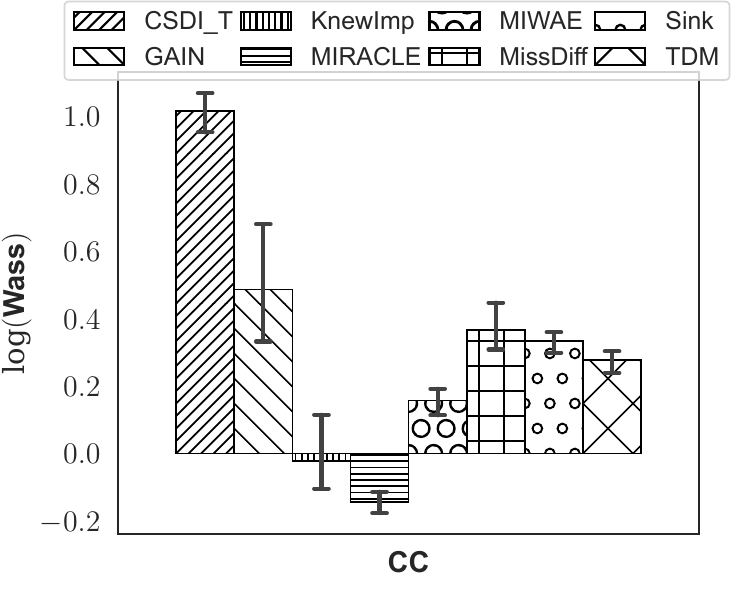}}
\subfigure{\includegraphics[width=0.32\linewidth]{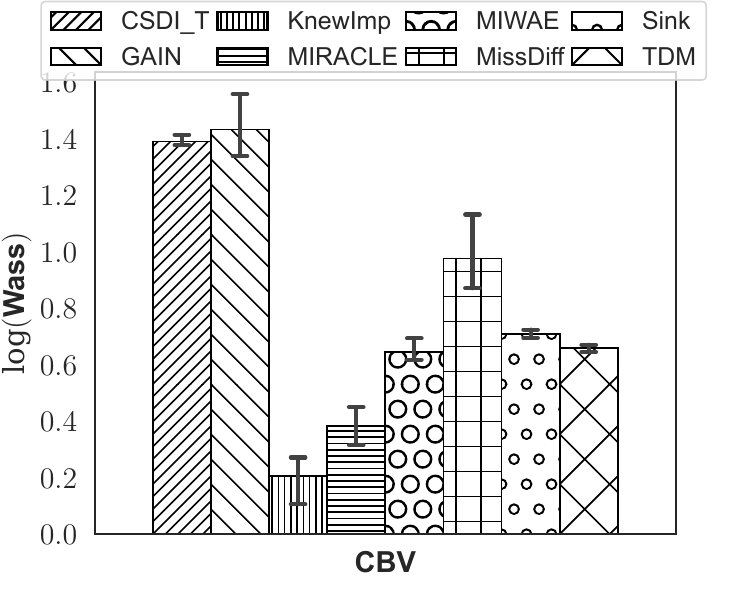}}
\subfigure{\includegraphics[width=0.32\linewidth]{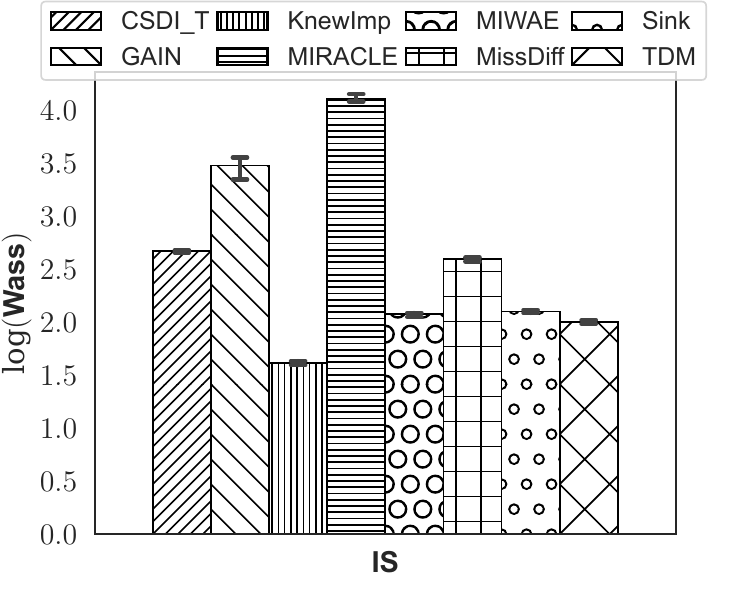}}
\subfigure{\includegraphics[width=0.32\linewidth]{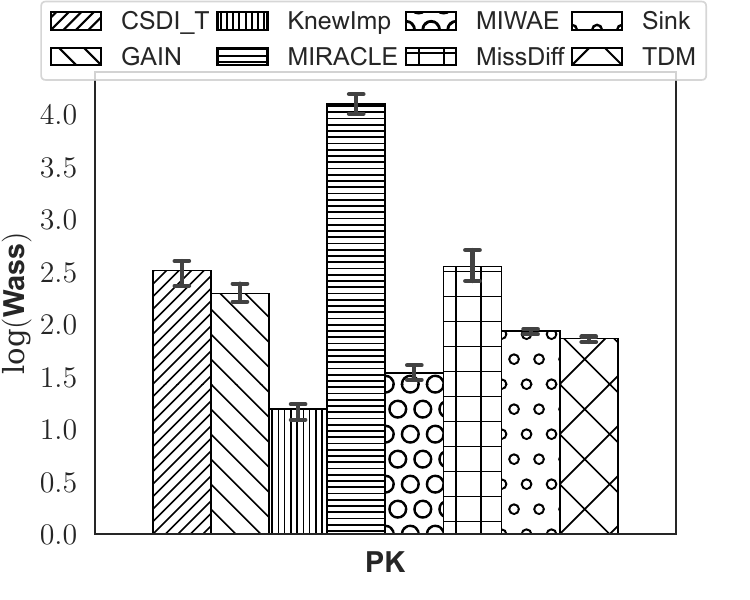}}
\subfigure{\includegraphics[width=0.32\linewidth]{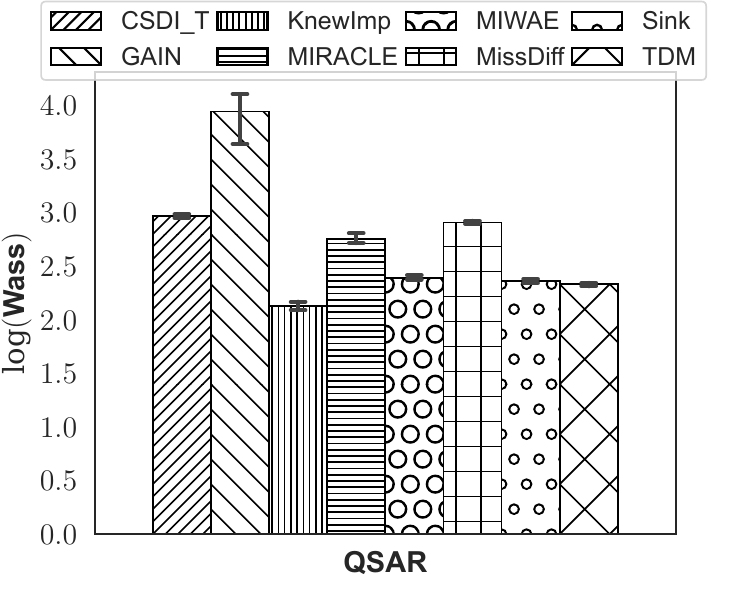}}
\subfigure{\includegraphics[width=0.32\linewidth]{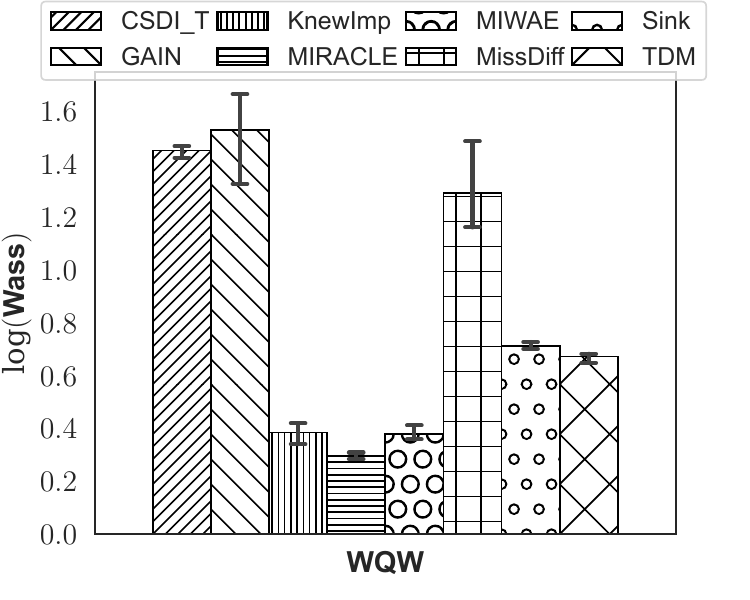}}
\vspace{-0.5cm}
\caption{Imputation accuracy comparison for MNAR scenario at 50\% missing rate. The error bars indicate the 100\% confidence intervals.}
\label{fig:extraExperMNAR50}
  \vspace{-0.5cm}
\end{figure}

\newpage
\section{Limitations \& Future Directions and Broader Impact}
\subsection{Limitations \& Future Directions}\label{appendix-subsec:limitations}
The limitations and future research directions of this work can be summarized as follows:
\begin{itemize}[leftmargin=*]
  \item{\textbf{Utilization of Kernel Function:} During the derivation of the velocity field, we employ RKHS to ensure implementation easiness. However, this regularization term may impose restrictions on the velocity field's direction, potentially limiting imputation accuracy in high-dimensional settings. Additionally, the computational complexity tends to scale quadratically with dataset size increases. Exploring alternative regularization terms~\cite{dong2022particle} to replace RKHS presents a promising direction for future research.}
  \item{\textbf{Training of Score Function:} As discussed in Section~\ref{appendix-subsec:timeComplexity}, the runtime of KnewImp is predominantly governed by the `Estimate' part. Investigating techniques to reduce the training costs of this part, such as employing sliced score matching~\cite{song2020sliced}, represents an intriguing area for future exploration.}
  \item{\textbf{WGF Framework:} The WGF framework currently operates as a first-order system where each sample is equally weighted. A critical advancement would be the incorporation of second-order systems, such as Hamiltonian dynamics~\cite{wang2024gad,wang2022accelerated}, and other gradient flows like, Fisher-Rao gradient flow~\cite{ijcai2022p679} that assign variable weights to samples. This adaptation aims to decrease computational times inherently.}
\end{itemize}
\subsection{Broader Impact Statement}\label{appendix-subsec:boarderImpact}
MDI and DMs are pivotal areas within machine learning, each boasting a wide array of real-world applications. While numerous applications exist, this paper does not single out any specific ones; instead, it focuses on addressing fundamental challenges in these fields. This study significantly advances the application of DMs for MDI by tackling prevalent issues such as inaccurate imputation and challenging training processes. We believe that the insights garnered here can be applied to related domains, such as probabilistic time-series forecasting and image inpainting, where accuracy is often more critical than diversity in results. A common challenge across these domains is the nuanced need for precision over variety, which can lead to overlooked opportunities in model application and development. Our proposed method provides a fresh perspective on these tasks through an optimization lens. It evaluates the appropriateness of directly applying existing diffusion models to these tasks and, where necessary, proposes the derivation of novel algorithms. This approach not only enhances the understanding of the underlying mechanisms but also paves the way for more targeted and effective solutions in the future.

\end{appendices}

\end{document}